\pdfoutput=1
\documentclass{article}
\usepackage{iclr2027_conference,times}
\iclrfinalcopy
\usepackage{etoolbox}
\makeatletter\patchcmd{\@maketitle}{Published as a conference paper at ICLR 2027}{Preprint}{}{\errmessage{header patch failed}}\makeatother
\usepackage[utf8]{inputenc}
\usepackage[T1]{fontenc}
\usepackage{url,booktabs,amsfonts,amsmath,amssymb,amsthm,microtype,graphicx,array}
\usepackage{placeins,float,algorithm,algpseudocode}
\usepackage[hidelinks]{hyperref}

\newcommand{\E}{\mathbb{E}}
\newcommand{\R}{\mathbb{R}}

\newcommand{\bCheckedOne}{26}

\newcommand{\bCheckedTwo}{140}

\newcommand{\bCheckedFour}{499}

\newcommand{\eLaws}{18}
\newcommand{\eNullLaws}{6}
\newcommand{\eWeeksLong}{18}
\newcommand{\eWeeksRatio}{4.5}

\newcommand{\rReqOurs}{3.07}
\newcommand{\rReqBest}{117.8}

\newcommand{\rInclEight}{1.81}

\newcommand{\rInclAll}{11.04}

\newcommand{\rOriginsAll}{9}
\newcommand{\rBaseAll}{131.1}
\newcommand{\rOursAll}{12.2}

\newcommand{\rCostFull}{0.215}
\newcommand{\rCostPoint}{0.660}
\newcommand{\rPointRatio}{3.08}
\newcommand{\rPointCI}{[2.48, 3.81]}
\newcommand{\rAwardFull}{4.0}
\newcommand{\rAwardPoint}{66.9}

\newcommand{\rIndRatio}{1.51}
\newcommand{\rIndCI}{[1.33, 1.74]}

\newcommand{\rNorRatio}{2.46}
\newcommand{\rNorCI}{[2.02, 3.01]}
\newcommand{\rNorAwardDiff}{-10.31}

\newcommand{\rPlgRatio}{1.85}
\newcommand{\rPlgCI}{[1.58, 2.19]}

\newcommand{\rEqLoss}{15.7}
\newcommand{\rEqCI}{[13.8, 17.4]}
\newcommand{\rRegFullHalf}{0.37}
\newcommand{\rRegEqwHalf}{0.78}

\newcommand{\decMMRatioA}{1.40}
\newcommand{\decMMRatioB}{1.41}

\providecommand{\blBitwise}{1{,}431}
\providecommand{\blBitwiseOf}{1{,}431}
\providecommand{\blCoupleShare}{1.4}
\providecommand{\blDPmax}{0.92}
\providecommand{\blDPms}{0.51}
\providecommand{\blDPn}{405}
\providecommand{\blDPnz}{3}
\providecommand{\blEqLoss}{15.7}
\providecommand{\blEqLossHi}{17.4}
\providecommand{\blEqLossLo}{13.8}
\providecommand{\blEqMs}{0.01}

\providecommand{\blIndCost}{1.51}
\providecommand{\blIndCostHi}{1.74}
\providecommand{\blIndCostLo}{1.33}
\providecommand{\blIndSpeed}{1.001}
\providecommand{\blIntegS}{3.86}
\providecommand{\blKAll}{11.0}

\providecommand{\blMfAlt}{84.8}

\providecommand{\blMfKAll}{7.11}
\providecommand{\blMfKEight}{1.66}
\providecommand{\blMfLossN}{405}
\providecommand{\blMfLossNz}{0}
\providecommand{\blMfOurs}{4.36}
\providecommand{\blMfRatio}{20.0}
\providecommand{\blMfRatioHi}{21.2}
\providecommand{\blMfRatioLo}{18.9}
\providecommand{\blMipCond}{396}
\providecommand{\blMipDPExact}{65{,}340}
\providecommand{\blMipEnvExact}{65{,}335}
\providecommand{\blMipEnvNine}{2.5}
\providecommand{\blMipEnvOne}{1.06}

\providecommand{\blMipEnvTT}{5.5}
\providecommand{\blMipRawExact}{65{,}315}
\providecommand{\blMipRawNine}{7.6}
\providecommand{\blMipRawOne}{1.08}

\providecommand{\blMipRawTT}{20.9}
\providecommand{\blMipReq}{65{,}340}
\providecommand{\blNormCost}{2.46}
\providecommand{\blNormCostFS}{3.29}
\providecommand{\blNormCostHi}{3.01}
\providecommand{\blNormCostLo}{2.02}
\providecommand{\blNormSpeed}{85.8}

\providecommand{\blPerRiderGiB}{190}

\providecommand{\blPointCost}{3.08}
\providecommand{\blPointCostHi}{3.81}
\providecommand{\blPointCostLo}{2.48}

\providecommand{\blRctLaws}{18}

\providecommand{\blReqAlt}{117.8}
\providecommand{\blReqOurs}{3.07}
\providecommand{\blReqRatio}{38.6}
\providecommand{\blReqRatioHi}{43.0}
\providecommand{\blReqRatioLo}{30.8}
\providecommand{\blTierTable}{4.05}

\newcommand{\xFWOne}{0.252}
\newcommand{\xSampOne}{0.286}
\newcommand{\xSCOne}{0.542}
\newcommand{\xNoSynOne}{0.658}
\newcommand{\xNoMomOne}{0.656}
\newcommand{\xPlugOne}{0.583}

\newcommand{\xRelRedOne}{53.6}
\newcommand{\xNoSynDOne}{$+$0.406}

\newcommand{\xNoMomDOne}{$+$0.404}

\newcommand{\xPlugDOne}{$+$0.331}

\newcommand{\xSampDOne}{$+$0.035}
\newcommand{\xSampDCIOne}{[$+$0.010, $+$0.061]}

\newcommand{\xValFWOne}{80.3}
\newcommand{\xValSCOne}{53.2}

\newcommand{\xFWTwo}{0.186}
\newcommand{\xSampTwo}{0.223}
\newcommand{\xSCTwo}{0.383}
\newcommand{\xNoSynTwo}{0.611}
\newcommand{\xNoMomTwo}{0.656}
\newcommand{\xPlugTwo}{0.599}

\newcommand{\xValFWTwo}{83.5}
\newcommand{\xValSCTwo}{69.0}

\newcommand{\xFWFour}{0.136}
\newcommand{\xSampFour}{0.189}
\newcommand{\xSCFour}{0.264}
\newcommand{\xNoSynFour}{0.569}
\newcommand{\xNoMomFour}{0.656}
\newcommand{\xPlugFour}{0.616}

\newcommand{\xFWmLFour}{$+$0.006}
\newcommand{\xFWmLCIFour}{[$-$0.013, $+$0.024]}

\newcommand{\xValFWFour}{87.5}
\newcommand{\xValSCFour}{76.5}

\newcommand{\xLong}{0.130}

\newcommand{\xExSCOne}{$-$0.296}
\newcommand{\xExSCCIOne}{[$-$0.345, $-$0.248]}
\newcommand{\xExSCTwo}{$-$0.198}
\newcommand{\xExSCCITwo}{[$-$0.243, $-$0.157]}
\newcommand{\xExSCFour}{$-$0.135}
\newcommand{\xExSCCIFour}{[$-$0.169, $-$0.101]}
\newcommand{\xExLFour}{$+$0.004}
\newcommand{\xExLCIFour}{[$-$0.015, $+$0.023]}
\newcommand{\xExNoSynOne}{$+$0.386}
\newcommand{\xExNoMomOne}{$+$0.391}
\newcommand{\xExPlugOne}{$+$0.328}
\newcommand{\xExSampOne}{$+$0.034}
\newcommand{\xExSampCIOne}{[$+$0.008, $+$0.061]}

\newcommand{\hoMargin}{0.045}
\newcommand{\hoSD}{$+$0.003}
\newcommand{\hoSDCI}{[$-$0.018, $+$0.020]}

\newcommand{\hoSHeld}{0.343}
\newcommand{\hoSJobs}{180}
\newcommand{\hoSNecEq}{$+$0.122}
\newcommand{\hoSNecEqCI}{[$+$0.099, $+$0.145]}
\newcommand{\hoSNecEqw}{$+$0.0046}
\newcommand{\hoSNecEqwCI}{[$+$0.0024, $+$0.0070]}

\newcommand{\hoSNecInd}{$+$0.058}
\newcommand{\hoSNecIndCI}{[$+$0.035, $+$0.080]}
\newcommand{\hoSNecNorm}{$+$0.205}
\newcommand{\hoSNecNormCI}{[$+$0.167, $+$0.240]}
\newcommand{\hoSNecPlug}{$+$0.123}
\newcommand{\hoSNecPlugCI}{[$+$0.097, $+$0.148]}
\newcommand{\hoSNecPoint}{$+$0.380}
\newcommand{\hoSNecPointCI}{[$+$0.292, $+$0.467]}
\newcommand{\hoSOrigins}{45}
\newcommand{\hoSPass}{yes}
\newcommand{\hoSRef}{0.340}

\newcommand{\hoSRiders}{311}

\newcommand{\hoTNecEq}{$+$0.048}
\newcommand{\hoTNecEqCI}{[$+$0.021, $+$0.080]}
\newcommand{\hoTNecEqw}{$+$0.0017}
\newcommand{\hoTNecEqwCI}{[$+$0.0006, $+$0.0028]}

\newcommand{\hoTNecInd}{$+$0.202}
\newcommand{\hoTNecIndCI}{[$+$0.142, $+$0.250]}
\newcommand{\hoTNecNorm}{$+$0.459}
\newcommand{\hoTNecNormCI}{[$+$0.382, $+$0.523]}
\newcommand{\hoTNecPlug}{$+$0.317}
\newcommand{\hoTNecPlugCI}{[$+$0.235, $+$0.382]}
\newcommand{\hoTNecPoint}{$+$0.416}
\newcommand{\hoTNecPointCI}{[$+$0.359, $+$0.472]}

\newcommand{\pmFWFour}{0.045}
\newcommand{\pmFWOne}{0.111}
\newcommand{\pmFWTwo}{0.073}

\newcommand{\pmFWmLCIFour}{[$-$0.004, $+$0.015]}

\newcommand{\pmFWmLFour}{$+$0.006}

\newcommand{\pmFWmSCIFour}{[$-$0.093, $-$0.044]}
\newcommand{\pmFWmSCIOne}{[$-$0.189, $-$0.133]}
\newcommand{\pmFWmSCITwo}{[$-$0.123, $-$0.071]}

\newcommand{\pmFWmSabsFour}{0.067}
\newcommand{\pmFWmSabsOne}{0.158}
\newcommand{\pmFWmSabsTwo}{0.095}
\newcommand{\pmLong}{0.040}

\newcommand{\pmNoSynFour}{0.353}
\newcommand{\pmNoSynOne}{0.401}
\newcommand{\pmNoSynTwo}{0.379}

\newcommand{\pmRelRedOne}{58.9}

\newcommand{\pmSCFour}{0.112}
\newcommand{\pmSCOne}{0.269}
\newcommand{\pmSCTwo}{0.169}

\newcommand{\pmSampDCIOne}{[$+$0.017, $+$0.042]}

\newcommand{\pmSampDOne}{$+$0.029}

\newcommand{\pmSampFour}{0.089}
\newcommand{\pmSampOne}{0.140}
\newcommand{\pmSampTwo}{0.105}

\providecommand{\dfUnits}{135}

\providecommand{\dfPlansTotal}{72,900}
\providecommand{\dfZeroPlans}{810}

\providecommand{\dfZeroShare}{1.1}

\providecommand{\dfBindHalf}{3}

\providecommand{\dfRegMeanHalf}{0.37}

\providecommand{\dfOffline}{114.2}

\providecommand{\dfReq}{3.07}
\providecommand{\dfReqFirst}{2.97}

\providecommand{\dfBase}{114.5}

\providecommand{\dfOriginsAll}{0, 5, 10, 15, 20, 25, 30, 35, 40}

\providecommand{\enumRone}{810}

\providecommand{\enumEcalls}{10{,}800}
\providecommand{\enumEbands}{356{,}400}
\providecommand{\enumEdiff}{0}

\providecommand{\enumDPpredN}{405}
\providecommand{\enumDPpredNZ}{3}
\providecommand{\enumDPpredMax}{0.92}

\providecommand{\nmNorAwardDiff}{10.31}
\providecommand{\nmEnumBrute}{675}
\providecommand{\nmEnumBrutePred}{12,150}
\providecommand{\nmEnumPred}{14,580}
\providecommand{\nmDPfbN}{135}
\providecommand{\nmDPfbZero}{134}
\providecommand{\nmDPfbMax}{0.69}
\providecommand{\nmActMean}{70.98}
\providecommand{\nmActMin}{67.41}
\providecommand{\nmActMax}{74.08}
\providecommand{\nmMfullEight}{1.59}
\renewcommand{\pmFWOne}{0.205}
\renewcommand{\pmSampOne}{0.245}
\renewcommand{\pmSCOne}{0.421}
\renewcommand{\pmNoSynOne}{0.603}

\renewcommand{\pmFWmSCIOne}{[$-$0.261, $-$0.173]}
\renewcommand{\pmFWmSabsOne}{0.215}
\renewcommand{\pmRelRedOne}{51.2}

\renewcommand{\pmSampDOne}{$+$0.039}
\renewcommand{\pmSampDCIOne}{[$+$0.017, $+$0.064]}

\renewcommand{\pmFWTwo}{0.159}
\renewcommand{\pmSampTwo}{0.197}
\renewcommand{\pmSCTwo}{0.313}
\renewcommand{\pmNoSynTwo}{0.560}

\renewcommand{\pmFWmSCITwo}{[$-$0.198, $-$0.115]}
\renewcommand{\pmFWmSabsTwo}{0.155}

\renewcommand{\pmFWFour}{0.121}
\renewcommand{\pmSampFour}{0.163}
\renewcommand{\pmSCFour}{0.225}
\renewcommand{\pmNoSynFour}{0.527}

\renewcommand{\pmFWmSCIFour}{[$-$0.140, $-$0.071]}
\renewcommand{\pmFWmSabsFour}{0.105}

\renewcommand{\pmFWmLFour}{$+$0.007}
\renewcommand{\pmFWmLCIFour}{[$-$0.012, $+$0.027]}

\renewcommand{\pmLong}{0.113}

\newcommand{\eFFWOne}{0.220}
\newcommand{\eFSCOne}{0.391}
\newcommand{\eFDOne}{$-$0.171}
\newcommand{\eFDCIOne}{[$-$0.201, $-$0.140]}
\newcommand{\eFRelOne}{43.8}
\newcommand{\eFNoMomOne}{0.592}
\newcommand{\eFdNoMomOne}{$+$0.373}
\newcommand{\eFdNoMomCIOne}{[$+$0.333, $+$0.412]}
\newcommand{\eFNoSynOne}{0.471}
\newcommand{\eFdNoSynOne}{$+$0.251}
\newcommand{\eFdNoSynCIOne}{[$+$0.191, $+$0.312]}
\newcommand{\eFPlugOne}{0.592}
\newcommand{\eFdPlugOne}{$+$0.372}
\newcommand{\eFdPlugCIOne}{[$+$0.328, $+$0.417]}
\newcommand{\eFIndOne}{0.311}
\newcommand{\eFdIndOne}{$+$0.091}
\newcommand{\eFdIndCIOne}{[$+$0.068, $+$0.116]}
\newcommand{\eFEqOne}{0.699}
\newcommand{\eFdEqOne}{$+$0.479}
\newcommand{\eFdEqCIOne}{[$+$0.392, $+$0.570]}
\newcommand{\eFCOne}{0.227}
\newcommand{\eFdCOne}{$+$0.007}
\newcommand{\eFdCCIOne}{[$+$0.004, $+$0.010]}
\newcommand{\eFValOne}{93.2}

\newcommand{\eFFWTwo}{0.172}
\newcommand{\eFSCTwo}{0.296}
\newcommand{\eFDTwo}{$-$0.124}
\newcommand{\eFDCITwo}{[$-$0.160, $-$0.089]}
\newcommand{\eFRelTwo}{42.0}
\newcommand{\eFNoMomTwo}{0.592}

\newcommand{\eFNoSynTwo}{0.433}

\newcommand{\eFPlugTwo}{0.604}

\newcommand{\eFIndTwo}{0.292}

\newcommand{\eFEqTwo}{0.692}

\newcommand{\eFCTwo}{0.177}
\newcommand{\eFdCTwo}{$+$0.005}
\newcommand{\eFdCCITwo}{[$+$0.003, $+$0.007]}
\newcommand{\eFValTwo}{94.0}

\newcommand{\eFFWFour}{0.148}
\newcommand{\eFSCFour}{0.216}
\newcommand{\eFDFour}{$-$0.069}
\newcommand{\eFDCIFour}{[$-$0.104, $-$0.035]}
\newcommand{\eFRelFour}{31.8}
\newcommand{\eFNoMomFour}{0.592}

\newcommand{\eFNoSynFour}{0.404}

\newcommand{\eFPlugFour}{0.631}

\newcommand{\eFIndFour}{0.296}

\newcommand{\eFEqFour}{0.684}

\newcommand{\eFCFour}{0.154}
\newcommand{\eFdCFour}{$+$0.006}
\newcommand{\eFdCCIFour}{[$+$0.003, $+$0.009]}
\newcommand{\eFValFour}{94.1}

\newcommand{\eFXFWOne}{0.258}
\newcommand{\eFXSCOne}{0.527}
\newcommand{\eFXDOne}{$-$0.269}
\newcommand{\eFXDCIOne}{[$-$0.297, $-$0.240]}

\newcommand{\eFXNoMomOne}{0.626}

\newcommand{\eFXNoSynOne}{0.561}

\newcommand{\eFXPlugOne}{0.669}

\newcommand{\eFXIndOne}{0.368}

\newcommand{\eFXEqOne}{0.902}

\newcommand{\eFXCOne}{0.265}

\newcommand{\eFXValOne}{87.7}

\newcommand{\eFXFWTwo}{0.198}
\newcommand{\eFXSCTwo}{0.388}
\newcommand{\eFXDTwo}{$-$0.190}
\newcommand{\eFXDCITwo}{[$-$0.225, $-$0.156]}

\newcommand{\eFXNoMomTwo}{0.626}

\newcommand{\eFXNoSynTwo}{0.509}

\newcommand{\eFXPlugTwo}{0.672}

\newcommand{\eFXIndTwo}{0.340}

\newcommand{\eFXEqTwo}{0.898}

\newcommand{\eFXCTwo}{0.205}

\newcommand{\eFXValTwo}{90.0}

\newcommand{\eFXFWFour}{0.164}
\newcommand{\eFXSCFour}{0.269}
\newcommand{\eFXDFour}{$-$0.105}
\newcommand{\eFXDCIFour}{[$-$0.139, $-$0.073]}

\newcommand{\eFXNoMomFour}{0.626}

\newcommand{\eFXNoSynFour}{0.470}

\newcommand{\eFXPlugFour}{0.690}

\newcommand{\eFXIndFour}{0.339}

\newcommand{\eFXEqFour}{0.892}

\newcommand{\eFXCFour}{0.169}

\newcommand{\eFXValFour}{91.1}

\newcommand{\eFDPmax}{0.001}
\newcommand{\eFBind}{72.3}

\newcommand{\eFIdentAll}{51,840}
\newcommand{\eFIdentFortyEight}{69,120}
\newcommand{\eFSpEightOne}{5.8}

\newcommand{\eFSpThirtyTwoOne}{11.9}

\newcommand{\eFSpSixtyOne}{14.1}
\newcommand{\eFSpSixtyCIOne}{[13.8, 14.5]}
\newcommand{\eFPrepMsOne}{7.27}
\newcommand{\eFReqMsOne}{0.22}
\newcommand{\eFNoReuseMsOne}{7.57}
\newcommand{\eFSpEightTwo}{5.8}

\newcommand{\eFSpThirtyTwoTwo}{11.9}

\newcommand{\eFSpSixtyTwo}{14.2}
\newcommand{\eFSpSixtyCITwo}{[13.8, 14.5]}
\newcommand{\eFPrepMsTwo}{7.40}
\newcommand{\eFReqMsTwo}{0.22}
\newcommand{\eFNoReuseMsTwo}{7.57}
\newcommand{\eFSpEightFour}{5.9}

\newcommand{\eFSpThirtyTwoFour}{12.0}

\newcommand{\eFSpSixtyFour}{14.3}
\newcommand{\eFSpSixtyCIFour}{[14.0, 14.7]}
\newcommand{\eFPrepMsFour}{6.25}
\newcommand{\eFReqMsFour}{0.21}
\newcommand{\eFNoReuseMsFour}{6.61}

\newcommand{\eFLaws}{48}
\newcommand{\eFRespLaws}{36}
\newcommand{\eFNullLaws}{12}
\newcommand{\eFReqs}{60}
\newcommand{\eFBudgeted}{180}
\newcommand{\eFReps}{8}

\newcommand{\exEqTenOne}{$+$0.610}

\newcommand{\exEqLvlTenOne}{0.861}

\newcommand{\exEqLvlTenTwo}{0.860}

\newcommand{\exEqLvlTenFour}{0.855}
\newcommand{\exEqTwentyOne}{$+$0.455}
\newcommand{\exEqTwentyCIOne}{[$+$0.344, $+$0.571]}

\newcommand{\eGFWOne}{0.188}
\newcommand{\eGSCOne}{0.477}
\newcommand{\eGDOne}{$-$0.289}
\newcommand{\eGDCIOne}{[$-$0.311, $-$0.267]}
\newcommand{\eGRelOne}{60.6}
\newcommand{\eGNoMomOne}{0.433}
\newcommand{\eGdNoMomOne}{$+$0.245}
\newcommand{\eGdNoMomCIOne}{[$+$0.214, $+$0.275]}
\newcommand{\eGNoSynOne}{0.415}
\newcommand{\eGdNoSynOne}{$+$0.228}
\newcommand{\eGdNoSynCIOne}{[$+$0.206, $+$0.252]}
\newcommand{\eGPlugOne}{0.447}
\newcommand{\eGdPlugOne}{$+$0.260}
\newcommand{\eGdPlugCIOne}{[$+$0.233, $+$0.287]}
\newcommand{\eGIndOne}{0.198}
\newcommand{\eGdIndOne}{$+$0.011}
\newcommand{\eGdIndCIOne}{[$+$0.009, $+$0.013]}
\newcommand{\eGEqOne}{0.498}
\newcommand{\eGdEqOne}{$+$0.311}
\newcommand{\eGdEqCIOne}{[$+$0.269, $+$0.355]}
\newcommand{\eGCOne}{0.210}
\newcommand{\eGdCOne}{$+$0.022}
\newcommand{\eGdCCIOne}{[$+$0.018, $+$0.027]}
\newcommand{\eGValOne}{93.1}
\newcommand{\eGFWTwo}{0.150}
\newcommand{\eGSCTwo}{0.326}
\newcommand{\eGDTwo}{$-$0.176}
\newcommand{\eGDCITwo}{[$-$0.192, $-$0.160]}
\newcommand{\eGRelTwo}{53.9}
\newcommand{\eGNoMomTwo}{0.433}

\newcommand{\eGNoSynTwo}{0.354}

\newcommand{\eGPlugTwo}{0.479}

\newcommand{\eGIndTwo}{0.158}

\newcommand{\eGEqTwo}{0.477}

\newcommand{\eGCTwo}{0.171}
\newcommand{\eGdCTwo}{$+$0.020}
\newcommand{\eGdCCITwo}{[$+$0.016, $+$0.024]}
\newcommand{\eGValTwo}{94.6}
\newcommand{\eGFWFour}{0.116}
\newcommand{\eGSCFour}{0.240}
\newcommand{\eGDFour}{$-$0.123}
\newcommand{\eGDCIFour}{[$-$0.140, $-$0.108]}
\newcommand{\eGRelFour}{51.4}
\newcommand{\eGNoMomFour}{0.433}

\newcommand{\eGNoSynFour}{0.309}

\newcommand{\eGPlugFour}{0.510}

\newcommand{\eGIndFour}{0.124}

\newcommand{\eGEqFour}{0.466}

\newcommand{\eGCFour}{0.131}
\newcommand{\eGdCFour}{$+$0.015}
\newcommand{\eGdCCIFour}{[$+$0.010, $+$0.019]}
\newcommand{\eGValFour}{95.7}
\newcommand{\eGXFWOne}{0.270}
\newcommand{\eGXSCOne}{0.765}
\newcommand{\eGXDOne}{$-$0.495}
\newcommand{\eGXDCIOne}{[$-$0.516, $-$0.473]}

\newcommand{\eGXNoMomOne}{0.781}

\newcommand{\eGXNoSynOne}{0.610}

\newcommand{\eGXPlugOne}{0.576}

\newcommand{\eGXIndOne}{0.314}

\newcommand{\eGXEqOne}{0.793}

\newcommand{\eGXCOne}{0.305}

\newcommand{\eGXValOne}{83.5}
\newcommand{\eGXFWTwo}{0.197}
\newcommand{\eGXSCTwo}{0.601}
\newcommand{\eGXDTwo}{$-$0.404}
\newcommand{\eGXDCITwo}{[$-$0.426, $-$0.383]}

\newcommand{\eGXNoMomTwo}{0.781}

\newcommand{\eGXNoSynTwo}{0.517}

\newcommand{\eGXPlugTwo}{0.577}

\newcommand{\eGXIndTwo}{0.222}

\newcommand{\eGXEqTwo}{0.775}

\newcommand{\eGXCTwo}{0.223}

\newcommand{\eGXValTwo}{88.9}
\newcommand{\eGXFWFour}{0.140}
\newcommand{\eGXSCFour}{0.415}
\newcommand{\eGXDFour}{$-$0.276}
\newcommand{\eGXDCIFour}{[$-$0.296, $-$0.257]}

\newcommand{\eGXNoMomFour}{0.781}

\newcommand{\eGXNoSynFour}{0.448}

\newcommand{\eGXPlugFour}{0.591}

\newcommand{\eGXIndFour}{0.162}

\newcommand{\eGXEqFour}{0.760}

\newcommand{\eGXCFour}{0.159}

\newcommand{\eGXValFour}{92.7}
\newcommand{\eGDPmax}{0.001}
\newcommand{\eGBind}{75.0}
\newcommand{\eGSpEightOne}{1.22}

\newcommand{\eGSpThirtyTwoOne}{2.33}

\newcommand{\eGSpSixtyOne}{2.70}
\newcommand{\eGSpSixtyCIOne}{[2.67, 2.73]}
\newcommand{\eGPrepMsOne}{52.4}
\newcommand{\eGReqMsOne}{4.1}
\newcommand{\eGNoReuseMsOne}{13.7}

\newcommand{\eGSpSixtyTwo}{2.72}

\newcommand{\eGSpSixtyFour}{2.78}

\newcommand{\eGIdentAll}{51,840}

\newcommand{\eGLaws}{48}
\newcommand{\eGRespLaws}{36}

\newcommand{\eGBudgeted}{180}
\newcommand{\eGReps}{8}

\makeatletter\newcommand{\rawinput}[1]{\input{#1} }\makeatother
\newtheorem{proposition}{Proposition}
\newtheorem{lemma}{Lemma}
\newtheorem{corollary}{Corollary}
\title{A General Framework for Budgeted\\Threshold Incentives on Request}
\author{Zhuolin Wu$^{1}$, Chengrui Zhu$^{1}$, Wenhua Nie$^{1,2}$, Kenny Ye Liang$^{1,3}$, \\
Junming Lin$^{1,4}$, Haiyang Li$^{1}$, Zhilin Li$^{1}$, Wenjia Geng$^{1}$, \\
Zeyu Wu$^{1}$, Yinan Wu$^{1}$, Jinghua Hao$^{1}$, Renqing He$^{1}$ \\[1ex]
{\normalfont $^{1}$Meituan} \\
{\normalfont $^{2}$National Taiwan University} \\
{\normalfont $^{3}$Tsinghua University} \\
{\normalfont $^{4}$Tongji University}}
\hypersetup{pdftitle={A General Framework for Budgeted Threshold Incentives on Request},pdfauthor={Zhuolin Wu, Chengrui Zhu, Wenhua Nie, Kenny Ye Liang, Junming Lin, Haiyang Li, Zhilin Li, Wenjia Geng, Zeyu Wu, Yinan Wu, Jinghua Hao, Renqing He}}
\begin{document}
\maketitle
\begin{abstract}
On-demand delivery platforms pay riders through incentive activities whose tiers are set from recent completions of riders with a similar history, so that a little extra effort earns a clearly stated reward. Operators request such plans for changing periods, rider populations, payment rules and budgets, days ahead and within minutes, often for holidays or bad weather where randomized trials are scarce and take months to collect. We present a request-driven framework that composes four stages---conditional prediction, population reduction, trajectory integration and budget allocation---through seven replaceable modules that exchange conditional trajectory laws, whose award probabilities and award-marked moments give payment and uplift for any activity rule. To shorten a long randomized campaign, a response-correction step reweights trajectories from abundant no-offer history to match the moments of a short pilot. We prove that, on a fixed plan menu and given the stage errors, the end-to-end value loss is bounded by the sum of four stage terms---synthetic data, moment matching, integration and decision---with constants that cannot be improved from the final tables, and that for every stage there are instances on which omitting it leaves an error floor the others cannot remove. On 3,000 riders over 45 weekly origins, re-drawing trajectories and solving exactly for every request takes \rReqBest\,s per week-long request against \rReqOurs\,s for the framework; including the one-off sampling pass, all 127 windows of a week are answered \rInclAll$\times$ faster with identical scenarios and at most \blDPmax\% value lost by the allocation on the audited week-long tables, a gain that comes entirely from reuse, while a point forecast, independent days or an equal budget split each lose accuracy. Cities held out of training keep their plan-cost error within a pre-set margin (change \hoSD\ \hoSDCI). On 24 new controlled response laws, with the plan, the optimum and validity all judged within 20\% of the budget, the response correction with a one-week pilot lowers regret by \pmRelRedOne\% relative to a trial with the same nominal randomized rider-weeks (\xRelRedOne\% in the stricter pre-specified 10\% band), and on these week-long requests with exact summation a four-week pilot, with \eWeeksRatio\ times fewer randomized rider-weeks, comes within \pmFWmLFour\ \pmFWmLCIFour\ of an \eWeeksLong-week trial. In the registered primary conditions of two studies where windows, populations, rules and binding budgets change from request to request, the framework's regret is below that of a trial with the same nominal rider-weeks and below that of dose interpolation of the same pilot data at every pilot length, and reusing its one-off preparation answers 60 requests \eFSpSixtyOne$\times$ and \eGSpSixtyOne$\times$ faster than re-running the pipeline for each, with identical answers. When the paid window or a selected subpopulation changes behaviour, a window-aware correction lowers one-week regret by 0.195 against direct reuse. Against a nine-offer trial fitted with the framework's own dose curve, one-week regret is 0.055 lower; 73.5\% of the gain over per-offer fits of that trial comes from sharing the tilt across offers. Public retail data and three payment rules confirm the transfer with identical exact outputs.

\end{abstract}
\section{Introduction}
\label{sec:intro}
On-demand meal delivery depends on the people who carry each order from a restaurant to a customer. Large platforms obtain much of this capacity from crowdsourced couriers, usually called \emph{riders}, who decide for themselves when to log in and how many orders to accept \citep{savelsbergh2022crowdsourced,savelsbergh2024update}. Demand, by contrast, peaks at meal times, in bad weather and on holidays. When too few riders are active, orders go unaccepted and are cancelled; one large platform reports several hundred thousand such cancellations a day \citep{wu2022mealbonus}. Platforms therefore add budgeted \emph{incentive activities} to per-order pay. A typical activity reads ``complete 40 deliveries this week for an extra reward; complete 55 for more'', possibly with a per-completion rate that steps up at each tier or a condition on the number of active days. The tiers are set from recent completions (here, quantiles of the weekly volume of the rider's layer), and this is what can make the activity effective. Threshold incentives change behaviour only while the threshold is still ahead \citep{liu2023thresholdlabor}, and effort is shaped by income targets \citep{chen2022targeteffects,allon2023gig}: a tier the rider would reach anyway pays for no extra work, and a tier far out of reach is ignored. A well-placed tier sits just above the rider's usual volume, so that a little extra effort, freely chosen, earns a clearly stated reward, with the aim of serving customers faster when service is scarce. The planning task is to place every tier at that point, for thousands of riders, within a fixed budget.

Operators express this task as a request. They name a rider population, an activity window (any subset of the coming days), an activity rule and a budget band, and they need a plan days before the window opens, within minutes. A week of $n$ days has $2^n$ windows, so one model per window is impossible and forecasts must be composable; payments are step functions of completions, attendance and thresholds, and rules differ between activities; and the campaigns that matter most fall on holidays and days of severe weather, where history is thinnest. Placing tiers also requires the response of completions to each candidate offer, and the clean way to measure it, a randomized controlled trial (RCT), is slow. An offer's effect is observed only after its window has closed and settled, and each group--offer cell needs many rider-weeks: nine offers for five rider groups, with 1,024 randomized observations per cell at 512 riders per group and week, occupy \eWeeksLong\ weeks, longer than the campaign the trial should inform. Experimental data are scarce while no-offer history is abundant, a tension well known in causal inference \citep{kallus2018removing,athey2020selection,athey2025surrogate}.

We present a request-driven framework built from four stages and seven replaceable modules (Figure~\ref{fig:pipeline}). \emph{Prediction} (M1--M3) turns pre-offer histories into per-rider, per-day completion laws from which any window is composed. \emph{Reduction} (M4) groups riders into a few activity layers, each receiving one plan, which reduces a portfolio search over individuals to a small multiple-choice problem. \emph{Integration} (M5--M6) corrects the no-offer law for the offered incentive and compiles joint trajectories into award probabilities: every rule, however complex, reduces to probabilities of nested award events and marked moments, from which expected payment and uplift follow (Section~\ref{sec:accounting}). \emph{Allocation} (M7) selects one plan per layer inside the budget band and returns executable rewards. Every element of the request is a hyperparameter of a downstream stage, so a new period, population, rule or budget reuses every upstream output whose inputs are unchanged. The correction in M5 is designed to shorten the long trial: abundant no-offer history supplies the shape of each group's trajectory law, a pilot of one to four weeks at three offers supplies a few offer-specific moments, and history reweighted to match those moments stands in for every offer.
\paragraph{Contributions.}
(i) A general framework whose four stages exchange fixed intermediates---per rider-day laws, activity layers, award probabilities and per-layer plan menus---so that a new period, population, activity rule or budget is answered by changing hyperparameters while the response law it relies on stays valid, and each stage can be replaced and priced on its own (Sections~\ref{sec:problem}--\ref{sec:accounting}).
(ii) A history-anchored response correction from a short pilot, tested on controlled laws against a trial with the same nominal randomized rider-weeks and an \eWeeksLong-week trial, and in two registered studies in which windows, populations, rules and binding budgets change from request to request, one with a new rider-level generator whose response depends on the rule, and against stronger controls, including the trial fitted with the framework's own dose curve, and behaviour-changing windows and subpopulations (Section~\ref{sec:synthetic}).
(iii) A four-term loss decomposition---synthetic data, moment matching, integration and decision---with a conditional end-to-end perturbation bound on a fixed plan menu, constants that cannot be improved from the final tables, lower bounds attained by the statistical chain itself, and, for each stage, instances on which omitting it leaves an error floor that the other stages cannot remove; with the ablations, this shows that every stage of the combination is needed (Section~\ref{sec:theory}).
(iv) An evaluation against the practices a platform would run instead, reporting the accuracy given up, the time saved and the accuracy on cities left out of training (Sections~\ref{sec:rider}--\ref{sec:generality-evidence}).

\begin{figure}[t]
\centering
\includegraphics[width=\linewidth]{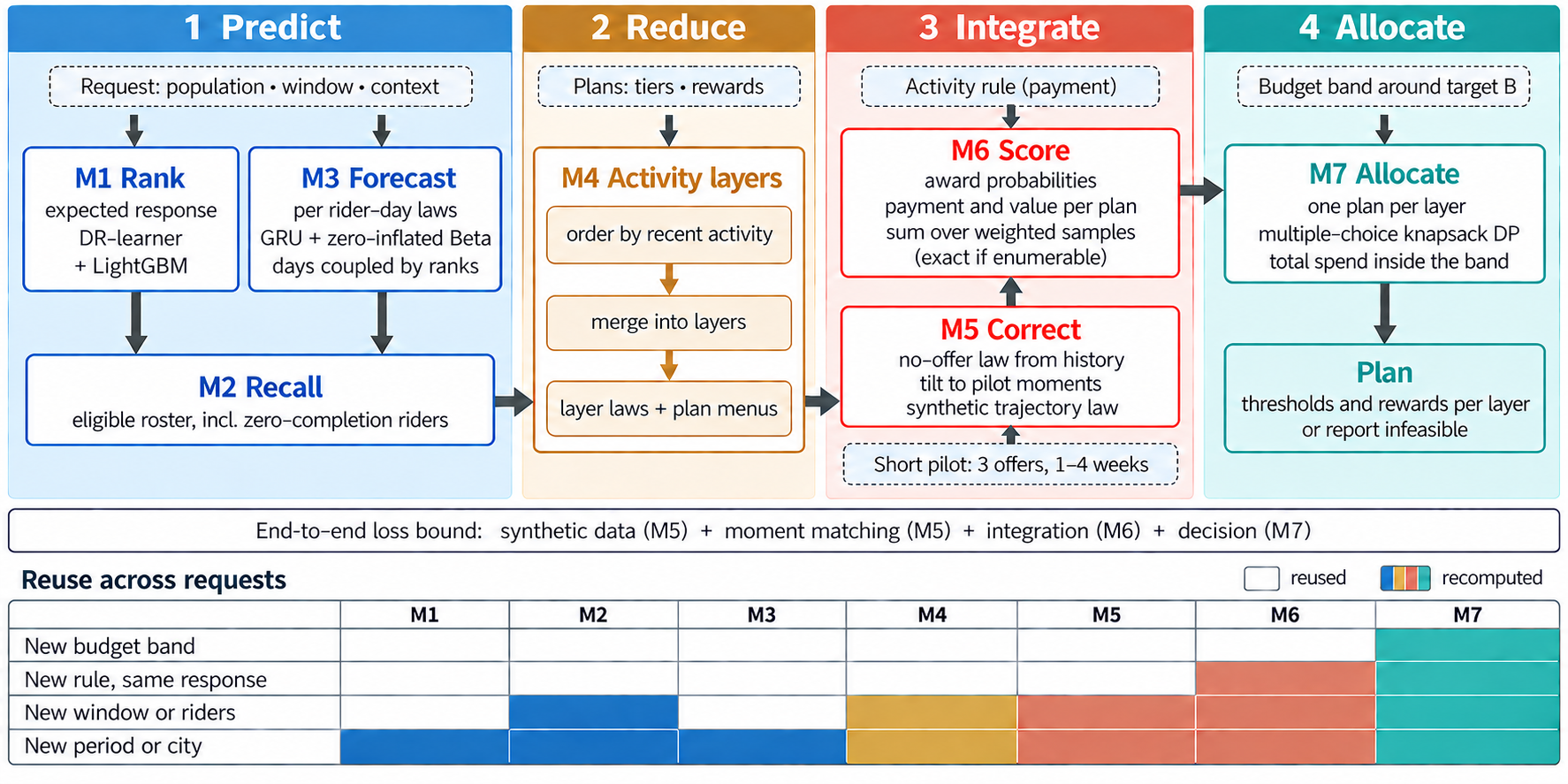}
\caption{Request-driven framework: four stages and seven modules. Dashed chips are inputs (request fields, candidate plans, activity rule, budget band, short pilot); arrows show data flow between modules. \emph{Predict} (M1--M3) gives every eligible rider a conditional completion law; \emph{Reduce} (M4) merges riders into activity layers with plan menus; \emph{Integrate} (M5--M6) tilts the no-offer law to short-pilot moments and scores every plan by a weighted sum over stored samples of the corrected law, an exact sum when the trajectory types can be enumerated; \emph{Allocate} (M7) chooses one plan per layer inside the two-sided budget band and returns the rewards that attain it. The bar names the four terms of the end-to-end bound (Proposition~\ref{prop:four-term}); the table marks what a new request recomputes.}
\label{fig:pipeline}
\end{figure}

\section{Requests and the four-stage framework}
\label{sec:problem}\label{sec:generality}\label{sec:pipeline}
\paragraph{Request.}
At issue time $t_0$ a request $r=(\mathcal I,W,x,\Gamma,\mathcal B)$ names an eligible population $\mathcal I$, an activity window $W$ (a subset of the forecast days), context $x$ (meal periods, weather, supply and demand), an activity rule $\Gamma$ and a budget set $\mathcal B$. A partition $\mathcal G=\{G_1,\ldots,G_L\}$ of $\mathcal I$ assigns plan $a_l\in\mathcal A_l$ to group $l$; a plan specifies ordered thresholds, rewards or rates, eligible orders and attendance conditions. Let $Z_i(a)$ be rider $i$'s activity trajectory in $W$ under plan $a$, $g_i(Z_i(a),a)$ the payment defined by the rule, and $V_i(a)$ an outcome such as completed deliveries. With $H_{t_0}$ the pre-offer history,
\begin{equation}
 c_l(a)=\sum_{i\in G_l}\E[g_i(Z_i(a),a)\mid H_{t_0},x],\qquad
 u_l(a)=\sum_{i\in G_l}\E[V_i(a)-V_i(0)\mid H_{t_0},x],
 \label{eq:cost-uplift}
\end{equation}
where plan $0$ is the no-offer baseline. The planner solves the multiple-choice problem
\begin{equation}
 \max_{a_l\in\mathcal A_l}\ \sum_l\widehat v_l(a_l)
 \quad\text{subject to}\quad \sum_l\widehat c_l(a_l)\in\mathcal B,
 \label{eq:allocation}
\end{equation}
with $v_l=u_l$ for pure uplift, or $v_l$ adding a monetary score for the probability that group spend stays within $\pm20\%$ of its prediction. The band is part of the request: operations tolerate $\mathcal B=[0.8B,1.2B]$, a stricter $[0.9B,1.1B]$ keeps spend closer to the budget, and an upper cap $\mathcal B=[0,B]$ is a special case.

\paragraph{Why four stages.}
Each stage removes one source of hardness. Per rider-day laws remove the $2^n$ windows: one forecast composes any window. Layers remove the individual search: $L$ layers with $|\mathcal A_l|$ plans each leave $\prod_l|\mathcal A_l|$ portfolios instead of one plan per rider. Award probabilities remove the dependence on the rule: every rule is integrated from the same trajectories. Budget allocation removes the coupling between layers: a single program over the band replaces a search over portfolios. The price of each removal is an approximation error. Section~\ref{sec:theory} shows that the errors of the estimation and allocation stages add at the allocation interface on a fixed menu; the price of reduction, restricting riders to layer menus, is measured rather than bounded.

\paragraph{Four stages, seven modules.}
The framework maps a request through fixed intermediates (Algorithm~\ref{alg:request} in Appendix~\ref{app:accounting} lists the steps),
\begin{equation}
 (H_{t_0},r)\ \xrightarrow{\ \text{M1--M3}\ }\ \widehat P_{i,d}\ \xrightarrow{\ \text{M4}\ }\ \mathcal G,\{\mathcal A_l\}
 \ \xrightarrow{\ \text{M5--M6}\ }\ \{\widehat c_l(a),\widehat v_l(a)\}\ \xrightarrow{\ \text{M7}\ }\ \widehat a .
 \label{eq:generic-interface}
\end{equation}
\emph{Prediction.} M1 ranks riders by expected response, M2 selects the eligible roster and M3 forecasts a completion distribution for every rider and day. Day-level laws are coupled across days by reordering each rider's forecast days according to the ranks observed in the rider's own history, so that any window $W$ is composed from one forecast pass. \emph{Reduction.} M4 orders riders by recent activity, merges them into contiguous activity layers and shrinks each menu only in ways that cannot remove a feasible optimum: under a sole cap it drops a plan when another costs no more and is worth no less; under a band with a floor only when another plan of the same cost is worth no less, since a cheaper plan may fall below the floor (Lemma~\ref{lem:admissible-reduction}). \emph{Integration.} M5 turns the no-offer law of each layer into an offer-conditioned trajectory law (below), and M6 integrates the rule over it as a weighted sum over stored trajectories, which is the exact sum when the trajectory types can be enumerated,
\begin{equation}
 (\widehat c_l(a),\widehat v_l(a))=
 \sum_{s}\bar w_s\bigl(g_l(Z_l^{(s)}(a),a),h_l(Z_l^{(s)}(a),a)\bigr),
 \quad \bar w_s\ge0,\quad\sum_s\bar w_s=1,
 \label{eq:generic-score}
\end{equation}
with $\bar w_s$ the law's probability of type $s$, or the normalized law weight of stored joint sample $s$ when types cannot be enumerated.
\emph{Allocation.} M7 solves Eq.~\eqref{eq:allocation} by exact enumeration or by a dynamic program over a budget grid, a multiple-choice knapsack \citep{sinha1979multiple}; along affine reward paths the program runs over piecewise-affine score segments and returns the same objective and reward backpointers as dense exact allocation (Appendix~\ref{app:segment-allocation}). Every returned plan carries rewards that attain its reported payment.

\paragraph{Synthetic trajectories from a short pilot.}
Let $p^0_l(z)$ be the no-offer law of layer $l$ over trajectory types $z$, estimated from history, and $\phi(z)$ low-dimensional trajectory features (normalized completions and active days). For each piloted offer $a$ the correction is the minimum relative-entropy law that reproduces the pilot moments $\bar\phi_{l,a}$,
\begin{equation}
 q_{l,a}(z)\ \propto\ p^0_l(z)\exp\bigl(\lambda_{l,a}^\top\phi(z)\bigr),\qquad \E_{q_{l,a}}[\phi]=\bar\phi_{l,a},
 \label{eq:tilt}
\end{equation}
an exponential tilt solved by Newton's method on the convex dual \citep{csiszar1975idivergence}. Offers between pilot points interpolate $\lambda$ piecewise linearly in the reward increment, with $\lambda=0$ at the no-offer plan. Integration then sums the rule over $q_{l,a}$ by Eq.~\eqref{eq:generic-score}, with no redraw; redrawing $S$ trajectories from $q_{l,a}$ only adds Monte Carlo error and is an ablation. History fixes the shape of the law and the pilot fixes how the offer moves it, so a pilot of one to four weeks at three offers is meant to stand in for a trial over every offer (Section~\ref{sec:synthetic} measures how far it does).

\paragraph{Generality.}
A new window selects other day coordinates of the same forecast; a new population re-runs M2--M4 on stored laws; a new budget re-runs M7 only; a new rule changes the payment function in M6 and, when it changes behaviour, the pilot moments in M5. Reusing stored laws is exact for computation; reusing a response correction assumes the response is unchanged, so a new city, season or rule that may change behaviour calls for its own pilot. Any activity whose payment is a function of the trajectory enters through Eq.~\eqref{eq:generic-score}, so fixed prizes, per-completion rates, attendance awards and their combinations share one pipeline. Group-additive allocation assumes no response interference between groups; joint chance constraints would need an additional state in M7.

\section{Award probabilities and the four-term loss decomposition}
\label{sec:accounting}\label{sec:theory}
\paragraph{One intermediate for every rule.}
For one rider and plan, let $E_1\supseteq\cdots\supseteq E_J$ be nested award events, including attendance and eligibility conditions, with award probabilities $p_j=\Pr(E_j\mid a,H_{t_0},x)$ and $p_{J+1}=0$. If only the highest achieved fixed prize $R_j$ is paid,
\begin{equation}
 \E[g]=\sum_{j=1}^J R_j(p_j-p_{j+1}); \label{eq:fixed-payment}
\end{equation}
if the activity pays rate $\rho_j$ on every eligible completion at the highest achieved tier,
\begin{equation}
 \E[g]=\sum_{j=1}^J(\rho_j-\rho_{j-1})\E[Y\mathbf1\{E_j\}],\qquad \rho_0=0, \label{eq:rate-payment}
\end{equation}
with $Y$ the eligible completions. Award probabilities and award-marked moments therefore carry everything a rule needs; the rule itself selects which of them are integrated. Expected spend is additive over riders without any independence assumption, whereas its variance and the probability of staying in a band depend on the joint law, which is why M6 integrates joint trajectories rather than marginals: rules with equal expected payment and award probability can differ in budget risk (Appendix~\ref{app:accounting}).

\paragraph{Four-term loss decomposition.}
For every layer $l$ and plan $a$ the framework passes through a chain of trajectory laws: the true offer law $P$; the best synthetic law $P^{\mathrm{syn}}$, which applies Eq.~\eqref{eq:tilt} to the true no-offer law with population moments at the piloted offers and the framework's own interpolation of $\lambda$ elsewhere, so that the interpolation bias at unpiloted offers belongs to this term and does not shrink with the pilot; the moment-matched law $P^{\mathrm{mm}}$, which uses the estimated no-offer law and the pilot moments; and the integrated law $P^{\mathrm{int}}$, which is $P^{\mathrm{mm}}$ under exact summation and a finite-sample law otherwise (weighted stored samples, or $S$ redrawn trajectories in the sampled ablation). Law $k$ induces tables $(c^{(k)}_l,v^{(k)}_l)$, and allocation finally works with tables $(c^{\mathrm{dec}},v^{\mathrm{dec}})$ on its budget grid, which equal $(c^{\mathrm{int}},v^{\mathrm{int}})$ under exact enumeration. For $k\in\{\mathrm{syn},\mathrm{mm},\mathrm{int},\mathrm{dec}\}$ with predecessor $k^-$ in this chain, let
$\varepsilon^{c}_{k}=\sum_l\max_{a\in\mathcal A_l}|c^{(k)}_l(a)-c^{(k^-)}_l(a)|$, define $\varepsilon^{v}_{k}$ in the same way, and put $E_c=\sum_k\varepsilon^c_k$ and $E_v=\sum_k\varepsilon^v_k$. With $\mathcal B_\delta=[\underline b+\delta,\bar b-\delta]$, let $V^\star(\delta)$ be the best true value among portfolios whose true expected spend lies in $\mathcal B_\delta$, and $\omega(\delta)=V^\star(0)-V^\star(\delta)$ the value carried by the outer margin $\delta$ of the band.

\begin{proposition}[Four-term decomposition on a fixed menu]\label{prop:four-term}\label{prop:upper}\label{prop:lower}
Suppose some portfolio has true expected spend in $\mathcal B_{2E_c}$, and the planner returns $\widehat a$ with $c^{\mathrm{dec}}(\widehat a)\in\mathcal B_{E_c}$ and $v^{\mathrm{dec}}(\widehat a)\ge\max\{v^{\mathrm{dec}}(a):c^{\mathrm{dec}}(a)\in\mathcal B_{E_c}\}-\eta$.
(i) \emph{Upper bound.} The true expected spend of $\widehat a$ lies in $\mathcal B$, and
\begin{equation}
 V^\star(0)-v(\widehat a)\ \le\ \omega(2E_c)+2\bigl(\varepsilon^v_{\mathrm{syn}}+\varepsilon^v_{\mathrm{mm}}+\varepsilon^v_{\mathrm{int}}+\varepsilon^v_{\mathrm{dec}}\bigr)+\eta .
 \label{eq:four-term}
\end{equation}
On a finite menu $V^\star$ is a step function, so the relevant control is a discrete margin: $\omega(2E_c)=0$ when an optimal portfolio has true spend in $\mathcal B_{2E_c}$, and if $\omega(\delta)\le\kappa\delta$ on $[0,2E_c]$ the right side is at most $\sum_k2(\kappa\varepsilon^c_k+\varepsilon^v_k)+\eta$.
(ii) \emph{Lower bound for tables.} For any nonnegative stage errors there are table instances on a fixed menu with exactly these errors on which Eq.~\eqref{eq:four-term} holds with equality and on which every deterministic planner that sees the final tables and keeps the true spend in $\mathcal B$ for all truths consistent with them loses at least $\omega(2E_c)+2E_v$. Without the tightening, the true spend can leave $\mathcal B$ by $E_c$.
(iii) \emph{Error floors of omitted stages.} There are instances on which, with the other stages exact, omitting one stage leaves an error that their accuracy cannot reduce: without the pilot correction the entire uplift is invisible; two matched moments without the history shape leave the probability of reaching a threshold $t$ standard deviations above the mean undetermined within $1/(1+t^2)$; plug-in means instead of trajectory integration misprice a threshold prize by half; and an equal budget split can forfeit the entire value.
\end{proposition}

The proposition is a conditional perturbation bound on true expected spend, given the stage errors and the tightened band. It is relative to the layer menus: the value given up by restricting individuals to layer menus enters only through $V^\star$, and on rider records alternative layerings are compared directly (Table~\ref{tab:ledger}). Part~(ii) shows that no planner seeing only the final tables can improve the constants. For the chain of laws itself, the I-projection attains the constant of the synthetic term exactly, and the pilot and sampling terms are attained in order but not in constant (Proposition~\ref{prop:chain-lower}).

\paragraph{Why the combination is accurate and fast.}
Costs and pure-uplift values are expectations, hence linear in the trajectory law, so errors of the four stages add at the allocation interface instead of compounding: a value error costs at most twice its size and a cost error at most $2\kappa$ times its size under the growth control $\omega(\delta)\le\kappa\delta$ (Eq.~\eqref{eq:four-term}). Each term is driven by a different resource and shrinks independently: the synthetic term by the part of the response that the matched features do not explain, the moment term as $n^{-1/2}$ in pilot size for a given base (an estimated base adds the term of Lemma~\ref{lem:joint-base}), the integration term as $S^{-1/2}$ and to zero under exact summation, and the decision term by one grid step per layer with upward rounding, half a step with nearest rounding (Appendix~\ref{app:loss-proofs}). Corollary~\ref{cor:rates} states the rates in expectation: they hold for entries with a fixed integrand, such as uplift and payment, and, under the margin condition of Corollary~\ref{cor:rates}(b), for a stability score centred at the estimated mean. The stage errors belong to the compiled tables, not to a request, so every budget and weight answered from the same tables inherits the same bound; this is why paying for the upstream stages once gives large speed-ups at no additional loss. On a given instance errors may partly cancel, so the ablations below measure what each stage buys in practice.

\section{Rider records: accuracy given up and time gained}
\label{sec:rider}
\paragraph{Data and protocol.}
The rider panel holds daily completed deliveries of 3,000 riders over 595 days in 2025--2026 across 17 separately modelled cities and one pooled node. We replay the framework at 45 weekly origins with three sampling seeds: models are refit at each origin on data before it, and each seven-day window is planned before its outcomes are read. Plans combine three tier shapes, six active-day requirements and six reward levels, 108 plans per layer, so five activity layers give $108^5\approx1.5\times10^{10}$ portfolios; budgets are upper caps at $\beta=0.5$, 1.0 and 1.5 times the realized cost of a reference plan, so they are set in hindsight. Because every plan is scored against the completions riders actually recorded, the replay measures what a plan must get right before any behavioural response: who attains which tier and what the plan costs. A plan's value is the number of deliveries completed by riders who reach an award, its cost the reward paid; plan-cost error is the median absolute relative error between predicted and realized plan cost over plans with nonzero realized cost (the \dfZeroPlans\ of \dfPlansTotal\ plans that pay nothing are excluded), and regret is measured against the exact optimum on the realized records. Recorded completions do not respond to a plan, so value does not depend on the reward level, and the budget binds in hindsight only in the Spring Festival week (\dfBindHalf\ of \dfUnits\ units at $\beta=0.5$). The replay therefore tests pricing and computation (M3: gradient-boosted quantiles with rider shares, Appendix~\ref{app:rider}); Section~\ref{sec:synthetic} tests the choice of reward level and allocation under a binding budget. Intervals resample origins (Appendix~\ref{app:rider}).

\paragraph{What the framework is compared with.}
We compare with the most accurate practices a platform would otherwise run, each exact or unbiased for the step it replaces: re-drawing trajectories for the requested window and solving exactly, request by request; one exact plan per rider; a mixed-integer program per request \citep{sinha1979multiple} solved to optimality \citep{huangfu2018highs}; and a trial long enough to observe every offer. Table~\ref{tab:baselines} reports the accuracy given up and the time saved.

\begin{table}[t]
\caption{Accuracy given up and time gained against the most accurate way to answer the same request without the framework, on the same inputs. Speed-up: median of paired ratios (mixed-integer program: ratio of geometric-mean times); brackets: interquartile range or 95\% interval over laws.}
\label{tab:baselines}\vspace{3pt}
\centering\small\setlength{\tabcolsep}{3pt}\renewcommand{\arraystretch}{1.05}
\begin{tabular}{@{}>{\raggedright\arraybackslash}p{0.225\linewidth}>{\raggedright\arraybackslash}p{0.335\linewidth}>{\raggedright\arraybackslash}p{0.205\linewidth}>{\raggedright\arraybackslash}p{0.19\linewidth}@{}}
\toprule
Alternative per request & Accuracy the framework gives up & Time: alternative $\to$ framework & Speed-up\\
\midrule
\multicolumn{4}{@{}l}{\emph{Rider records: 3,000 riders, 45 weekly origins, five activity layers}}\\
Re-draw trajectories and solve exactly & none in the scenarios (\blBitwise/\blBitwiseOf\ requests identical); allocation loses value in \blDPnz/\blDPn\ instances, at most \blDPmax\% & \blReqAlt\,s $\to$ \blReqOurs\,s per week-long request & \blReqRatio$\times$ [\blReqRatioLo, \blReqRatioHi]; \blKAll$\times$ over 127 requests with the one-off pass\\
One plan per rider, exact & not measurable: the exact plan did not finish within 3\,hours and \blPerRiderGiB\,GiB of memory & unfinished $\to$ \blTierTable\,s per table & ---\\
\midrule
\multicolumn{4}{@{}l}{\emph{Controlled response laws: \blRctLaws\ laws $\times$ 8 replicates, 33 budgets each (Section~\ref{sec:synthetic})}}\\
Randomized trial over all offers & relative regret (planned and scored within 20\% of budget) \pmLong\ $\to$ \pmFWFour\ (\pmFWmLFour\ \pmFWmLCIFour) with a 4-week pilot; a trial with the pilot's rider-weeks: \pmSCOne\ with 1 week (framework \pmFWOne) and \pmSCFour\ with 4 & 9,216 $\to$ 2,048 nominal randomized rider-weeks per group (18 $\to$ 4 weeks) & 4.5$\times$ fewer\\
\midrule
\multicolumn{4}{@{}l}{\emph{Public retail (M5) and three payment rules (Section~\ref{sec:generality-evidence})}}\\
Re-draw and solve exactly (M5, 3,000 series) & none: identical scenarios, value lost in \blMfLossNz/\blMfLossN\ tables & \blMfAlt\,s $\to$ \blMfOurs\,s per request & \blMfRatio$\times$ [\blMfRatioLo, \blMfRatioHi]; \blMfKAll$\times$ over 127 with preparation\\
Mixed-integer program per request, model kept between requests (\blMipCond\ conditions) & none: the budget program is exact on \blMipDPExact/\blMipReq\ requests; the MIP on \blMipRawExact\ (raw) and \blMipEnvExact\ (envelope) & paired on conditions both solve exactly & 1, 9, 33 requests: \blMipRawOne, \blMipRawNine, \blMipRawTT$\times$ (envelope \blMipEnvOne, \blMipEnvNine, \blMipEnvTT$\times$)\\
\bottomrule
\end{tabular}
\end{table}

\paragraph{Near-lossless and much faster.}
Against re-drawing and solving exactly for every request, the framework gives up nothing in the trajectories and almost nothing in the allocation, while a week-long request takes \rReqOurs\,s instead of \rReqBest\,s. The \rReqOurs\,s follow a one-off sampling pass of \dfOffline\,s per origin, the only cold-start cost. Charging that pass to the framework, eight windows are answered \rInclEight$\times$ faster and all 127 windows of a week take \rOursAll\ instead of \rBaseAll\ minutes, \rInclAll$\times$ faster (Figure~\ref{fig:evidence}b; \rOriginsAll\ timed origins, fixed before the run); solving one plan per rider is out of reach, which is what population reduction buys. The gain is reuse of the sampling pass, and it follows from the interface: the per-day laws are fixed before a request arrives, so one forecast serves every request (Table~\ref{tab:time-attribution}).

\begin{table}[t]
\caption{Rider records, 45 weekly origins $\times$ 3 seeds, five activity layers. Each row replaces one stage of the full framework by a cheaper alternative. Ratios are of origin means with 95\% intervals over origins. Regret at budget fraction 0.5. $^{*}$45 origins $\times$ 3 budget fractions.}
\label{tab:ledger}
\centering\small\setlength{\tabcolsep}{2.5pt}
\begin{tabular}{llll}
\toprule
Stage & Full framework & Cheaper alternative & Accuracy change\\
\midrule
Prediction & per rider-day distribution & point forecast & plan-cost error $\times$\rPointRatio\ \rPointCI\\
Reduction & activity layers & equal-width buckets & regret \rRegFullHalf\% $\to$ \rRegEqwHalf\%\\
Integration & joint law across days & independent days & plan-cost error $\times$\rIndRatio\ \rIndCI\\
 & & simplified normal & $\times$\rNorRatio\ \rNorCI; 5--6-day award \rNorAwardDiff\,pp\\
 & & plug-in mean & $\times$\rPlgRatio\ \rPlgCI\\
Allocation & budget-grid program & equal budget split & value $-$\rEqLoss\% \rEqCI\\
 & & exact solver (reference) & loss 0 in \nmDPfbZero/\nmDPfbN$^{*}$; max \nmDPfbMax\%\\
\bottomrule
\end{tabular}
\end{table}

\paragraph{Every stage earns its place.}
Table~\ref{tab:ledger} replaces one stage at a time by the cheaper practice it displaces. Distributional prediction is the largest single gain: a point forecast raises plan-cost error from \rCostFull\ to \rCostPoint\ and median award-rate error from \rAwardFull\ to \rAwardPoint\ percentage points. Joint integration is next: composing days independently raises plan-cost error by \rIndRatio$\times$, and a simplified normal baseline defined here (window totals normal, tiers priced without the truncated-moment term) by \rNorRatio$\times$; its award rates for five- and six-day plans fall \nmNorAwardDiff\ points below those of the joint law, exactly where attendance persists across days. Equal-width buckets change the layer menus and raise regret against the same hindsight optimum from \rRegFullHalf\% to \rRegEqwHalf\%, and splitting the budget equally across layers gives up \rEqLoss\% of the value found by the budget program. On the realized five-layer records the grid program equals the exact optimum in \nmDPfbZero\ of \nmDPfbN\ main-grid instances (largest loss \nmDPfbMax\%; Appendix~\ref{app:segment-allocation}).

\paragraph{Cities left out of training.}
The replay refits every model weekly on all cities, as the framework does in operation. A held-out condition fixed before the runs removes the evaluated cities from training (Appendix~\ref{app:holdout}): cities removed in four folds (\hoSOrigins\ origins) and forecast from their riders' own history keep plan-cost error at \hoSHeld\ against \hoSRef\ for the reference that trains on them (paired change \hoSD\ \hoSDCI), within the pre-set margin of \hoMargin. On these riders every stage replacement again loses accuracy (Table~\ref{tab:holdout-stages}).
\section{Synthetic trajectories and a short pilot versus a long trial}
\label{sec:synthetic}
\paragraph{Design.}
Behavioural response cannot be read from logged completions, so we test the response-correction step M5 on controlled response laws where the truth is known; every other stage is fixed. We draw 24 new laws in four families of six (saturating, threshold and hump responses combined; saturating only; threshold only; none), each with five rider groups under three payment rules, nine offers per group and 32 weekly trajectory types; allocation enumerates all $9^5$ portfolios for 33 budgets, and every request covers the whole week and all five groups. Every replicate (eight per law) gives each method its own data: 4,096 no-offer weeks per group; a \emph{long RCT} of 1,024 randomized observations per group and offer, which at 512 riders per group and week takes \eWeeksLong\ weeks; a \emph{same-calendar RCT} that spreads the rider-weeks of a pilot evenly over all nine offers; and the \emph{framework}, which runs the pilot at three offers for one, two or four weeks, corrects history by Eq.~\eqref{eq:tilt} and, since the 32 types can be enumerated, sums the rule exactly over the corrected law, which removes the integration term (Lemma~\ref{lem:int}); drawing $S=2{,}048$ trajectories instead is an ablation. Samples are independent across indices, each with one weekly shock shared by all cells; weeks convert nominal rider-weeks at 512 riders per group. The endpoint is relative regret of pure uplift against the true optimum in the request's band; an abstention scores zero, and so does a plan whose true expected spend leaves the band. We report the operating band $[0.8B,1.2B]$, in which every method plans and is scored and the optimum is taken (chosen after the runs), and the pre-specified, stricter $[0.9B,1.1B]$ alongside, which the appendix tables of this design use. Laws are the unit of inference (\eLaws\ laws with a response; the \eNullLaws\ null laws check validity only), with 95\% intervals by bootstrap over laws. Arms were specified before the runs, after one disclosed development run whose exclusion changes no conclusion (Appendix~\ref{app:synthetic}).

\begin{table}[t]
\caption{Relative regret (lower is better), 18 response laws $\times$ 8 replicates $\times$ 33 budgets, pilot of 1/2/4 weeks; planning, optimum and scoring within 20\% of the budget or 10\% (pre-specified). Indented rows change one stage (others: Table~\ref{tab:e2-ablations}). Validity (\%): issued plans with true expected spend in the 10\% band, all 24 laws.}
\label{tab:synthetic}
\centering\small\setlength{\tabcolsep}{3.5pt}
\begin{tabular}{lcccccc}
\toprule
 & \multicolumn{3}{c}{Within 20\% of budget} & \multicolumn{3}{c}{Within 10\% (pre-specified)}\\
\cmidrule(lr){2-4}\cmidrule(l){5-7}
Method & 1 wk & 2 wk & 4 wk & 1 wk & 2 wk & 4 wk\\
\midrule
Long RCT, \eWeeksLong\ weeks (business alternative) & \multicolumn{3}{c}{\pmLong} & \multicolumn{3}{c}{\xLong}\\
Same-calendar RCT & \pmSCOne & \pmSCTwo & \pmSCFour & \xSCOne & \xSCTwo & \xSCFour\\
\textbf{Framework} (tilt, exact sum) & \textbf{\pmFWOne} & \textbf{\pmFWTwo} & \textbf{\pmFWFour} & \textbf{\xFWOne} & \textbf{\xFWTwo} & \textbf{\xFWFour}\\
\quad sampled integration ($S=2{,}048$) & \pmSampOne & \pmSampTwo & \pmSampFour & \xSampOne & \xSampTwo & \xSampFour\\
\quad without history (pilot frequencies) & \pmNoSynOne & \pmNoSynTwo & \pmNoSynFour & \xNoSynOne & \xNoSynTwo & \xNoSynFour\\
Validity (framework / same-calendar) & & & & \xValFWOne/\xValSCOne & \xValFWTwo/\xValSCTwo & \xValFWFour/\xValSCFour\\
\bottomrule
\end{tabular}
\end{table}

\paragraph{Results.}
Table~\ref{tab:synthetic} gives three findings; the pre-specified 10\% band is in parentheses and is the one plotted in Figure~\ref{fig:evidence}a (Appendix~\ref{app:synthetic}). First, for the same nominal randomized rider-weeks the framework is far better than randomizing the pilot over all offers: regret falls by \pmFWmSabsOne\ \pmFWmSCIOne\ with a one-week pilot, \pmRelRedOne\% lower (\xRelRedOne\%), by \pmFWmSabsTwo\ \pmFWmSCITwo\ with two weeks and by \pmFWmSabsFour\ \pmFWmSCIFour\ with four (at 10\%, intervals also below zero). Second, with exact summation on these week-long requests, a four-week pilot, with \eWeeksRatio\ times fewer randomized rider-weeks, comes within \pmFWmLFour\ \pmFWmLCIFour\ of the \eWeeksLong-week trial (\xFWmLFour\ \xFWmLCIFour); sampled integration, the pre-specified arm: Table~\ref{tab:prereg}. Third, each stage is needed against the replacement tested for it (10\% scoring): dropping history, moment matching or joint integration raises one-week regret by \xNoSynDOne, \xNoMomDOne\ and \xPlugDOne, and an equal budget split on the exact tables raises it by \exEqTenOne, each interval excluding zero; drawing 2,048 trajectories instead of summing exactly raises it by \pmSampDOne\ \pmSampDCIOne\ (\xSampDOne\ \xSampDCIOne).
 With the truth known, only the moment-matching term of Eq.~\eqref{eq:four-term} shrinks with the pilot, at the $\sqrt2$ rate per doubling of Corollary~\ref{cor:rates}(a) (Table~\ref{tab:loss-check}).

\begin{table}[t]
\caption{Two registered studies with changing requests and binding budgets (registered primary condition, 36 response laws each; relative regret within 20\% of the budget). A: fresh laws from the generator of Table~\ref{tab:synthetic}; B: a new rider-level generator with a rule-dependent response. Speed-up: 60 requests, one-off preparation charged, against re-running the pipeline per request.}
\label{tab:studies}
\centering\small\setlength{\tabcolsep}{3.2pt}
\begin{tabular}{lcccccc}
\toprule
 & \multicolumn{3}{c}{A: integrated study} & \multicolumn{3}{c}{B: new generator}\\
\cmidrule(lr){2-4}\cmidrule(l){5-7}
Pilot & 1 wk & 2 wk & 4 wk & 1 wk & 2 wk & 4 wk\\
\midrule
Same-calendar RCT & \eFSCOne & \eFSCTwo & \eFSCFour & \eGSCOne & \eGSCTwo & \eGSCFour\\
\textbf{Framework} & \textbf{\eFFWOne} & \textbf{\eFFWTwo} & \textbf{\eFFWFour} & \textbf{\eGFWOne} & \textbf{\eGFWTwo} & \textbf{\eGFWFour}\\
Dose interpolation (same information) & \eFCOne & \eFCTwo & \eFCFour & \eGCOne & \eGCTwo & \eGCFour\\
Independent days (least costly replacement) & \eFIndOne & \eFIndTwo & \eFIndFour & \eGIndOne & \eGIndTwo & \eGIndFour\\
Budgets that bind (\%, all 48 laws) & \multicolumn{3}{c}{\eFBind} & \multicolumn{3}{c}{\eGBind}\\
Speed-up, 60 requests & \eFSpSixtyOne$\times$ & \eFSpSixtyTwo$\times$ & \eFSpSixtyFour$\times$ & \eGSpSixtyOne$\times$ & \eGSpSixtyTwo$\times$ & \eGSpSixtyFour$\times$\\
\bottomrule
\end{tabular}
\end{table}

\paragraph{Changing requests, binding budgets and a new generator.}
The design above fixes the request. Two further registered studies let behavioural response, changing requests and binding budgets occur together (Table~\ref{tab:studies}; Appendices~\ref{app:integrated} and~\ref{app:structural}). Each replicate issues 60 requests with a random window among the 127 subsets of the week, a random population and rule, and budgets at 0.3, 0.5 and 0.7 times the cost of the request's unconstrained optimum; the framework prepares its laws once per origin, and the preparation is charged. Study A draws \eFLaws\ fresh laws from the generator above. Study B uses a rider-level generator unrelated to it: six groups of 16 latent rider types with day-of-week activity, a response that depends on the rule shown, and weeks that cannot be enumerated, so every request is integrated as a weighted sum over 4,096 stored weeks per group. In the registered primary condition of both studies, both registered hypotheses hold (Table~\ref{tab:prereg}): regret is below the same-calendar trial's at every pilot length (by \eFDOne\ \eFDCIOne\ and \eGDOne\ \eGDCIOne\ with one week), and reusing the prepared laws for 60 requests is faster than re-running the pipeline for each, with identical answers. Every stage replacement loses accuracy, and dose interpolation of the same history and pilot data, a same-information control, has higher regret at every pilot length (\eFdCOne\ \eFdCCIOne\ and \eGdCOne\ \eGdCCIOne\ with one week). In three more registered studies all hypotheses hold (Appendix~\ref{app:strong}): at matched cache the framework beats a history-anchored tilt of the trial fitted per offer and a normal law with exact truncated moments, it beats the trial under shared calendar-week shocks, and request-aware reuse beats direct reuse when windows or subpopulations change behaviour. The third fits the trial jointly with the framework's own dose curve (Table~\ref{tab:joint}): at the same parameterization the framework's regret is lower by 0.055 [0.051, 0.059] at one week (0.052 and 0.060 at two and four), and of its one-week advantage over the per-offer tilt, sharing the tilt across offers accounts for 73.5\% [72.1, 74.8] and placing the pilot at three offers for 26.5\% (68.3\% at four weeks). In a fourth, five history layers retain 94.6\% [94.4, 94.8] of an individual-level optimum, against 82.5\% for one plan per group.

\section{Transfer across domains and rules}
\label{sec:generality-evidence}\label{sec:evidence}
The same modules serve other domains and rules by changing inputs and hyperparameters only; both sides of every comparison get the same laws, menus and budget bands, and every output is checked against exact allocation (Appendix~\ref{app:matched-computation}). M5 sales \citep{makridakis2022m5accuracy,makridakis2022m5uncertainty} follow the same rule as a supplier's tiered rebate to a store. Designed on riders and run unchanged on 3,000 item--store series, law reuse cuts the median request from \blMfAlt\,s to \blMfOurs\,s (paired \blMfRatio$\times$) with identical scenarios and identical exact values on all 405 tables; charging the 80.63\,s preparation to the framework, it is \blMfKEight$\times$, 3.92$\times$ and \blMfKAll$\times$ faster for 8, 32 and 127 requests per origin. An active-day rule reuses the same laws through joint histograms, with all 14,580 answers identical to direct settlement. Across fixed prizes, per-completion rates and attendance awards (396 conditions), the budget program reuses its value curve and resolves all 65,340 requests exactly, whereas a mixed-integer program per request leaves 25 (raw) and 5 (envelope) unresolved and is \blMipRawTT$\times$ and \blMipEnvTT$\times$ slower at 33 requests (Table~\ref{tab:baselines}).

\section{Related work}
\label{sec:related}
Response or uplift models feed budgeted allocation in marketing, sometimes trained with the decision or summarized by budget--value curves \citep{zhao2019unified,albert2022mckp,ai2022lbcf,zhou2024dfcl,sun2024e3ir,zhang2025bidfcl,yang2026unimvt,cong2025tradeoff}; delivery bonuses, driver subsidies and rider incentives are allocated in the same spirit \citep{wu2022mealbonus,chen2024uber,yang2025dflsubsidy,chen2026d3subsidy,chen2025mmce}. Here the decision unit is a multi-day threshold activity whose payment depends on the whole trajectory, and every request may change window, population, rule and budget. Our allocation stage extends multiple-choice knapsacks and graphical Bellman algorithms \citep{sinha1979multiple,kameshwaran2009nonconvex,gafarov2014graphical} to two-sided budget bands and reward traceback. Where small experiments are combined with observational data to estimate effects \citep{kallus2018removing,athey2020selection,athey2025surrogate}, our pilot corrects history for planning by entropy balancing \citep{hainmueller2012entropy} on offer-conditioned trajectories coupled across days by rank \citep{clark2004schaake,schefzik2013ecc}, and its error is priced as its own term of an end-to-end bound.

\section{Conclusion}
\label{sec:discussion}
Four stages exchanging conditional trajectory laws answer changing requests while the response law stays reusable; a window-aware correction covers behaviour-changing windows and subpopulations. The four-term decomposition prices every stage, every cheaper replacement loses accuracy, and reuse answers all windows of a week \rInclAll$\times$ faster at almost no loss. In registered controlled studies the framework beats equal-effort trials, same-information controls and the trial fitted with its own dose curve. Rider records support pricing and computation; response evidence comes from controlled studies.

\FloatBarrier
\phantomsection\label{pg:main-end}
\bibliographystyle{iclr2027_conference}
\bibliography{references}
\appendix
\makeatletter\providecommand{\rawinput}[1]{\input{#1} }\makeatother

\section{Rider replay: protocol and timing}\label{app:rider}
\paragraph{Panel and calendar.}
The panel holds daily completed deliveries of 3,000 riders from 2025-01-01 to 2026-08-18 (595 days), recorded per rider, city and day; a missing rider-day means no recorded completion. Seventeen cities with enough riders are modelled as separate nodes and the rest are pooled. No feature uses information from after the origin. Forecast origins are 45 consecutive weeks; every model is refit at each origin on data up to the origin, and a request window is any subset of the seven days that follow.

\paragraph{Stages as run.}
The replay instance of M3 combines quantile gradient-boosted forecasts of node totals with per-rider share and concentration models into per rider-day completion distributions (Table~\ref{tab:app-modules} lists a recurrent network as another instance of M3); $R$ joint samples per origin and seed are coupled across days by reordering each rider's forecast days according to the ranks observed in the rider's own 182-day history. Reduction forms five activity layers (or three coarser ones) from participation in the preceding 28 days. Integration averages each plan's payment and value over the $R$ joint samples. Allocation solves the multiple-choice problem on a 400-point budget grid with costs rounded up, so every portfolio it finds respects the cap; the exact reference enumerates each layer's menu after cost--value dominance, which is exact here because every budget is a sole cap (Lemma~\ref{lem:admissible-reduction}). The main-grid instances in which the grid program loses value fall in the Spring Festival week.

\paragraph{Plans, value and budget.}
A layer's plan is a ladder of three rungs, an active-day requirement $d\in\{0,2,3,4,5,6\}$ and a reward level $\rho\in\{0.25,0.5,0.75,1,1.5,2\}$ (108 plans; a coarse grid keeps 18 and a day-requirement grid keeps $d\in\{5,6\}$). The rungs are the 50/70/85\%, 60/80/92\% or 40/65/85\% quantiles of the layer's seven-day totals over the 182 days before the origin. With $Y_i$ the rider's window completions, $D_i$ its active days and $J_i$ the highest rung reached when $D_i\ge d$ ($J_i=0$ otherwise), the plan pays $\rho\,m_{J_i}Y_i$ with $m=(0,1,1.6,2.4)$. The value is the award-covered volume $\sum_i\mathbf 1\{J_i\ge1\}Y_i$. It is a pricing target, not the uplift of Eq.~\eqref{eq:cost-uplift}: with recorded completions held fixed, $\rho$ changes only the cost. No stability weight enters the objective. The budget is a cap $B=\beta$ times the realized cost of a reference portfolio (middle ladder, $d=3$, $\rho=1$ in every layer); because this cost is read from the evaluated week, the budget is retrospective, whereas operationally a request fixes $B$ in advance.

\paragraph{Endpoints and uncertainty.}
Plan-cost error is $|\widehat c-c|/c$ for predicted and realized plan cost, summarized by its median over the plans of a unit; the \dfZeroPlans\ plans (\dfZeroShare\% of \dfPlansTotal) with $c=0$ are excluded and counted, all in the dormant layer. Award-rate error is the predicted minus realized share of riders with $J_i\ge1$, over all plans. Regret is $(V^\star-V(\widehat a))/V^\star$, with $V(\widehat a)$ the realized value of the portfolio chosen on the predicted table and $V^\star$ the exact optimum of the realized table of the full framework's layering at the same budget; $V^\star>0$ in every unit. Because every replacement is graded against this one optimum, regret can be negative: a portfolio that overspends can exceed $V^\star$, and a replacement with other layers has other menus. The framework's regret is nonzero only in the two weeks around the Spring Festival (\dfRegMeanHalf\% on average at $\beta=0.5$). Equal split gives every layer $B/L$ and takes, per layer, the plan of largest predicted value whose predicted cost fits the share. The replacements are: a point forecast (rounded per rider-day mean); equal-width buckets of the 28-day active-day count; independent days (the same per-day draws without coupling); a two-moment normal shortcut (per-rider normal window total; each tier's payment and covered volume priced as the mean times the normal tier probability, without the truncated-mean correction; the day requirement judged on mean active days); and a plug-in mean. Each origin is first averaged over its three seeds; 95\% intervals use an origin-block bootstrap (blocks of one and two origins, 20,000 draws, the wider interval reported). Every origin, including the four Spring Festival weeks, is kept.

\paragraph{Timing.}
All timings use one CPU core at 45 origins, seed 1 (Table~\ref{tab:time-attribution}). A framework request reduces the stored per-day draws to the requested days, integrates the rule over the 540 plans and allocates. The alternative re-draws the per-day samples of the requested days with the same random streams, re-couples them and runs the same integration and exact allocation; the re-drawn samples equal the stored ones bit for bit at every request, so the comparison is of time only. Model fitting is needed by both sides once per origin and is shown separately. The first request takes \dfReqFirst\,s against \dfReq\,s for later ones, so the only cold-start cost is the \dfOffline\,s sampling pass. $K=8$ uses all 45 origins; $K=32$ and $127$ use origins \dfOriginsAll, fixed before the run. The complete per-request alternative of Table~\ref{tab:baselines} (\rReqBest\,s) also recomputes the rank template; re-drawing alone takes \dfBase\,s. By window length, the framework and the alternative take 2.39 and 18.5\,s for one day and 5.63 and 101.1\,s for six days.

\begin{table}[ht]
\caption{Where the time of rider requests goes (seconds; medians over origins, one CPU core). Top: one week-long request, by step. Bottom: $K$ requests per origin in the fixed order (whole week, seven single days, remaining subsets); the framework total includes its once-per-origin sampling pass; speed-ups are medians of per-origin ratios. $^{a}$Gradient-boosted node quantiles (29.8\,s) plus share and concentration models (189.1\,s), needed equally by both approaches. $^{b}$Speed-up when model fitting is charged once to both sides.}
\label{tab:time-attribution}
\centering\small\setlength{\tabcolsep}{3pt}
\begin{tabular}{lrr}
\toprule
Step & Framework & Re-draw per request\\
\midrule
Model fitting (once per origin)$^{a}$ & 221.3 once & 221.3 once\\
Per-day draws, 7 days & 108.7 once & 108.7\\
Cross-day rank template & 3.5 once & reused\\
Cross-day coupling & 1.6 once & 1.8\\
Window reduction & 0.54 & 0.54\\
Integration, 540 plans & 2.52 & 2.52\\
Allocation & 0.44 ms (grid) & 0.50 ms (exact)\\
\midrule
Per request, once-per-origin steps excluded & 3.07 & 114.5\\
Same, first request & 2.97 & ---\\
\bottomrule
\end{tabular}

\medskip
\setlength{\tabcolsep}{3pt}\begin{tabular}{rrrrrrrr}
\toprule
$K$ & Origins & First & Total & Per req. & Re-draw & Ratio & +Fitting$^{b}$\\
\midrule
8 & 45 & 117.8 & 136.1 & 17.01 & 245.0 & 1.81 & 1.31\\
32 & 9 & 117.7 & 240.4 & 7.51 & 1410.7 & 5.87 & 3.58\\
127 & 9 & 117.7 & 734.5 & 5.78 & 7864.2 & 11.04 & 8.82\\
\bottomrule
\end{tabular}
\end{table}

\paragraph{Time saved by cheaper shortcuts.}
Table~\ref{tab:shortcuts} adds, for the shortcuts of Table~\ref{tab:ledger}, the time each would save. Drawing and coupling trajectories run once per origin; window reduction, integration (2.52 of the 3.07\,s of a week-long request) and allocation run for every request.
\begin{table}[ht]
\caption{Cheaper shortcuts the framework declines, on rider records (five layers, main plan grid). Accuracy: plan-cost error of the shortcut over the framework's, or value lost (95\% intervals over origins). Time: the framework's stage time over the shortcut's.}
\label{tab:shortcuts}
\centering\small\setlength{\tabcolsep}{3pt}
\begin{tabular}{@{}>{\raggedright\arraybackslash}p{0.11\linewidth}>{\raggedright\arraybackslash}p{0.19\linewidth}>{\raggedright\arraybackslash}p{0.37\linewidth}>{\raggedright\arraybackslash}p{0.24\linewidth}@{}}
\toprule
Stage & Shortcut & Accuracy cost of the shortcut & Time the shortcut saves\\
\midrule
Prediction & point forecast & plan-cost error $\times$\blPointCost\ [\blPointCostLo, \blPointCostHi] & not timed separately\\
Integration & independent days & plan-cost error $\times$\blIndCost\ [\blIndCostLo, \blIndCostHi] & \blCoupleShare\% of sampling; integration $\times$\blIndSpeed\\
Integration & simplified normal & plan-cost error $\times$\blNormCost\ [\blNormCostLo, \blNormCostHi]; 5--6-day plans $\times$\blNormCostFS & integration \blNormSpeed$\times$ (\blIntegS\,s per table)\\
Allocation & equal budget split & value $-$\blEqLoss\% [\blEqLossLo, \blEqLossHi] & \blDPms\ $\to$ \blEqMs\,ms\\
\bottomrule
\end{tabular}
\end{table}

\subsection{Cities left out of training}\label{app:holdout}
The reference refits the forecasting models at every weekly origin on all cities. For the held-out condition the 17 separately
modelled city nodes are ranked by median daily participation and assigned to four folds (5/4/4/4 nodes); the pooled node of small
cities always stays in training. For each fold and each of the 45 origins, all training rows of the fold's cities, and the all-city
aggregate rows, are removed, together with the fold's contexts in the share and concentration models, and the fold's riders are
forecast from their own history only. All later stages, menus and budgets are unchanged, and the seed is that of the reference, so
both use the same random streams. At each origin the held-out result is paired with the reference on the same riders and outcomes;
intervals resample origins ($B=20{,}000$), averaging folds within an origin. The pre-set rule calls a loss immaterial if the upper
95\% bound of the paired increase in plan-cost error at budget fraction 1 is at most $0.045$ (the half-width of the reference's own
interval) and that of selection regret at most $0.01$. All \hoSJobs\ runs finished; nothing is dropped, and a control run that reuses the
reference models at one origin reproduces the reference window exactly. Table~\ref{tab:holdout} gives the result, and
Table~\ref{tab:holdout-stages} the stage replacements on the held-out riders. The design fixed before the runs also contained a
condition with models fitted once and left unrefit; every pre-specified rule and its outcome is listed in Table~\ref{tab:prereg}.

\begin{table}[ht]
\caption{Cities left out of training, rider records (five-band system, budget fraction 1). Plan-cost error is the median absolute relative error of predicted plan cost; change = held-out minus reference on the same riders and outcomes, 95\% interval over origins. Pre-set margin: upper bound $\le$ \hoMargin.}
\label{tab:holdout}
\centering\footnotesize\setlength{\tabcolsep}{4pt}
\begin{tabular}{@{}lcccccc@{}}
\toprule
 & Origins / & Riders & \multicolumn{2}{c}{Plan-cost error} & & Within\\
Condition & runs & per run & Held-out & Reference & Change & margin\\
\midrule
Unseen cities (four folds) & \hoSOrigins\ / \hoSJobs & \hoSRiders & \hoSHeld & \hoSRef & \hoSD\ \hoSDCI & \hoSPass\\
\bottomrule
\end{tabular}
\end{table}

\begin{table}[ht]
\caption{Stage replacements on held-out riders: replacement minus full framework, budget fraction 1, 95\% interval over origins. Prediction and integration are scored by plan-cost error, reduction and allocation by selection regret. Unseen cities: the condition of Table~\ref{tab:holdout}; fit once: models fitted at the origin of 2026-04-21 and used for the next 16 weekly origins.}
\label{tab:holdout-stages}
\centering\small\setlength{\tabcolsep}{4pt}
\begin{tabular}{@{}llcc@{}}
\toprule
Replacement & Endpoint & Fit once & Unseen cities\\
\midrule
Point forecast & plan-cost error & \hoTNecPoint\ \hoTNecPointCI & \hoSNecPoint\ \hoSNecPointCI\\
Plug-in integration & plan-cost error & \hoTNecPlug\ \hoTNecPlugCI & \hoSNecPlug\ \hoSNecPlugCI\\
Two-moment normal & plan-cost error & \hoTNecNorm\ \hoTNecNormCI & \hoSNecNorm\ \hoSNecNormCI\\
Independent days & plan-cost error & \hoTNecInd\ \hoTNecIndCI & \hoSNecInd\ \hoSNecIndCI\\
Equal-width buckets & regret & \hoTNecEqw\ \hoTNecEqwCI & \hoSNecEqw\ \hoSNecEqwCI\\
Equal budget split & regret & \hoTNecEq\ \hoTNecEqCI & \hoSNecEq\ \hoSNecEqCI\\
\bottomrule
\end{tabular}
\end{table}

\section{Synthetic trajectories and a short pilot: design and further results}\label{app:synthetic}
\begin{figure}[ht]
\centering
\includegraphics[width=\linewidth]{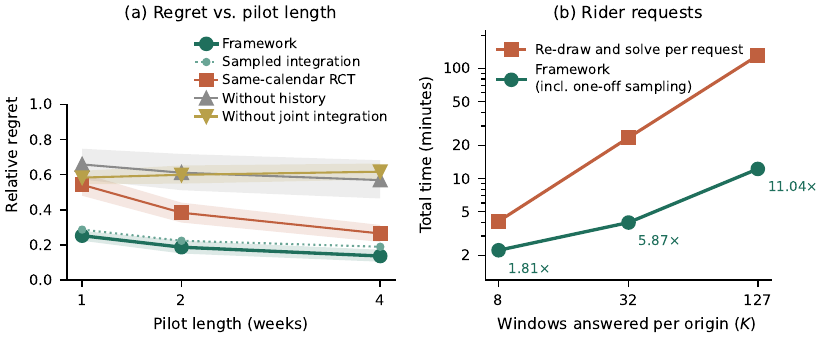}
\caption{(a) Relative regret against pilot length, 18 response laws, 10\% scoring; bands: 95\% intervals over laws. (b) Total time to answer $K$ windows per origin on rider records, one-off sampling pass charged; labels: median paired speed-ups.}
\label{fig:evidence}
\end{figure}

\paragraph{Laws.}
The 24 laws use the support, offers, payment rules and budgets of Appendix~\ref{app:response-experiment} (five groups, 32 weekly trajectory types, nine offers per group whose reward-rate increments form a $3\times3$ grid, group rules fixed award, per-completion, attendance, fixed award, per-completion) with their own response generator and a fresh seed. With dose $x=2b+t$ of offer $a=3b+t$, completions $C_k$ and active days $D_k$ of type $k$, base weights $w_{gk}\in\{1,\dots,17\}$ and a weekly shock $z\in\{0,1\}$ (probability $\tfrac12$, drawn independently for every sample index and shared by all groups and offers at that index), the no-offer law is proportional to $w_{gk}s_{gk}(z)$ with
\[
 s_{gk}(z)=64+u_g[zC_k+(1-z)(C_{\max}-C_k)],
\]
and the law under dose $x$ to
\begin{align*}
 w_{gk}s_{gk}(z)\,\bigl[&256+a_g\min(x,\tau_g)C_k+b_g(x-\tau_g)_+^2D_k\\
 &+c_gx(6-x)(C_{\max}-C_k)\bigr],
\end{align*}
with integers $a_g\in\{1,\dots,4\}$, $b_g\in\{1,\dots,6\}$, $c_g\in\{0,\dots,3\}$, $\tau_g\in\{1,\dots,4\}$ and $u_g\in\{1,2,3\}$. The combined family keeps all three response terms; the saturating family sets $b_g=c_g=0$, the threshold family $a_g=c_g=0$ and the null family $a_g=b_g=c_g=0$. The response form is unknown to every method.

\paragraph{Methods.}
All methods share exact enumeration over $9^5$ portfolios, the same 33 budgets $B\in\{0,5,\ldots,160\}$ and the band $[0.9B,1.1B]$. The long RCT uses offer-conditional frequencies from 1,024 observations per group and offer. The same-calendar RCT uses the pilot's rider-weeks spread uniformly over the nine offers. Budgets are nominal: with integer division over offers, the nominal 512, 1,024 and 2,048 rider-weeks per group give 510, 1,023 and 2,046 pilot observations and 504, 1,017 and 2,043 same-calendar observations. The framework pilots offers with increments 0.2, 0.6 and 1.0, solves Eq.~\eqref{eq:tilt} exactly for two moments (normalized completions and active days) at each piloted offer, interpolates the tilt linearly in the increment through zero at the no-offer plan, and sums payment and value exactly over the 32 types. Ablations change one stage each: \emph{sampled integration} averages 2,048 draws per group and offer from the same corrected law; \emph{without history} uses the pilot frequencies directly (unpiloted offers take the nearest piloted offer; called ``without synthesis'' in the pre-specified rules); \emph{without moment matching} uses the history law for every offer; \emph{without joint integration} evaluates payment and value on the mean trajectory; \emph{without allocation} splits the budget equally and lets each group choose its best offer within its share; the pre-specified version ran on the sampled tables (Table~\ref{tab:prereg}), and Section~\ref{sec:synthetic} and Table~\ref{tab:e2-ablations} apply it to the exact-sum tables, a re-analysis of the stored laws without new draws (within 20\% of the budget it raises one-week regret by \exEqTwentyOne\ \exEqTwentyCIOne).

\paragraph{Endpoint.}
For one replicate and pilot length, relative regret is $\sum_r(U^\star_r-U_r)/\sum_rU^\star_r$ over the budgets $r$ whose true optimum is feasible, with $U^\star_r$ the optimal true uplift and $U_r$ the true uplift of the issued plan, set to zero for an abstention or a plan whose true expected spend leaves the band. Budgets without a feasible portfolio enter only the validity count; a replicate whose optimal uplifts sum to zero, as in the null laws, has no regret. Replicates are averaged within a law.

\paragraph{Pre-specified design.}
Design, endpoints, arms, unit of inference and decision rules were fixed before the runs, with one exception. After a single development run on one law (one replicate, one-week pilot), the operator was changed from a single tilt with a linear dose to one tilt per piloted offer at three offers; that law remains in the analysis. The pre-specified primary arm integrated 2,048 sampled trajectories; the exact sum over the 32 types was pre-specified as the finite-pilot arm of the stage decomposition, and we use it as the default because it removes the integration term (Lemma~\ref{lem:int}); the pre-specified rules and their outcomes for the sampled arm are listed in Table~\ref{tab:prereg}. Excluding the development law (17 response laws) leaves every conclusion unchanged: the framework minus the same-calendar RCT is \xExSCOne\ \xExSCCIOne, \xExSCTwo\ \xExSCCITwo\ and \xExSCFour\ \xExSCCIFour\ at one, two and four weeks; the gap to the long RCT at four weeks is \xExLFour\ \xExLCIFour; and the ablations raise one-week regret by \xExNoSynOne, \xExNoMomOne, \xExPlugOne\ and, for sampling, \xExSampOne\ \xExSampCIOne.

\paragraph{The decomposition, measured.}
Because the truth is known, each term of Eq.~\eqref{eq:four-term} can be computed; Table~\ref{tab:loss-check} does so for the sampled variant, which carries the synthetic, moment-matching and integration terms; allocation enumerates exactly, so the decision term is zero, and exact summation also sets the integration term to zero. As a separate allocation diagnostic, the budget-grid program is within 0.0015 relative regret of exact enumeration. Lengthening the pilot shrinks only the moment-matching term: each doubling divides it by \decMMRatioA\ and \decMMRatioB, the $\sqrt2\approx1.41$ rate of Corollary~\ref{cor:rates}(a) for uplift; the synthetic term, which includes the interpolation bias, does not depend on the pilot and is the fixed price of shortening experimentation. As a heuristic, the decomposition suggests lengthening a pilot until the moment term is comparable to the synthetic term, which happens at two weeks here; realized regret keeps falling with a longer pilot (Table~\ref{tab:synthetic}). Proposition~\ref{prop:four-term} needs the true cost error, so it is an after-the-fact error analysis; on the \bCheckedOne, \bCheckedTwo\ and \bCheckedFour\ requests (one-, two- and four-week pilots) whose comparator fits in the band tightened by twice that error it holds without exception.

\begin{table}[ht]
\caption{Further stage replacements of the framework (penalized relative regret, pre-specified 10\% scoring, 18 response laws; the first two arms were not rescored at 20\%, and the equal split on the exact tables is also given at 20\% in the description of the methods above). Their differences from the framework are in Section~\ref{sec:synthetic}.}
\label{tab:e2-ablations}
\centering\small
\begin{tabular}{lccc}
\toprule
Replacement & 1 wk & 2 wk & 4 wk\\
\midrule
Without moment matching (history only) & \xNoMomOne & \xNoMomTwo & \xNoMomFour\\
Without joint integration (plug-in mean) & \xPlugOne & \xPlugTwo & \xPlugFour\\
Without budget allocation (equal split, exact tables) & \exEqLvlTenOne & \exEqLvlTenTwo & \exEqLvlTenFour\\
\bottomrule
\end{tabular}
\end{table}

\subsection{Empirical check of the decomposition}
\label{app:loss-check}
\begin{table}[h]
\caption{Empirical check of Proposition~\ref{prop:four-term}. \textbf{A}: controlled response laws with known truth (24 laws $\times$ 8 replicates, 33 budgets each), sampled variant ($S=2{,}048$), which carries the synthetic, moment-matching and integration terms; under exact summation the integration term is zero and the other terms are unchanged; allocation enumerates exactly, so $\varepsilon_{\mathrm{dec}}=0$, and the last two rows of A are an allocation diagnostic comparing the budget-grid program with enumeration; value errors are uplift in completions and cost errors are currency units, both summed over five layers (the median optimal uplift of a request is 14.1). \textbf{B}: rider records, 45 weekly origins $\times$ 3 seeds, five activity layers; each row replaces one stage of the framework and is compared with the framework on the realized records; intervals resample origins.}
\label{tab:loss-check}
\centering\small\setlength{\tabcolsep}{4pt}
\begin{tabular}{lccc}
\toprule
\textbf{A} \quad Pilot length & 1 week & 2 weeks & 4 weeks\\
\midrule
Synthetic term $\varepsilon^v_{\mathrm{syn}}$ / $\varepsilon^c_{\mathrm{syn}}$ & 5.92 / 5.53 & 5.92 / 5.53 & 5.92 / 5.53\\
Moment term $\varepsilon^v_{\mathrm{mm}}$ / $\varepsilon^c_{\mathrm{mm}}$ & 8.82 / 7.67 & 6.29 / 5.31 & 4.48 / 3.69\\
Integration term $\varepsilon^v_{\mathrm{int}}$ / $\varepsilon^c_{\mathrm{int}}$ ($S=2{,}048$) & 3.43 / 3.00 & 3.44 / 3.00 & 3.51 / 3.00\\
Sum of stage terms & 18.16 / 16.20 & 15.65 / 13.83 & 13.90 / 12.22\\
Direct error $\bar E_v$ / $\bar E_c$ (truth vs.\ final tables) & 11.79 / 10.88 & 10.06 / 9.35 & 8.89 / 8.24\\
Requests with a comparator in $\mathcal B_{2\bar E_c}$ & 26/6336 & 140/6336 & 499/6336\\
\quad share of the 4,968 feasible requests & 0.5\% & 2.8\% & 10.0\%\\
\quad bound below the comparator value & 16 & 62 & 290\\
Violations of part (i), tightening by $\bar E_c$ & 0 & 0 & 0\\
Grid program $-$ enumeration, relative regret & 0.0005 & 0.0015 & 0.0009\\
\quad 95\% interval over laws & $[-0.0014, 0.0028]$ & $[-0.0005, 0.0038]$ & $[-0.0014, 0.0033]$\\
\bottomrule
\end{tabular}

\vspace{4pt}
\begin{tabular}{>{\raggedright\arraybackslash}p{0.26\linewidth}>{\raggedright\arraybackslash}p{0.19\linewidth}>{\raggedright\arraybackslash}p{0.47\linewidth}}
\toprule
\textbf{B} \quad Stage replaced by & Mechanism & Observed on rider records\\
\midrule
Prediction: point forecast & plug-in floor, part (iii) & plan-cost error $\times$3.08 [2.48, 3.81]; median award-rate error 4.0 $\to$ 66.9\,pp\\
Integration: independent days & cross-day dependence & award rate $-3.07$\,pp $[-3.28, -2.87]$ on 5--6-day plans, $+3.32$\,pp [2.93, 3.69] on the others$^\dagger$; cost error $\times$1.51\\
Integration: per-rider normal, attendance at its mean & two moments; plug-in gate & $-10.31$\,pp $[-10.97, -9.69]$ on 5--6-day plans, $+0.23$\,pp $[-0.19, 0.57]$ on the others$^\dagger$; cost error $\times$2.46\\
Integration: plug-in mean & plug-in floor, part (iii) & $+10.17$\,pp [9.31, 11.03] on the others$^\dagger$, $-4.11$\,pp on 5--6-day plans; cost error $\times$1.85\\
Decision: budget grid vs.\ exact optimum & Lemma~\ref{lem:dec} & zero loss in \nmDPfbZero\ of \nmDPfbN\ five-layer main-grid instances (45 origins $\times$ 3 budgets), max \nmDPfbMax\% (realized-record tables)\\
Decision: equal budget split & allocation floor, part (iii) & value $-$15.7\% [13.8, 17.4] at budget fraction 0.5\\
\bottomrule
\end{tabular}

\smallskip{\raggedright\footnotesize $^\dagger$Plans requiring 0, 2, 3 or 4 active days, derived exactly from the all-plan and 5--6-day means. Award-rate entries are differences from the framework's joint law in percentage points.\par}
\end{table}

In panel A each stage term is measured against the preceding law, exactly as in the proposition, and the direct error is below the sum of stage terms, as the triangle inequality allows. The synthetic and integration terms do not change with the pilot, while the moment term falls by factors 1.40 and 1.41 per doubling of the pilot (cost: 1.44 and 1.44), close to the $\sqrt2\approx1.41$ of Corollary~\ref{cor:rates}(a). Part (i) held on every request where its comparator exists. The check tightens by the direct error $\bar E_c$, which needs the truth, so it is an after-the-fact error analysis. Among the 4,968 requests whose band holds a portfolio of true cost, the tightened band still holds one in 26, 140 and 499 (0.5\%, 2.8\% and 10.0\% at one, two and four weeks), and the bound is below the comparator value, the trivial bound, in 16, 62 and 290 (0.3\%, 1.2\% and 5.8\%); a rule held to the tightened band would abstain on the other 4,942, 4,828 and 4,469; realized regrets are reported in Section~\ref{sec:synthetic}. The grid program is indistinguishable from enumeration. In panel B the signs follow the mechanisms of Appendix~\ref{app:loss-proofs}: composing days independently understates the 5--6-day requirements and overstates lower ones, the crossing expected under shared daily shocks; the normal law with attendance evaluated at its mean fails where attendance binds, the plug-in floor of part (iii); and dropping distributional prediction or joint allocation costs most. Panel B concerns quantities that can be replayed under the offered plans; it involves no behavioural response.

\section{Integrated study: changing requests and binding budgets}\label{app:integrated}
Design, endpoints and decision rules were registered, with the simulation code, before the run. One pre-run check on the first six requests of one law and replicate is disclosed with the registration and changed nothing. The analysis script was completed after the run and pooled all \eFLaws\ laws for timing; H2 is reported here, as registered, on the \eFRespLaws\ response laws.

\paragraph{Laws and requests.}
The \eFLaws\ laws are fresh draws of the generator of Appendix~\ref{app:synthetic}, twelve per family (\eFRespLaws\ with a response; the \eFNullLaws\ null laws check validity and timing only). Behaviour responds to the offer at the week level, and the request decides which days are paid and counted. Every replicate issues its own stream of \eFReqs\ requests. Each request draws, uniformly, an activity window among the 127 non-empty subsets of the seven days, a population among the 26 subsets of at least two of the five groups, and one rule for all its groups (fixed award, per-completion rate or attendance). Tier thresholds scale with the window, $\max(1,\mathrm{round}(t\,|W|/7))$ for $t\in\{14,28,42\}$ completions or $\{2,4,6\}$ active days; uplift is the change in completions inside the window. With $c_{\rm ref}$ the true expected cost of the cheapest portfolio that maximizes uplift without a budget, the request carries three budgets $B=\beta c_{\rm ref}$, $\beta\in\{0.3,0.5,0.7\}$, which gives \eFBudgeted\ budgeted requests per replicate; the budgets are set from the true law so that they bind, and every method receives the same ones. A budget binds when its band optimum is below the unconstrained maximum; \eFBind\% of budgeted requests bind (over all \eFLaws\ laws).

\paragraph{Data and methods.}
Each of eight replicates per law draws 4,096 no-offer weeks per group, a pilot of one, two or four weeks at 512 riders per group and week spread over offers 2, 5 and 8, and a same-calendar RCT that spreads the same nominal rider-weeks over all nine offers (with integer division per cell, 510, 1,023 and 2,046 pilot and 504, 1,017 and 2,043 same-calendar observations per group for one, two and four weeks); in the registered primary condition reported here, every sample index has its own weekly shock shared by all cells. The \emph{framework} prepares its laws once per origin, from history and pilot only: the history law, the moment-matched tilt of Eq.~\eqref{eq:tilt} at the piloted offers interpolated in the reward increment, and the corrected law of every offer. Each request then reuses these laws: its tables come from the exact sum over the 32 trajectory types for the request's window, population and rule, and its plan from exact enumeration over the $9^{|\text{population}|}$ portfolios in the band. The same-calendar RCT uses its empirical offer laws with the same integration and allocation. Stage replacements change one stage: the history law for every offer (without moment matching); pilot frequencies at the nearest piloted offer (without history); payment at the mean window completions and active days (plug-in integration); days composed independently from per-day marginals (independent days); the band split equally over the population's groups (equal split). We also report the budget-grid program with 1,024 states in place of enumeration and, as a same-information control, the history and pilot laws interpolated in the reward increment (dose interpolation). For timing, every request is also answered by re-running the whole pipeline, including the preparation, from the raw data; its answers must equal the framework's.

\paragraph{Endpoints and decision rules.}
Relative regret is computed per replicate over the budgeted requests whose band optimum exists, scoring an abstention or a plan whose true expected spend leaves the band as zero; replicates are averaged within a law and 95\% intervals come from 20,000 bootstrap resamples of the \eFRespLaws\ response laws. The primary band is $[0.8B,1.2B]$, with $[0.9B,1.1B]$ alongside. The confirmatory family (Table~\ref{tab:prereg}) is H1, the framework has lower regret than the same-calendar RCT at each pilot length (upper bound below zero), and H2, answering all \eFReqs\ requests of an origin by reuse with the preparation charged is faster than re-running the pipeline per request (lower bound of the law-mean log speed-up over the \eFRespLaws\ response laws above zero, every timed answer identical). The breakdowns of Table~\ref{tab:integrated-breakdown} are descriptive; a subgroup is summarized over the laws in which it contains a request with a feasible band optimum (at $\beta=0.3$, two laws have none), a correction of the registered summary made after the run.

\begin{table}[ht]
\caption{Integrated study: relative regret (lower is better) on \eFRespLaws\ response laws $\times$ \eFReps\ replicates $\times$ \eFBudgeted\ budgeted requests, with each request's window, population, rule and budget drawn afresh. Indented rows change one stage of the framework; every difference from the framework has a 95\% interval above zero. Validity: share of the framework's issued plans whose true expected spend lies in the band, over all \eFLaws\ laws.}
\label{tab:integrated}
\centering\small\setlength{\tabcolsep}{3.5pt}
\begin{tabular}{lcccccc}
\toprule
 & \multicolumn{3}{c}{Within 20\% of budget} & \multicolumn{3}{c}{Within 10\%}\\
\cmidrule(lr){2-4}\cmidrule(l){5-7}
Method & 1 wk & 2 wk & 4 wk & 1 wk & 2 wk & 4 wk\\
\midrule
Same-calendar RCT & \eFSCOne & \eFSCTwo & \eFSCFour & \eFXSCOne & \eFXSCTwo & \eFXSCFour\\
\textbf{Framework} & \textbf{\eFFWOne} & \textbf{\eFFWTwo} & \textbf{\eFFWFour} & \textbf{\eFXFWOne} & \textbf{\eFXFWTwo} & \textbf{\eFXFWFour}\\
\quad without moment matching & \eFNoMomOne & \eFNoMomTwo & \eFNoMomFour & \eFXNoMomOne & \eFXNoMomTwo & \eFXNoMomFour\\
\quad without history & \eFNoSynOne & \eFNoSynTwo & \eFNoSynFour & \eFXNoSynOne & \eFXNoSynTwo & \eFXNoSynFour\\
\quad plug-in integration & \eFPlugOne & \eFPlugTwo & \eFPlugFour & \eFXPlugOne & \eFXPlugTwo & \eFXPlugFour\\
\quad independent days & \eFIndOne & \eFIndTwo & \eFIndFour & \eFXIndOne & \eFXIndTwo & \eFXIndFour\\
\quad equal split & \eFEqOne & \eFEqTwo & \eFEqFour & \eFXEqOne & \eFXEqTwo & \eFXEqFour\\
Dose interpolation (same information) & \eFCOne & \eFCTwo & \eFCFour & \eFXCOne & \eFXCTwo & \eFXCFour\\
Validity of the framework (\%) & \eFValOne & \eFValTwo & \eFValFour & \eFXValOne & \eFXValTwo & \eFXValFour\\
\bottomrule
\end{tabular}
\end{table}

\paragraph{Results.}
Framework minus same-calendar RCT is \eFDOne\ \eFDCIOne, \eFDTwo\ \eFDCITwo\ and \eFDFour\ \eFDCIFour\ within 20\% (\eFRelOne\%, \eFRelTwo\% and \eFRelFour\% lower), and \eFXDOne\ \eFXDCIOne, \eFXDTwo\ \eFXDCITwo\ and \eFXDFour\ \eFXDCIFour\ within 10\%. With a one-week pilot, the stage replacements raise regret by \eFdNoMomOne\ \eFdNoMomCIOne\ (moment matching), \eFdNoSynOne\ \eFdNoSynCIOne\ (history), \eFdPlugOne\ \eFdPlugCIOne\ (plug-in integration), \eFdIndOne\ \eFdIndCIOne\ (independent days) and \eFdEqOne\ \eFdEqCIOne\ (equal split). The budget-grid program's regret differs from that of enumeration by at most \eFDPmax\ in every cell, and dose interpolation has higher regret by \eFdCOne\ \eFdCCIOne, \eFdCTwo\ \eFdCCITwo\ and \eFdCFour\ \eFdCCIFour\ at one, two and four weeks. Table~\ref{tab:integrated-breakdown} shows that the gain over the same-calendar RCT holds in every rule, window length, population size and budget level, and Table~\ref{tab:integrated-time} gives the preparation-charged timing.

\begin{table}[ht]
\caption{Integrated study by request type (within 20\% of budget, descriptive): framework regret and framework minus same-calendar RCT with 95\% intervals over the laws in which the subgroup is defined.}
\label{tab:integrated-breakdown}
\centering\scriptsize\setlength{\tabcolsep}{3pt}
\begin{tabular}{@{}llrcccc@{}}
\toprule
 & & & \multicolumn{2}{c}{1-week pilot} & \multicolumn{2}{c}{2-week pilot}\\
\cmidrule(lr){4-5}\cmidrule(l){6-7}
 & Level & Laws & Framework & minus same-calendar & Framework & minus same-calendar\\
\midrule
Rule & fixed award & 36 & 0.225 & $-$0.175 [$-$0.208, $-$0.142] & 0.176 & $-$0.127 [$-$0.166, $-$0.089]\\
 & per completion & 36 & 0.229 & $-$0.172 [$-$0.206, $-$0.138] & 0.180 & $-$0.125 [$-$0.165, $-$0.087]\\
 & attendance & 36 & 0.205 & $-$0.165 [$-$0.192, $-$0.139] & 0.162 & $-$0.120 [$-$0.151, $-$0.089]\\
\midrule
Window days & 1--2 & 36 & 0.221 & $-$0.193 [$-$0.226, $-$0.159] & 0.173 & $-$0.146 [$-$0.184, $-$0.110]\\
 & 3--5 & 36 & 0.219 & $-$0.171 [$-$0.200, $-$0.141] & 0.171 & $-$0.123 [$-$0.159, $-$0.088]\\
 & 6--7 & 36 & 0.226 & $-$0.155 [$-$0.194, $-$0.113] & 0.172 & $-$0.119 [$-$0.159, $-$0.082]\\
\midrule
Groups & 2 & 36 & 0.216 & $-$0.214 [$-$0.238, $-$0.190] & 0.171 & $-$0.152 [$-$0.183, $-$0.123]\\
 & 3 & 36 & 0.215 & $-$0.160 [$-$0.193, $-$0.126] & 0.169 & $-$0.116 [$-$0.154, $-$0.076]\\
 & 4 & 36 & 0.230 & $-$0.153 [$-$0.190, $-$0.114] & 0.183 & $-$0.111 [$-$0.154, $-$0.067]\\
 & 5 & 36 & 0.223 & $-$0.148 [$-$0.187, $-$0.111] & 0.186 & $-$0.105 [$-$0.148, $-$0.063]\\
\midrule
Budget $\beta$ & 0.3 & 34 & 0.305 & $-$0.290 [$-$0.346, $-$0.237] & 0.255 & $-$0.205 [$-$0.253, $-$0.160]\\
 & 0.5 & 36 & 0.228 & $-$0.197 [$-$0.258, $-$0.151] & 0.180 & $-$0.151 [$-$0.216, $-$0.101]\\
 & 0.7 & 36 & 0.176 & $-$0.149 [$-$0.181, $-$0.118] & 0.131 & $-$0.109 [$-$0.145, $-$0.075]\\
\bottomrule
\end{tabular}
\end{table}

\begin{table}[ht]
\caption{Integrated study, preparation charged (milliseconds on one core; medians over replicates for the one-off preparation and over requests otherwise). Speed-up: time to answer the first $K$ requests of a replicate by re-running the whole pipeline for each, over the preparation plus the reusing requests; geometric mean over the \eFRespLaws\ response laws (replicates averaged on the log scale) with 95\% intervals. Answers are identical on all \eFIdentAll\ timed requests of these laws and on all \eFIdentFortyEight\ of all \eFLaws\ laws.}
\label{tab:integrated-time}
\centering\small\setlength{\tabcolsep}{4pt}
\begin{tabular}{lcccccc}
\toprule
Pilot & Preparation & Request (reuse) & Request (re-run) & $K=8$ & $K=32$ & $K=60$\\
\midrule
1 week & \eFPrepMsOne & \eFReqMsOne & \eFNoReuseMsOne & \eFSpEightOne$\times$ & \eFSpThirtyTwoOne$\times$ & \eFSpSixtyOne$\times$ \eFSpSixtyCIOne\\
2 weeks & \eFPrepMsTwo & \eFReqMsTwo & \eFNoReuseMsTwo & \eFSpEightTwo$\times$ & \eFSpThirtyTwoTwo$\times$ & \eFSpSixtyTwo$\times$ \eFSpSixtyCITwo\\
4 weeks & \eFPrepMsFour & \eFReqMsFour & \eFNoReuseMsFour & \eFSpEightFour$\times$ & \eFSpThirtyTwoFour$\times$ & \eFSpSixtyFour$\times$ \eFSpSixtyCIFour\\
\bottomrule
\end{tabular}
\end{table}

\section{Response validation with a new generator}\label{app:structural}
This study answers whether the results of Appendix~\ref{app:integrated} depend on the generator that produced its laws. Its design, endpoints, decision rules, simulation code and analysis script were all registered before the run. One disclosed pre-run check on four requests changed nothing.

\paragraph{Generator.}
None of it is shared with the 32-type construction of Appendix~\ref{app:synthetic}. Each of \eGLaws\ laws (twelve per family) has six rider groups, and each group is a mixture of 16 latent rider types with Dirichlet weights.
\begin{itemize}
\item \emph{Activity.} A type is active on day $d$ with probability $\mathrm{logit}^{-1}(\alpha_k+\gamma_d)$, where $\gamma_d$ is a day-of-week effect of the group.
\item \emph{Completions.} An active day has $1+\mathrm{Poisson}(\lambda_k-1)$ completions, truncated at 15, and a weekly shock multiplies $\lambda_k$ by $1\pm\delta$.
\item \emph{Response.} The response depends on the rule shown. Under dose $x$ of rule $\rho$, activity becomes $p+(1-p)\min\{0.9,A_kS^A_\rho f(x)\}$ and intensity $\lambda(1+B_kS^B_\rho f(x))$. The sensitivities are $(S^A,S^B)=(0.4,0.8)$ for a fixed award, $(0.2,1)$ for a per-completion rate and $(1,0.2)$ for attendance.
\item \emph{Families.} Both channels with a saturating $f$; intensity only; a threshold $f$ in which only types whose usual week lies between the first and third tier respond fully; no response.
\end{itemize}
Weeks take $16^7$ daily patterns per rider, so the framework cannot enumerate them. The truth is exact: a mixture over types and shocks of day-by-day convolutions of the law of window completions and active days.

\paragraph{Requests, data and methods.}
Requests follow Appendix~\ref{app:integrated}, with populations of two to four of the six groups; \eGBind\% of budgeted requests bind (over all \eGLaws\ laws). Each of eight replicates per law draws:
\begin{itemize}
\item 4,096 no-offer weeks per group;
\item a pilot of one, two or four weeks at 512 riders per group and week, spread over the three rules times offers 2, 5 and 8;
\item a same-calendar RCT that spreads the same nominal rider-weeks over the three rules times all nine offers. With integer division per cell the actual observations per group are 504, 1,017 and 2,043 in the pilot and 486, 999 and 2,025 in the same-calendar RCT for one, two and four weeks, so the trial has slightly fewer.
\end{itemize}
For every group and rule the \emph{framework} tilts the stored weeks to the pilot moments by Eq.~\eqref{eq:tilt} (interpolated in the reward increment). It does this once per origin; every request then reuses the weights and integrates its window rule as a weighted sum over the stored weeks (Algorithm~\ref{alg:request}, lines~5--6). The stage replacements are the same as in Appendix~\ref{app:integrated}: history weeks for every offer; pilot weeks at the nearest piloted offer; payment at the weighted mean; days composed independently; equal split. Dose interpolation mixes history and pilot weeks with interpolation weights in the reward increment. For timing, each request re-fits the tilts it needs from the raw data. The endpoints, the unit of inference (\eGRespLaws\ response laws) and H1--H2 are as in Appendix~\ref{app:integrated}.

\begin{table}[ht]
\caption{New generator: relative regret (lower is better) on \eGRespLaws\ response laws $\times$ \eGReps\ replicates $\times$ \eGBudgeted\ budgeted requests. Indented rows change one stage of the framework; every difference from the framework, and that of dose interpolation, has a 95\% interval above zero. Validity: share of the framework's issued plans whose true expected spend lies in the band, over all \eGLaws\ laws.}
\label{tab:structural}
\centering\small\setlength{\tabcolsep}{3.5pt}
\begin{tabular}{lcccccc}
\toprule
 & \multicolumn{3}{c}{Within 20\% of budget} & \multicolumn{3}{c}{Within 10\%}\\
\cmidrule(lr){2-4}\cmidrule(l){5-7}
Method & 1 wk & 2 wk & 4 wk & 1 wk & 2 wk & 4 wk\\
\midrule
Same-calendar RCT & \eGSCOne & \eGSCTwo & \eGSCFour & \eGXSCOne & \eGXSCTwo & \eGXSCFour\\
\textbf{Framework} & \textbf{\eGFWOne} & \textbf{\eGFWTwo} & \textbf{\eGFWFour} & \textbf{\eGXFWOne} & \textbf{\eGXFWTwo} & \textbf{\eGXFWFour}\\
\quad without moment matching & \eGNoMomOne & \eGNoMomTwo & \eGNoMomFour & \eGXNoMomOne & \eGXNoMomTwo & \eGXNoMomFour\\
\quad without history & \eGNoSynOne & \eGNoSynTwo & \eGNoSynFour & \eGXNoSynOne & \eGXNoSynTwo & \eGXNoSynFour\\
\quad plug-in integration & \eGPlugOne & \eGPlugTwo & \eGPlugFour & \eGXPlugOne & \eGXPlugTwo & \eGXPlugFour\\
\quad independent days & \eGIndOne & \eGIndTwo & \eGIndFour & \eGXIndOne & \eGXIndTwo & \eGXIndFour\\
\quad equal split & \eGEqOne & \eGEqTwo & \eGEqFour & \eGXEqOne & \eGXEqTwo & \eGXEqFour\\
Dose interpolation (same information) & \eGCOne & \eGCTwo & \eGCFour & \eGXCOne & \eGXCTwo & \eGXCFour\\
Validity of the framework (\%) & \eGValOne & \eGValTwo & \eGValFour & \eGXValOne & \eGXValTwo & \eGXValFour\\
\bottomrule
\end{tabular}
\end{table}

\paragraph{Results.}
\begin{itemize}
\item \emph{Same-calendar RCT.} Framework minus same-calendar RCT is \eGDOne\ \eGDCIOne, \eGDTwo\ \eGDCITwo\ and \eGDFour\ \eGDCIFour\ within 20\% (\eGRelOne\%, \eGRelTwo\% and \eGRelFour\% lower), and \eGXDOne\ \eGXDCIOne, \eGXDTwo\ \eGXDCITwo\ and \eGXDFour\ \eGXDCIFour\ within 10\%. The gain holds in every rule, window length, population size and budget level (Table~\ref{tab:structural-breakdown}).
\item \emph{Stage replacements.} With a one-week pilot they raise regret by:
  \begin{itemize}
  \item \eGdNoMomOne\ \eGdNoMomCIOne\ without moment matching;
  \item \eGdNoSynOne\ \eGdNoSynCIOne\ without history;
  \item \eGdPlugOne\ \eGdPlugCIOne\ with plug-in integration;
  \item \eGdIndOne\ \eGdIndCIOne\ with independent days;
  \item \eGdEqOne\ \eGdEqCIOne\ with an equal split.
  \end{itemize}
\item \emph{Dose interpolation.} Its regret is higher by \eGdCOne\ \eGdCCIOne, \eGdCTwo\ \eGdCCITwo\ and \eGdCFour\ \eGdCCIFour\ at one, two and four weeks.
\item \emph{Allocation.} The budget-grid program's regret is within \eGDPmax\ of that of enumeration.
\item \emph{Timing.} Here the per-request alternative re-fits only the tilts the request needs. With the preparation charged, answering the first 8, 32 and 60 requests of an origin is \eGSpEightOne$\times$, \eGSpThirtyTwoOne$\times$ and \eGSpSixtyOne$\times$ \eGSpSixtyCIOne\ faster (one-week pilot; \eGSpSixtyTwo$\times$ and \eGSpSixtyFour$\times$ at two and four weeks). Medians are \eGPrepMsOne\,ms for the preparation, \eGReqMsOne\,ms per reusing request and \eGNoReuseMsOne\,ms per re-fitting request. Answers are identical on all \eGIdentAll\ timed requests of the response laws.
\end{itemize}

\begin{table}[ht]
\caption{New generator by request type (within 20\% of budget, descriptive): framework regret and framework minus same-calendar RCT with 95\% intervals over laws.}
\label{tab:structural-breakdown}
\centering\scriptsize\setlength{\tabcolsep}{3pt}
\begin{tabular}{@{}llrcccc@{}}
\toprule
 & & & \multicolumn{2}{c}{1-week pilot} & \multicolumn{2}{c}{4-week pilot}\\
\cmidrule(lr){4-5}\cmidrule(l){6-7}
 & Level & Laws & Framework & minus same-calendar & Framework & minus same-calendar\\
\midrule
Rule & fixed award & 36 & 0.179 & $-$0.318 [$-$0.343, $-$0.293] & 0.120 & $-$0.127 [$-$0.143, $-$0.112]\\
 & per completion & 36 & 0.196 & $-$0.322 [$-$0.343, $-$0.300] & 0.111 & $-$0.127 [$-$0.144, $-$0.111]\\
 & attendance & 36 & 0.202 & $-$0.183 [$-$0.213, $-$0.155] & 0.132 & $-$0.114 [$-$0.137, $-$0.093]\\
\midrule
Window days & 1--2 & 36 & 0.186 & $-$0.340 [$-$0.365, $-$0.315] & 0.119 & $-$0.154 [$-$0.175, $-$0.133]\\
 & 3--5 & 36 & 0.188 & $-$0.285 [$-$0.308, $-$0.263] & 0.116 & $-$0.121 [$-$0.138, $-$0.106]\\
 & 6--7 & 36 & 0.189 & $-$0.269 [$-$0.294, $-$0.242] & 0.122 & $-$0.111 [$-$0.130, $-$0.092]\\
\midrule
Groups & 2 & 36 & 0.211 & $-$0.301 [$-$0.318, $-$0.285] & 0.131 & $-$0.154 [$-$0.168, $-$0.140]\\
 & 3 & 36 & 0.188 & $-$0.294 [$-$0.315, $-$0.273] & 0.117 & $-$0.121 [$-$0.139, $-$0.104]\\
 & 4 & 36 & 0.176 & $-$0.277 [$-$0.306, $-$0.248] & 0.110 & $-$0.111 [$-$0.131, $-$0.093]\\
\midrule
Budget $\beta$ & 0.3 & 36 & 0.261 & $-$0.275 [$-$0.300, $-$0.252] & 0.159 & $-$0.182 [$-$0.217, $-$0.150]\\
 & 0.5 & 36 & 0.181 & $-$0.288 [$-$0.312, $-$0.264] & 0.114 & $-$0.118 [$-$0.138, $-$0.100]\\
 & 0.7 & 36 & 0.170 & $-$0.301 [$-$0.326, $-$0.276] & 0.106 & $-$0.106 [$-$0.121, $-$0.093]\\
\bottomrule
\end{tabular}
\end{table}

\section{Stronger controls, weight diagnostics, behaviour-changing requests and population reduction}\label{app:strong}
Four further studies were registered, with designs, decision rules, simulation code and analysis scripts fixed before the runs. All use the rider-level generator of Appendix~\ref{app:structural} with fresh seeds, 60 changing requests per replicate, budgets at 0.3, 0.5 and 0.7 times the unconstrained optimum, and the law as the unit of inference (95\% intervals from a law bootstrap, Holm correction within each family); studies C--E use eight replicates per law. Regret is relative regret of pure uplift within 20\% of the budget, as in Table~\ref{tab:studies}.

\paragraph{Study C: stronger same-information controls at matched cache.}
It draws 384 laws (96 per family; 288 with a response). The pilot and the same-calendar trial have the same number of observations per group (486, 999 and 2,025 at one, two and four weeks). Every arm reads the same cached 4,096 stored weeks per group and the same data, and every arm is solved by the same exact enumeration.
\begin{itemize}
\item \emph{History-anchored tilt of the trial.} The framework's tilt is fitted at each of the eight non-zero offers to that offer's cell of the same-calendar trial. This is the strongest control that uses the same history and the same number of randomized observations.
\item \emph{Normal law with exact truncated moments.} It keeps the framework's weights and replaces only the law of the paid statistic by a normal with the weighted mean and variance. Every tier is priced by
\[
 \E[X\mathbf 1\{\ell\le X<h\}]=\mu\{\bar\Phi(z_\ell)-\bar\Phi(z_h)\}+\sigma\{\varphi(z_\ell)-\varphi(z_h)\},
\]
with $z=(t-0.5-\mu)/\sigma$ at every integer threshold $t$. This separates distribution shape from the moment computation.
\item \emph{Calendar condition.} A rider keeps its type through all observed weeks. One shock per group and calendar week is shared by every rider observed in that week, and the pilot, the trial and 512 concurrent no-offer riders share the pilot weeks' shocks. The framework then anchors its moment targets to the concurrent riders (pilot mean minus concurrent mean plus history mean), and the trial measures uplift against the same concurrent riders.
\end{itemize}

\begin{table}[ht]
\caption{Study C: relative regret within 20\% of the budget, 288 response laws $\times$ 8 replicates $\times$ 60 requests $\times$ 3 budgets. Every arm uses the same stored weeks, data and exact enumeration. Calendar: repeated riders with one shock per group and calendar week.}
\label{tab:strong}
\centering\small\setlength{\tabcolsep}{4pt}
\begin{tabular}{llccc}
\toprule
Condition & Method & 1 wk & 2 wk & 4 wk\\
\midrule
Independent & \textbf{Framework} & \textbf{0.196} & \textbf{0.154} & \textbf{0.126}\\
 & Dose interpolation (same information) & 0.219 & 0.171 & 0.138\\
 & Normal law, exact truncated moments & 0.250 & 0.208 & 0.179\\
 & History-anchored tilt of the trial & 0.405 & 0.291 & 0.211\\
 & Same-calendar RCT & 0.480 & 0.340 & 0.241\\
\midrule
Calendar & \textbf{Framework}, concurrent anchor & \textbf{0.223} & \textbf{0.215} & \textbf{0.207}\\
 & Same-calendar RCT, concurrent control & 0.490 & 0.429 & 0.386\\
\bottomrule
\end{tabular}
\end{table}

\paragraph{Results of study C.}
All seven registered hypotheses hold (Table~\ref{tab:prereg}).
\begin{itemize}
\item History anchoring helps the trial as well: the tilted trial has lower regret than the raw trial by 0.075 [0.072, 0.078] at one week.
\item The framework's regret is below that of every same-information control at every pilot length (Table~\ref{tab:strong}). At one week it is below dose interpolation by 0.022 [0.021, 0.024], below the normal law with exact truncated moments by 0.053 [0.051, 0.055], and below the tilted trial by 0.209 [0.201, 0.217]. Study E below splits this last difference into its two sources. The normal comparison isolates the distribution shape, since the weights and the tier-moment formula are exact in both arms.
\item Under calendar shocks shared by the riders of a week, the concurrently anchored framework has lower regret than the concurrently controlled trial by 0.267 [0.259, 0.274] at one week and 0.214 [0.206, 0.222] at two weeks.
\item Reuse gives identical answers in all 184,320 comparisons per condition and pilot length. Charging the one-off preparation, it answers the 60 requests of an origin 2.73$\times$ [2.71, 2.74] faster than re-running the pipeline per request at one week (2.71$\times$ and 2.73$\times$ at two and four weeks).
\end{itemize}

\paragraph{Study E: placement of the observations versus sharing across offers.}
The tilted trial differs from the framework in two ways: its randomized observations are spread over nine offers instead of three, and its tilt is fitted separately at every offer instead of along one dose curve. Study E adds a \emph{joint fit of the trial} that keeps the trial's data and removes the second difference. It uses the framework's own parameterization, $\lambda(s)$ piecewise linear in the reward increment $s$ with $\lambda(0)=0$ and nodes $\theta=(\theta_2,\theta_5,\theta_8)$ at the three piloted offers, and fits $\theta$ to all eight non-zero trial cells by maximum likelihood of the tilted-history model:
\begin{multline*}
 \widehat\theta=\arg\max_{\theta}\ \sum_{a=1}^{8}n_a\Bigl[\lambda_\theta(s_a)^{\top}m_a\\
 -\log\sum_{i=1}^{N}\exp\bigl\{\lambda_\theta(s_a)^{\top}\phi_i\bigr\}\Bigr],
\end{multline*}
where $m_a$ is the mean feature vector of the $n_a$ trial observations at offer $a$ and $\phi_i$ are the features of the $N$ stored weeks. The problem is concave, and when the data sit at the three nodes only it separates into the framework's per-node moment matching. The difference between the framework and the tilted trial then splits into two contrasts,
\[
 R_{\mathrm{FW}}-R_{\mathrm{tilt}}=\underbrace{(R_{\mathrm{FW}}-R_{\mathrm{joint}})}_{\text{placement}}
 +\underbrace{(R_{\mathrm{joint}}-R_{\mathrm{tilt}})}_{\text{sharing across offers}},
\]
the first at the same parameterization and the second on the same data. The study reuses the 384 laws of study C with fresh requests and data, the same observation counts, stored weeks, cache and exact enumeration. In the first run, the registered check that the joint fit reproduces the framework's per-node fit on pilot data found a numerical fault in the joint-fit implementation; it was corrected under an amendment registered before the re-run, the study was re-run in full with the same laws, data streams and analysis script, and all results reported here come from the re-run.

\begin{table}[ht]
\caption{Study E: relative regret within 20\% of the budget, 288 response laws $\times$ 8 replicates $\times$ 60 requests $\times$ 3 budgets; share of the difference to the tilted trial due to placing the observations at three offers, with 95\% law-bootstrap intervals.}
\label{tab:joint}
\centering\small\setlength{\tabcolsep}{2pt}
\begin{tabular}{lccc}
\toprule
Method & 1 wk & 2 wk & 4 wk\\
\midrule
\textbf{Framework} (pilot at three offers) & \textbf{0.198} & \textbf{0.154} & \textbf{0.123}\\
Joint fit of the trial, framework's dose curve & 0.253 & 0.206 & 0.183\\
History-anchored tilt of the trial, per offer & 0.405 & 0.289 & 0.211\\
Same-calendar RCT & 0.479 & 0.337 & 0.241\\
\midrule
Placement share & 0.265 [0.252, 0.279] & 0.382 [0.361, 0.405] & 0.683 [0.623, 0.745]\\
\bottomrule
\end{tabular}
\end{table}

\paragraph{Results of study E.}
All five registered hypotheses hold (Table~\ref{tab:prereg}).
\begin{itemize}
\item Sharing the tilt across offers along the framework's dose curve improves the trial by 0.152 [0.146, 0.158] at one week.
\item At the same parameterization, the framework's regret is below that of the joint fit by 0.055 [0.051, 0.059], 0.052 [0.048, 0.056] and 0.060 [0.053, 0.066] at one, two and four weeks, and its cost tables are more accurate (relative cost-table error 0.076 against 0.090 at one week; difference 0.014 [0.014, 0.015]).
\item At one week, sharing across offers accounts for 73.5\% [72.1, 74.8] of the framework's advantage over the tilted trial and placement for 26.5\% [25.2, 27.9]. The placement share grows with the pilot length (Table~\ref{tab:joint}) as the per-offer fits improve.
\item By risk component at one week (Table~\ref{tab:risk}), the framework's plans leave the band less often than the joint fit's, below (3.4\% against 6.5\% of requests) and above (4.1\% against 4.4\%), and by less (3.5\% against 4.6\% of the budget), and its valid plans lose less value (13.6\% against 16.5\%).
\end{itemize}

\begin{table}[ht]
\caption{Risk components at one week within 20\% of the budget (study E, 288 response laws; law means). Every arm issues a plan for more than 99.99\% of the requests. Below/above: share of requests whose plan has true expected spend below/above the band. Excursion: mean distance outside the band, in units of the budget, over the plans that leave it. Valid-plan loss: value lost by the plans inside the band relative to the band optimum.}
\label{tab:risk}
\centering\small\setlength{\tabcolsep}{4pt}
\begin{tabular}{lcccc}
\toprule
Method & Below & Above & Excursion & Valid-plan loss\\
\midrule
\textbf{Framework} & \textbf{0.034} & \textbf{0.041} & \textbf{0.035} & \textbf{0.136}\\
Joint fit of the trial & 0.065 & 0.044 & 0.046 & 0.165\\
Dose interpolation (same information) & 0.044 & 0.045 & 0.041 & 0.149\\
History-anchored tilt of the trial, per offer & 0.179 & 0.010 & 0.070 & 0.266\\
Same-calendar RCT & 0.272 & 0.010 & 0.082 & 0.288\\
\bottomrule
\end{tabular}
\end{table}

\paragraph{Weight diagnostics.}
For every group, rule and piloted offer, the study records the effective sample size $\mathrm{ESS}=(\sum_iw_i)^2/\sum_iw_i^2$ of the tilt over the $N=4{,}096$ stored weeks. It also records $N\max_iw_i$, whether the moment target lies in the coordinate-wise range (box) of the stored weeks' features, a necessary condition for lying in their convex hull, whether the fit reaches its tolerance, and the condition number of the feature covariance. Percentiles below are over all such cells at one week, independent condition.
\begin{itemize}
\item $\mathrm{ESS}/N$: median 0.951, 5th percentile 0.600, 1st percentile 0.422.
\item $N\max_iw_i$: median 2.05, 95th percentile 12.6, 99th percentile 26.4, so even at the 99th percentile no stored week carries more than 0.65\% of the mass.
\item The target lies outside the feature box in 0\% of the cells, and no fit fails to reach its tolerance.
\item The covariance condition number has median 7.8, 95th percentile 31.3 and 99th percentile 47.4.
\end{itemize}
Both registered diagnostic hypotheses hold. Within a law, the ESS deficit $1-\mathrm{ESS}/N$ is rank-correlated with the cost-table error of the framework (Spearman 0.415 [0.400, 0.431]). Fitting the tilt to exact pilot moments on 1,024 rather than 16,384 stored weeks raises the relative cost-table error by 0.006 [0.006, 0.006]. These are heuristic stability diagnostics: fits converge, weights are not concentrated, and the ESS deficit tracks the cost-table error. They do not check the interior-hull and eigenvalue conditions of Lemma~\ref{lem:joint-base} along the parameter path (a condition number is not an eigenvalue bound), which remain sufficient conditions of the analysis.

\paragraph{Study D: requests that change behaviour.}
It draws 360 laws: 216 window-dependent laws (three families), 72 laws whose response does not depend on the window, and 72 null laws. In the window-dependent families, a short paid window concentrates effort on its days, riders shift some activity away from unpaid days, and less active riders respond more. A request may also target the riders above or below the group's median of the previous no-offer week, whose response differs from the group's.
\begin{itemize}
\item \emph{Direct reuse} keeps the full-week, group-level calibration of Appendix~\ref{app:structural} for every request.
\item \emph{Request-aware reuse} runs the pilot with a random paid window per rider. It then fits a window-relative tilt on four features (completions and active days inside and outside the paid window), pooled over the pilot riders' windows. Stored weeks and pilot riders are conditioned on the request's selection. The tilt is fitted once per origin for every group, rule, selection and piloted offer, and reused by every request.
\end{itemize}
The objective is the uplift in weekly completions, and the payment follows the request's window and rule. The truth is exact.

\begin{table}[ht]
\caption{Study D: relative regret within 20\% of the budget on the 216 window-dependent laws (8 replicates $\times$ 60 requests $\times$ 3 budgets). Selected: requests that target a subpopulation chosen on the previous week.}
\label{tab:transfer}
\centering\small\setlength{\tabcolsep}{4pt}
\begin{tabular}{lcccc}
\toprule
Method & 1 wk & 2 wk & 4 wk & Selected, 1 wk\\
\midrule
\textbf{Request-aware reuse} & \textbf{0.380} & \textbf{0.329} & \textbf{0.297} & \textbf{0.395}\\
Direct reuse of the full-week calibration & 0.576 & 0.584 & 0.596 & 0.742\\
\bottomrule
\end{tabular}
\end{table}

\paragraph{Results of study D.}
All five registered hypotheses hold (Table~\ref{tab:prereg}).
\begin{itemize}
\item When the window changes behaviour, request-aware reuse lowers regret against direct reuse by 0.195 [0.189, 0.202], 0.255 [0.249, 0.261] and 0.300 [0.293, 0.306] at one, two and four weeks. For requests that target a selected subpopulation the reduction is 0.347 [0.338, 0.356] at one week.
\item On all 288 response laws, request-aware reuse has lower regret than the same-calendar trial by 0.129 [0.123, 0.136] at one week.
\item The reused tilt returns the same answers as a tilt re-fitted for every request in all 18,432 comparisons at each pilot length. With preparation charged it is 3.13$\times$ [3.09, 3.16] faster over 60 requests at one week (3.07$\times$ and 3.09$\times$ at two and four weeks).
\end{itemize}
The study makes the reuse rule explicit. A cached calibration is valid for the windows and populations on which the response law was measured. When the paid window or the selection changes behaviour, the cache is keyed by window-relative features and by the selection, and the pilot must randomize windows. Its cost is the pilot's calendar weeks, which are charged above.

\paragraph{Study F: value retained by population reduction.}
Study F prices the reduction stage against a common individual-level optimum. It draws 3,072 laws from the rider-level generator (1,024 in each response family) and scores every method with the exact per-type tables, so estimation error is absent and only the loss from giving one plan to a whole layer remains. Each request takes a window, one rule and two of the six groups, with three budgets; every budget binds. A rider's history statistic is its total completions over the previous four no-offer weeks. The individual-level optimum lets every rider receive its own offer; since riders of one latent type are exchangeable, it is a linear programme over types, computed exactly as the sup-convolution of the types' upper concave hulls, and every layered plan is a feasible point of it. The framework's reduction cuts each group into $L$ equal-share quantile layers of the history statistic with one offer per layer, and the layered optimum is found by exact enumeration. The endpoint is the retained share: the value of a method summed over a law's requests, divided by the individual optimum summed over the same requests.

\begin{table}[ht]
\caption{Study F: share of the individual-level optimum retained by $L$ history-quantile layers per group, within 20\% of the budget; 3,072 laws, 60 requests $\times$ 3 budgets each; 95\% law-bootstrap intervals. By family: saturating response on both channels, intensity only, threshold response near the goal.}
\label{tab:layers}
\centering\small\setlength{\tabcolsep}{4pt}
\begin{tabular}{lcccc}
\toprule
Layers per group & All laws & Both channels & Intensity & Threshold\\
\midrule
1 (group-level plan) & 0.825 [0.820, 0.831] & 0.938 & 0.944 & 0.594\\
2 & 0.900 [0.897, 0.904] & 0.966 & 0.971 & 0.764\\
3 & 0.926 [0.923, 0.928] & 0.974 & 0.978 & 0.825\\
4 & 0.938 [0.936, 0.940] & 0.977 & 0.982 & 0.856\\
\textbf{5} & \textbf{0.946 [0.944, 0.948]} & \textbf{0.979} & \textbf{0.984} & \textbf{0.875}\\
\bottomrule
\end{tabular}
\end{table}

\paragraph{Results of study F.}
\begin{itemize}
\item Five history layers retain 94.6\% [94.4, 94.8] of the individual-level optimum, against 82.5\% [82.0, 83.1] for one group-level plan: 0.120 [0.116, 0.125] more, which recovers 68.9\% [68.8, 69.0] of the group-level loss (Table~\ref{tab:layers}).
\item Every added layer raises the retained share: five layers retain 0.046 [0.044, 0.047] more than two, and the fourth and fifth layers add 0.013 [0.012, 0.013] and 0.008 [0.007, 0.008].
\item Retention is highest when the response is smooth (97.9\% and 98.4\% with five layers) and lowest under a threshold response near the goal (87.5\%), where the plan that pays depends most on a rider's distance to the tier.
\item No layered plan abstains in this band and every budget binds, so the loss measured here is the structural cost of reduction; estimation error, which studies A--E price, comes on top of it.
\end{itemize}

\section{Public retail data}
\label{app:sources}\label{app:details}\label{app:reproduction}\label{app:public-protocol}\label{app:joint-events}
\paragraph{M5.}
M5 \citep{makridakis2022m5accuracy} provides 3,000 series (300 per store) from an 873-day panel, 45 weekly origins and seven-day
targets; eligibility uses the preceding 182 days and five groups use positive-sales days in the preceding 28 days. Each forecast is
trained on 13 earlier weekly origins, and 2,048 aligned scenarios keep cross-day ranks. Scale tests use all 30,490 series at three
origins. Sales carry no incentive response, so M5 tests computation and settlement transfer only.

\section{Computation and approximation boundaries}
\label{app:matched-computation}\label{app:scale-crossover}\label{app:m5-envelope}
Complete time runs from loading the inputs to returning every requested plan and includes all preparation; model fitting is reported
separately. All methods in a comparison receive the same laws, menus, budget bands and the same permission to reuse inputs.

\paragraph{Reuse across requests.}
On 3,000 M5 series and 45 origins, reusing a valid law instead of regenerating it for each request lowers the median request time from
84.78 to 4.36 seconds (paired $20.00\times$) with identical scenarios and identical exact values on all 405 tables; charging the
80.63-second preparation to the framework, 8, 32 and 127 requests per origin are $1.66\times$, $3.92\times$ and $7.11\times$ faster. On all 30,490 series the times are 82.62 and 1,146.49 seconds ($13.88\times$), and all nine
optima agree. With the 1,063.87-second preparation charged, eight requests are \nmMfullEight$\times$ faster (median over the three origins).

\paragraph{Per-request mixed-integer programming.}\label{app:pwl-comparator}
The multiple-choice piecewise-linear MIP uses binaries $y_s$ and integers $t_s$ with
\begin{align*}
 &\textstyle\sum_{s\in g}y_s=1, &&0\le t_s\le(U_s-L_s)y_s,\\
 &\textstyle x=\sum_s(L_sy_s+t_s), &&\textstyle v=\sum_s[(a_sL_s+b_s)y_s+a_st_s]/D,
\end{align*}
solved to zero gap with a time limit per request; every answer is
checked in exact arithmetic. Over 396 conditions (180 M5, 216 controlled) with 33 budget bands each, the dynamic program reuses its
value curve across requests and is $5.52\times$ faster than the envelope MIP at 33 requests (raw dense allocation: $5.47\times$), with
the same optimal value on every request the MIP resolves. Each condition's 33 budgets were answered in five balanced method orders, a timing
repetition, which gives 65,340 requests per method; the raw and envelope MIP formulations leave 25 and 5
unresolved, while all dynamic-programming variants resolve every request. The gain over per-request MIP comes from reusing the value curve.

\paragraph{New payment rule with the same laws.}\label{app:active-day-transfer}
For volume tier $J_{is}$ and active days $D_{is}$, the active-day rule pays $7D_{is}\rho_{J_{is}}$ when $D_{is}\ge d$ and $J_{is}>0$.
The histogram
\[
 H_s(k,j)=\sum_{i\in G_l}\mathbf 1\{D_{is}=k,J_{is}=j\}
\]
gives the exact coefficients
\[
 A_{sj}(d)=7\sum_{k=d}^7kH_s(k,j)
\]
for all eight attendance requirements. Over 45 origins, 2,048 scenarios, five groups and 270 reward paths, all 14,580 answers equal
direct settlement, and shared preparation cuts complete time at every origin, by \nmActMean\% on average (\nmActMin--\nmActMax\%).

\section{Controlled response instance}
\label{app:controlled-response}\label{app:unified-response}\label{app:response-experiment}\label{app:multilaw-response}\label{app:strong-response-controls}
\paragraph{Support, settlement and objective.}
Five groups share 32 weekly trajectory labels (29 distinct trajectories) with daily completions
\[
 y_{kd}=\begin{cases}2(\lfloor k/8\rfloor+1)+((d+k)\bmod 3), & d<k\bmod 8,\\ 0, & \text{otherwise},\end{cases}
\]
and $C_k=\sum_dy_{kd}$, $D_k=\sum_d\mathbf 1\{y_{kd}>0\}$.
Groups use fixed-prize, per-completion, attendance, fixed-prize and per-completion rules. Each group has nine offers $a=3b+t$,
$b,t\in\{0,1,2\}$, with rung rates $(\ell_j+40b+10t)/100$, $\ell=(10,15,20)$. Fixed prizes and per-completion rates use completion
rungs $(14,28,42)$ and pay $\mathrm{rung}_j\,r_j$ or $C_kr_j$; attendance pays $7D_kr_j$ at rungs of $(2,4,6)$ active days; below the
first rung the payment is zero. With uplift $U_g$, mean payment $M_g$ and stability $A_g=\Pr(0.8M_g\le Y_g\le 1.2M_g)$, the composite
objective is
\[
 J=\sum_g[U_g+M_gA_g]/104;
\]
requests are $B=0,5,\ldots,160$ with band $[0.9B,1.1B]$ on expected total payment, over the
$9^5=59{,}049$ portfolios.

\paragraph{Response laws.}
With $x=2b+t$, base weights $w_{gk}\in\{1,\ldots,17\}$, elasticity $e_g$, threshold $\tau_g$ and rule index $j$, offer-conditioned
masses are proportional to $w_{gk}f_{gk}(x)$ with
\begin{align*}
 f^{\rm null}&=1,\\
 f^{\rm sat}&=128+e_g(j+1)\tfrac{x}{2+x}C_k,\\
 f^{\rm thr}&=128+e_g(j+1)(x-\tau_g)_+^2C_k,\\
 f^{\rm trade}&=256+e_g\bigl[(j+1)\tfrac{16x}{1+x}D_k+\tfrac{x(6-x)}6C_k\bigr].
\end{align*}
The null and threshold families add a shared weekly shock $z$ through the factor
\[
 32+u_g\bigl[zC_k+(1-z)(C_{\max}-C_k)\bigr];
\]
the no-offer law omits $f$.

\section{Settlement interface and exact allocation}
\label{app:accounting}\label{app:operational}\label{app:historical-components}
Table~\ref{tab:app-modules} lists the seven modules of Figure~\ref{fig:pipeline} (the figure's short names in parentheses), their inputs and outputs, and the instance used on
delivery data. Each module can be replaced by any estimator with the same input and output; allocation exactness says nothing about
the accuracy of Modules~1--5.

\begin{algorithm}[ht]
\caption{Answering a request $r=(\mathcal I,W,x,\Gamma,\mathcal B)$}\label{alg:request}
\small
\begin{algorithmic}[1]
\State \textbf{Once per origin:} fit M3 on pre-origin history; draw $R$ joint per-day samples for every rider (cross-day ranks from each rider's history); store.
\State \textbf{Once per rule family:} run a pilot of $n$ rider-weeks at three offers; fit $\lambda_{l,a}$ by Eq.~\eqref{eq:tilt}.
\State Select roster $\mathcal I$ (M1--M2) and form layers $\mathcal G$ with plan menus $\mathcal A_l$ from recent activity (M4).
\For{each layer $l$ and plan $a\in\mathcal A_l$}
  \State Compose window $W$ from stored samples; tilt to $q_{l,a}$ by reweighting the stored samples, or the trajectory types when they can be enumerated (M5).
  \State Compute award probabilities and $(\widehat c_l(a),\widehat v_l(a))$ for rule $\Gamma$ as the weighted sum of Eq.~\eqref{eq:generic-score} over the stored samples, the exact sum when types are enumerable; no redraw (M6).
\EndFor
\State Reduce menus by the rule admissible for $\mathcal B$ (M4); solve Eq.~\eqref{eq:allocation} on $\mathcal B$ by exact enumeration or the budget-grid program (M7).
\State \Return plans with attaining rewards, or report that no plan satisfies $\mathcal B$.
\end{algorithmic}
\end{algorithm}

\begin{table}[ht]
\caption{Four stages and seven modules. The last column gives the instance used on the delivery platform.}\label{tab:app-modules}\label{tab:generality-contract}
\centering\scriptsize
\begin{tabular}{@{}p{.09\linewidth}p{.15\linewidth}p{.34\linewidth}p{.29\linewidth}@{}}\toprule
Stage & Module & Input $\to$ output & Delivery instance\\\midrule
Predict & 1. Response ranking (Rank) & pre-decision context and past offers $\to$ ranking by incremental activity & doubly robust learner with a 5\% no-offer group \citep{dudik2011doubly}\\
 & 2. Roster (Recall) & ranking and eligibility $\to$ population fixed before attendance, incl.\ zero-activity riders & at least one delivery in the preceding 182 days\\
 & 3. Baseline forecast (Forecast) & histories and context $\to$ no-offer outcome law for the requested window & recurrent network, 60-day history, five-day zero-inflated Beta shares\\
Reduce & 4. Grouping (Activity layers) & baseline capacity $\to$ groups and non-dominated plan menus & ordered buckets merged to balance dispersion and spending stability\\
Integrate & 5. Response correction (Correct) & no-offer law and pilot or context-matched outcomes $\to$ offer-conditioned joint trajectories & exponential tilt, Eq.~\eqref{eq:tilt}; residual matching on incentive, time and weather\\
 & 6. Settlement (Score) & trajectories, payment rule, attainable rewards $\to$ cost, award probability, score and reward witness & Eqs.~\eqref{eq:fixed-payment}--\eqref{eq:rate-payment}\\
Allocate & 7. Budget allocation (Allocate) & group domains and both budget bounds $\to$ reward portfolio or infeasibility & two-sided Bellman recursion, Eq.~\eqref{eq:segment-bellman}\\\bottomrule
\end{tabular}
\end{table}

\paragraph{Settlement identities.}
For nested award events $E_1\supseteq\cdots\supseteq E_J$ and $E_{J+1}=\varnothing$, a highest-rung prize pays
\[
 g=\sum_jR_j\mathbf 1\{E_j\setminus E_{j+1}\};
\]
taking expectations gives Eq.~\eqref{eq:fixed-payment} without independence across riders. A rate paid on every completion satisfies, by telescoping,
\[
 g=\sum_j(\rho_j-\rho_{j-1})Y\mathbf 1\{E_j\},
\]
which gives Eq.~\eqref{eq:rate-payment}. For integrable quantile functions,
\begin{equation}
 \mathbb E Y(a)-\mathbb E Y(0)=\int_0^1[Q_a(v)-Q_0(v)]\,dv ,\label{eq:uplift-area}
\end{equation}
so uplift is integrated once on the final response curve, which avoids counting overlapping rung effects twice; the identity is an
accounting statement, not causal identification. Mean payment and award probability do not determine spending stability: if one or
two completions are equally likely, a fixed prize $3/2$ and a unit rate $1$ have the same mean payment and award probability, but
they fall in $[1.2,1.8]$ with probability one and zero, respectively. This is why M6 integrates joint trajectories.

\subsection{Exact allocation over attaining score segments}\label{app:segment-allocation}\label{app:generality}
Fix a conditional trajectory law and uplift $u_{lq}$ on each group/path $(l,q)$. On a straight reward path
$r(t)=r^-+t(r^+-r^-)$, $0\le t\le 1$, with expected-payment endpoints $\mu^-<\mu^+$,
\begin{equation}
 \begin{aligned}
 t(\mu)&=\frac{\mu-\mu^-}{\mu^+-\mu^-},\\
 S_s(\mu)&=k_s^\top r(t(\mu)),
 \end{aligned}\label{eq:segment-attainment}
\end{equation}
where $k_s$ is the payment coefficient vector of scenario $s$ under the rule. Each closed event $0.8\mu\le S_s(\mu)\le1.2\mu$ is an
interval, a point or empty, so its probability $p_{lq}(\mu)$ is piecewise constant and the score
\[
 v_{lq}(\mu)=u_{lq}/T_0+\lambda\mu p_{lq}(\mu)/T_0
\]
is piecewise affine ($T_0$ is a common normalizer). Gaps between paths are
kept, never interpolated, and a path is used only inside one response branch.

\paragraph{Same-spend reduction.}
For attainable sets $I_{lq}$ keep the envelope $F_l(b)=\max_{q:b\in I_{lq}}v_{lq}(b)$ with an attaining reward. With additive scores
this preserves the optimum under any constraint that depends only on the group amounts: each candidate can be replaced by a no-worse
witness at the same amount, and every witness was feasible. Lemma~\ref{lem:admissible-reduction} states which menu reductions are
admissible for which budget.

\begin{lemma}[Admissible menu reductions]\label{lem:admissible-reduction}
Let each layer $l$ choose one plan from a finite menu $\mathcal A_l$ with cost $c_l(a)\ge0$ and value $v_l(a)$, and let a portfolio be feasible when $\sum_lc_l(a_l)\in\mathcal B$.
(i) If $\mathcal B=[0,\bar b]$ (a sole cap), deleting a plan $a$ for which some $a'\in\mathcal A_l$ has $c_l(a')\le c_l(a)$ and $v_l(a')\ge v_l(a)$ (ties broken by a fixed order) preserves feasibility and the optimal value.
(ii) For any $\mathcal B$, deleting $a$ when some $a'$ has $c_l(a')=c_l(a)$ and $v_l(a')\ge v_l(a)$ preserves feasibility and the optimal value; in both cases the retained plan is an attaining witness.
(iii) Rule (i) is not admissible once $\mathcal B$ has a positive floor.
\end{lemma}
\begin{proof}
(i) Replacing $a_l$ by $a'$ in a feasible portfolio lowers total cost weakly, so it stays in $[0,\bar b]$, and raises value weakly; repeating the replacement ends on retained plans, since the fixed tie order excludes cycles.
(ii) The replacement leaves total cost unchanged, so membership in any $\mathcal B$ is unchanged.
(iii) One layer with plans $(c,v)=(1,2),(2,1)$ and $\mathcal B=[1.8,2.2]$: rule (i) deletes the second plan, which is the only feasible one.
\end{proof}

\paragraph{Which rule each experiment uses, and an independent check.}
Rider-record budgets are sole caps, so M4 applies rule (i) and the exact reference enumerates the products of the rule-(i) menus. The
controlled study of Section~\ref{sec:synthetic} has two-sided bands and applies no reduction: it enumerates all $9^5$ portfolios. The
segment and envelope allocators below use only rule (ii). To rule out an error that a reduction shared by solver and reference would
hide, we re-solved every rider table small enough for complete enumeration (at most $2\times10^8$ portfolios: \nmEnumBrute\ of the
\enumRone\ realized and \nmEnumBrutePred\ of the \nmEnumPred\ predicted tables) by a separately written complete enumeration without any
reduction: feasibility and optimal value agree with the reference on every one of them (the returned plan can differ among
portfolios of equal value). The budget-grid program loses value on \enumDPpredNZ\ of the \enumDPpredN\ predicted five-layer main-grid instances (largest \enumDPpredMax\%) and never exceeds the cap. In the
controlled study, \enumEcalls\ enumeration calls covering \enumEbands\ budget bands differ in \enumEdiff, and rerunning the study with
the independent solver's choices reproduces every stored result bitwise.

\paragraph{Segment recursion.}
Let $\delta=T_0/h$, $o_l=\min_q\mu^-_{lq}/\delta$ and $\mu_l=\delta(x_l+o_l)$ with integer $x_l$. With $R$ equally weighted scenarios
and stability count $n$, candidate $q$ of group $l$ attains score $\alpha x+\beta$, with
\begin{align*}
 \alpha&=\lambda n/(Rh),\\
 \beta&=u_{lq}/T_0+\alpha o_l,
\end{align*}
at every integer spend index $x\in[L,U]$ of a segment. A physical band $[B^-,B^+]$ becomes the integer interval
\[
 \Bigl[\max\bigl\{0,\lceil B^-/\delta-\textstyle\sum_l o_l\rceil\bigr\},\ \lfloor B^+/\delta-\textstyle\sum_l o_l\rfloor\Bigr].
\]
Each segment contributes
\begin{equation}
 \begin{aligned}
 D_l(b)=\max_{q,s}\Bigl\{&\alpha_{qs}b+\beta_{qs}\\
 &+\max_{k\in[b-U_{qs},\,b-L_{qs}],\ 0\le k<N}\bigl[D_{l-1}(k)-\alpha_{qs}k\bigr]\Bigr\}.
 \end{aligned}
 \label{eq:segment-bellman}
\end{equation}
Minimum-cost options, closed endpoints, gaps and unreachable states are retained, and no no-offer action is added.

\begin{proposition}[Segment and dense allocation coincide]\label{prop:segment-dense}
For a fixed law, additive scores, group-wise menus and exact rational arithmetic, the recursion~\eqref{eq:segment-bellman} returns the
same value curve and the same canonical rewards as dense allocation on the same grid.
\end{proposition}
\begin{proof}

\emph{Step 1: the recursion.} Set $D_0(0)=0$ and all unreachable values to $-\infty$. On a segment, substitute $x=b-k$ in $D_{l-1}(b-x)+\alpha x+\beta$; the
feasible predecessor interval gives Eq.~\eqref{eq:segment-bellman}, and maximizing over segments keeps every action.

\emph{Step 2: ties.} Ties within a
segment keep the largest $k$ (smallest current $x$); ties across segments keep the smallest $x$, then the smallest candidate index. By
induction the choices equal the dense ones.

\emph{Step 3: rewards.} Backtracking from the smallest optimal feasible total through
Eq.~\eqref{eq:segment-attainment} returns rewards with the same score and physical spend; an unreachable band is infeasible for both.
\end{proof}
A monotone queue evaluates each window maximum in $O(N)$, so $K$ segments, $L$ groups and $N$ budget states need $O(KN+LK+LN)$
operations against $O(LN^2+KN+LK)$ for dense allocation; both store $O(LN)$ choices. If each group's candidate intervals are pairwise
disjoint, the envelope equals the raw menu and its construction can be skipped without changing the solution.\label{app:identity-fastpath}
Compared with graphical dynamic programming for piecewise-linear investment under an upper cap \citep{gafarov2014graphical,gafarov2016effective}
and nonconvex piecewise-linear knapsacks \citep{kameshwaran2009nonconvex}, the recursion handles reward domains with gaps, two-sided budgets,
rule-dependent event scores and a reward witness for every scored spend.\label{app:allocation-positioning}

\section{Four-term loss decomposition: proofs and stage bounds}
\label{app:loss-proofs}

\paragraph{Setting.}
Layers $l=1,\dots,L$ have finite plan menus $\mathcal A_l$; a portfolio is $a=(a_1,\dots,a_L)$ with
\begin{align*}
 c(a)&=\sum_lc_l(a_l),\\
 v(a)&=\sum_lv_l(a_l).
\end{align*}
Tables without a superscript are the true ones; $(c^{(k)},v^{(k)})$ for $k=\mathrm{syn},\mathrm{mm},\mathrm{int},\mathrm{dec}$ are the successive tables of the framework and $(\widehat c,\widehat v)=(c^{\mathrm{dec}},v^{\mathrm{dec}})$. Stage errors $\varepsilon^c_k,\varepsilon^v_k$ and their sums $E_c,E_v$ are as in Section~\ref{sec:theory}. The band and the value it carries are
\begin{align*}
 \mathcal B&=[\underline b,\bar b],\\
 \mathcal B_\delta&=[\underline b+\delta,\bar b-\delta],\\
 V^\star(\delta)&=\max\{v(a):c(a)\in\mathcal B_\delta\},\\
 \omega(\delta)&=V^\star(0)-V^\star(\delta)\ge0,
\end{align*}
with $V^\star(\delta)=-\infty$ if no portfolio qualifies. Cost and uplift entries are expectations of bounded functions of the trajectory (payment, completions), or differences of two such expectations, so they are linear in the law; the same holds for a stability indicator with a fixed band. Parts (i)--(ii) use only table errors and therefore also cover the mean-centred stability score, which is not linear in the law. In particular, if a layer law is the mixture of its members' laws, $n_l$ times its expectation equals the sum of the members' expectations exactly; population reduction therefore adds no error to expected tables and acts only through the menus, which enter $V^\star$ and $\omega$.

\subsection{Proof of Proposition~\ref{prop:four-term}}
\paragraph{(i) Upper bound.}

\emph{Step 1: table errors.} For every portfolio, the triangle inequality along the chain gives
\begin{align*}
 |c(a)-\widehat c(a)|
 &\le\sum_l\max_{a'\in\mathcal A_l}|c_l(a')-\widehat c_l(a')|\\
 &\le\sum_l\sum_k\max_{a'}|c^{(k)}_l(a')-c^{(k^-)}_l(a')|\\
 &=E_c,
\end{align*}
and likewise $|v(a)-\widehat v(a)|\le E_v$.

\emph{Step 2: feasibility.} Since $\widehat c(\widehat a)\in\mathcal B_{E_c}$,
\[
 c(\widehat a)\in[\widehat c(\widehat a)-E_c,\ \widehat c(\widehat a)+E_c]\subseteq\mathcal B .
\]

\emph{Step 3: value.} Let $a^\circ$ attain $V^\star(2E_c)$. Then $\widehat c(a^\circ)\in\mathcal B_{E_c}$, so $a^\circ$ is admissible for the planner, the admissible set is nonempty, and $\widehat v(\widehat a)\ge\widehat v(a^\circ)-\eta$. Hence
\begin{align*}
 v(\widehat a)
 &\ge\widehat v(\widehat a)-E_v && \text{(Step 1)}\\
 &\ge\widehat v(a^\circ)-\eta-E_v && \text{($\eta$-optimality)}\\
 &\ge v(a^\circ)-2E_v-\eta && \text{(Step 1)}\\
 &=V^\star(2E_c)-2E_v-\eta .
\end{align*}

\emph{Step 4: conclusion.} Subtracting from $V^\star(0)$ gives Eq.~\eqref{eq:four-term}. If $\omega(\delta)\le\kappa\delta$ on $[0,2E_c]$,
\[
 \omega(2E_c)\le2\kappa E_c=\sum_k2\kappa\varepsilon^c_k .
\]
If the comparator ranges over a finer reward domain than the planner's menu, the difference between the two optima at $2E_c$ is added to $\eta$. The proof uses only $|c(a)-\widehat c(a)|\le E_c$ and $|v(a)-\widehat v(a)|\le E_v$, so (i) also holds with the direct errors
\begin{align*}
 \bar E_c&=\sum_l\max_a|c_l(a)-\widehat c_l(a)|\le E_c,\\
 \bar E_v&\le E_v
\end{align*}
in their place. \qed

\paragraph{(ii) Lower bound.}

\emph{Step 1: the instance.} Fix stage errors $(\varepsilon^c_k,\varepsilon^v_k)_k$ with $E_c>0$ (the case $E_c=0$ is treated in Step~5), any $\Delta>0$, and a band with $\bar b-\underline b>4E_c$ and midpoint $m$. Order the stages as syn, mm, int, dec and write $s_k(\varepsilon)=\sum_{j\le k}\varepsilon_j$ for the partial sums along this order.
Layer~1 has plans $s$ and $h$ with true $(c,v)=(m,0)$ and $(\bar b-E_c,\Delta)$. Plan $s$ is exact at every stage; plan $h$ has exact value and cost
\[
 c^{(k)}(h)=\bar b-E_c+s_k(\varepsilon^c),
\]
so that $\widehat c(h)=\bar b$.
Layer~2 has plans $x,y$ of zero cost at every stage, true values $0$ and $2E_v$, and
\begin{align*}
 v^{(k)}(x)&=s_k(\varepsilon^v),\\
 v^{(k)}(y)&=2E_v-s_k(\varepsilon^v),
\end{align*}
so that $\widehat v(x)=\widehat v(y)=E_v$.
Each stage error is exactly the prescribed one.

\emph{Step 2: the benchmark.} In $\mathcal B_{2E_c}$ plan $h$ is excluded. For this truth
\begin{align*}
 V^\star(0)&=\Delta+2E_v &&\text{(attained by $(h,y)$ at spend $\bar b-E_c\in\mathcal B$)},\\
 V^\star(2E_c)&=2E_v &&\text{(via $(s,y)$)},\\
 \omega(2E_c)&=\Delta .
\end{align*}

\emph{Step 3: equality for the planner of (i).} With $\eta=0$ it cannot use $h$ because $\widehat c(h)=\bar b\notin\mathcal B_{E_c}$; breaking the tie in layer~2 toward $x$ it earns $0$, so Eq.~\eqref{eq:four-term} holds with equality.

\emph{Step 4: every planner.} Now take any deterministic planner that sees the final tables and the stage-error magnitudes and keeps the true spend in $\mathcal B$ for every consistent truth. The truth with $c(h)=\bar b+E_c$ (the same final tables, reached by decreasing stage steps) makes every portfolio with $h$ infeasible, so the planner never selects $h$. In layer~2 its choice is a function of identical final tables, so the truth that gives value $0$ to the chosen plan and $2E_v$ to the other is consistent; the planner earns $0$ against $V^\star(0)=\omega(2E_c)+2E_v$. Abstaining also earns $0$. A randomized planner loses at least $\omega(2E_c)+E_v$ in expectation, by averaging over the two value truths.
Without tightening, a planner that maximizes $\widehat v$ over $\widehat c\in\mathcal B$ selects $(h,\cdot)$ with estimated spend $\bar b$, whose true spend is $\bar b+E_c$ under the second truth.
Adding a layer with two exact zero-cost plans of value $0$ and $\eta$, an $\eta$-optimal planner may choose the first, so equality also holds for $\eta>0$.

\emph{Step 5: the case $E_c=0$.} Then $\mathcal B_{E_c}=\mathcal B$, plan $h$ would stay admissible and $\omega(0)=0$, so drop $h$: layer~1 has only the exact plan $s$ and layer~2 is as above. Now every cost is exact and $V^\star(0)=2E_v$, attained via $(s,y)$ at spend $m\in\mathcal B$. The planner of (i) breaking the tie toward $x$ earns $0$, and by the argument for layer~2 any deterministic planner earns $0$ against one of the two value truths. Hence the bound
\[
 \omega(0)+2E_v=2E_v
\]
is attained. \qed

\paragraph{(iii) Error floors of each omitted stage.}
All instances use exact tables for the stages that are kept.

\emph{Pilot correction.} Take one layer with two offers whose payment functions coincide, true offer laws $P_{a_1}$ equal to the no-offer law and $P_{a_2}$ with expected completions larger by $U>0$, and a band containing both costs. Without the correction both synthetic laws equal the no-offer law, the two offers have identical tables and estimated uplift $0$, and relabelling the offers makes any deterministic planner choose $a_1$: it loses the entire optimal uplift $U$.

\emph{History shape.} By Lemma~\ref{lem:cantelli} two laws with the same mean and variance can differ by $1/(1+t^2)$ in the probability of reaching $\mu+t\sigma$, so any estimate computed from the two moments alone errs by at least $1/(2(1+t^2))$ on one of them, and a fixed prize $R$ at that threshold is mispriced by at least $R/(2(1+t^2))$.

\emph{Trajectory integration.} Let $Y\in\{0,2\tau\}$ with probability $\tfrac12$ each, known exactly, and a prize $R$ paid when $Y\ge\tau$. Evaluating the rule at the mean gives
\[
 R\mathbf 1\{\E Y\ge\tau\}=R,
\]
whereas
\[
 \E[R\mathbf1\{Y\ge\tau\}]=R/2 .
\]
The probability that payment lies within $\pm20\%$ of its mean $R/2$ is $0$, while the plug-in value is $1$.

\emph{Joint allocation.} Under the cap $\mathcal B=[0,B]$ with $L\ge2$ layers, let layer~1 offer $(0,0)$ and $(B,U)$ and all other layers offer only $(0,0)$. The joint optimum earns $U$; splitting the budget equally gives layer~1 at most $B/L<B$, so it earns $0$. \qed

\begin{lemma}[Two moments do not fix a threshold probability]\label{lem:cantelli}
For $t>0$,
\[
 \sup\bigl|\Pr(S\ge\mu+t\sigma)-\Pr(S'\ge\mu+t\sigma)\bigr|=\frac1{1+t^2}
\]
over pairs of laws with common mean $\mu$ and variance $\sigma^2>0$, and the supremum is attained. For integer laws, $\{0{:}\tfrac12,2{:}\tfrac12\}$ and $\{0{:}\tfrac13,1{:}\tfrac12,3{:}\tfrac16\}$ both have mean and variance $1$ but $\Pr(S\ge2)$ equal to $\tfrac12$ and $\tfrac16$.
\end{lemma}
\begin{proof}

\emph{Upper bound.} For $u>0$,
\[
 \Pr(S-\mu\ge t\sigma)\le\frac{\E[(S-\mu+u)^2]}{(t\sigma+u)^2}=\frac{\sigma^2+u^2}{(t\sigma+u)^2},
\]
which equals $1/(1+t^2)$ at $u=\sigma/t$; both probabilities lie in $[0,1/(1+t^2)]$.

\emph{Equality.} $S=\mu+t\sigma$ with probability $1/(1+t^2)$ and $S=\mu-\sigma/t$ otherwise has mean $\mu$ and variance $\sigma^2$; $S'$ on $\{\mu+\alpha,\mu-\sigma^2/\alpha\}$ with $0<\alpha<t\sigma$ and mass $\sigma^2/(\alpha^2+\sigma^2)$ at $\mu+\alpha$ has the same moments and never reaches $\mu+t\sigma$. The integer example is checked directly.
\end{proof}

\subsection{Bounds on the four stage terms}
\begin{lemma}[Award probabilities carry every threshold rule]\label{lem:award-interface}
Let $Y\in\{0,\dots,M\}$ count eligible completions, $D$ be a common attendance or eligibility event, $E_j=\{Y\ge\tau_j\}\cap D$ with integers $1\le\tau_1<\dots<\tau_J$, and $G_F(k)=F(Y\ge k,D)$. For laws $F,F'$ write $\Delta(k)=|G_F(k)-G_{F'}(k)|$.
For fixed prizes $0=R_0\le R_1\le\dots\le R_J$ paid at the highest rung reached,
\[
 |\E_Fg-\E_{F'}g|\le\sum_j(R_j-R_{j-1})\Delta(\tau_j)\le R_J\max_j\Delta(\tau_j).
\]
For rates $0=\rho_0\le\dots\le\rho_J$ paid on all eligible completions at the highest rung reached,
\begin{align*}
 |\E_Fg-\E_{F'}g|
 &\le\sum_j(\rho_j-\rho_{j-1})\Bigl\{\tau_j\Delta(\tau_j)+\sum_{k>\tau_j}\Delta(k)\Bigr\}\\
 &\le\rho_JM\max_{k\ge\tau_1}\Delta(k).
\end{align*}
Expected completions satisfy
\[
 |\E_FY-\E_{F'}Y|\le\sum_{k=1}^M|F(Y\ge k)-F'(Y\ge k)|.
\]
\end{lemma}
\begin{proof}
Nested events give
\begin{align*}
 g&=\sum_j(R_j-R_{j-1})\mathbf1\{E_j\} &&\text{(fixed prizes)},\\
 g&=\sum_j(\rho_j-\rho_{j-1})Y\mathbf1\{E_j\} &&\text{(rates)}.
\end{align*}
With $p(y)=F(Y=y,D)$, summation by parts gives
\begin{align*}
 \E_F[Y\mathbf1\{E_j\}]
 &=\sum_{y\ge\tau_j}yp(y)\\
 &=\tau_jG_F(\tau_j)+\sum_{k>\tau_j}G_F(k),
\end{align*}
and, in the same way,
\[
 \E_FY=\sum_{k\ge1}F(Y\ge k).
\]
Take differences; at most $M-\tau_j$ terms follow $\tau_j$.
\end{proof}
Thus each table error of a threshold rule is controlled by errors in award probabilities $G(k)$. When $D$ requires attendance on several days, $G$ depends on the joint law of completions and attendance, not on their separate marginals: over two days with $D=\{\text{both days active}\}$ and $\tau_1=3$, the laws $\tfrac12\delta_{(2,1)}+\tfrac12\delta_{(3,2)}$ and the uniform law on $\{(2,1),(2,2),(3,1),(3,2)\}$ of (completions, active days) share both marginals but give $G(3)=\tfrac12$ and $\tfrac14$. This is why integration uses joint trajectories.

\begin{lemma}[Synthetic term]\label{lem:syn}
On a finite trajectory space let $p_0$ be the no-offer law, $\phi$ the matched features, $P$ a true offer law with $P\ll p_0$, and $Q_\lambda\propto p_0e^{\lambda^\top\phi}$. If $\lambda^\star$ solves $\E_{Q_{\lambda^\star}}\phi=\E_P\phi$, then
\[
 \mathrm{KL}(P\Vert p_0)=\mathrm{KL}(P\Vert Q_{\lambda^\star})+\mathrm{KL}(Q_{\lambda^\star}\Vert p_0),
\]
and for every function $f$ with range $\mathrm{osc}(f)=\max f-\min f$,
\[
 |\E_Pf-\E_{Q_{\lambda^\star}}f|\ \le\ \mathrm{osc}(f)\sqrt{\tfrac12\{\mathrm{KL}(P\Vert p_0)-\mathrm{KL}(Q_{\lambda^\star}\Vert p_0)\}} .
\]
\end{lemma}
\begin{proof}

\emph{Step 1: the identity.}
\begin{align*}
 \mathrm{KL}(P\Vert p_0)-\mathrm{KL}(P\Vert Q_{\lambda^\star})
 &=\E_P\log(Q_{\lambda^\star}/p_0)\\
 &=\lambda^{\star\top}\E_P\phi-\log Z,
\end{align*}
and $\mathrm{KL}(Q_{\lambda^\star}\Vert p_0)=\lambda^{\star\top}\E_{Q_{\lambda^\star}}\phi-\log Z$ is the same number.

\emph{Step 2: the bound.} Then $|\E_Pf-\E_Qf|\le\mathrm{osc}(f)\,\mathrm{TV}(P,Q)$ and Pinsker's inequality $\mathrm{TV}\le\sqrt{\mathrm{KL}/2}$ \citep[Lemma~2.5]{tsybakov2009introduction}.
\end{proof}
The same identity applied to any $Q$ with the matched moments shows $\mathrm{KL}(Q\Vert p_0)\ge\mathrm{KL}(Q_{\lambda^\star}\Vert p_0)$: the tilt is the least change of history that reproduces the experiment \citep{csiszar1975idivergence}. Without the correction the same argument gives only $\mathrm{osc}(f)\sqrt{\mathrm{KL}(P\Vert p_0)/2}$; matching removes $\mathrm{KL}(Q_{\lambda^\star}\Vert p_0)$ from the radicand, and the term vanishes when the response is a tilt in $\phi$. The bound covers piloted offers. At an unpiloted offer, $P^{\mathrm{syn}}$ uses the interpolated parameter
\[
 \lambda^\star_t=(1-t)\lambda^\star_j+t\lambda^\star_k
\]
of its piloted neighbours $j,k$ (Section~\ref{sec:theory}), so the synthetic term also contains the interpolation bias $|\E_Pf-\E_{Q_{\lambda^\star_t}}f|$, which the lemma does not bound and which persists however large the pilot. For example, with $Y\sim$ Bernoulli$(\tfrac12)$ under no offer, $\phi(Y)=Y$ and a true response $P_a(Y=1)=\sigma(a^2)$ inside the tilt family ($\sigma$ the logistic function), pilots at $a=0.2$ and $0.6$ give $\lambda^\star=0.04$ and $0.36$ exactly, but the interpolated $\lambda=0.20$ at $a=0.4$ differs from the true $0.16$, leaving an error of $\sigma(0.20)-\sigma(0.16)\approx0.0099$ in $\E Y$ with exact moments and exact integration. More pilot offers, or smoothness of $\lambda$ in the offer, reduce this part.

\begin{lemma}[Moment-matching term]\label{lem:mm}
Keep $p_0$ fixed and let $\Sigma_\lambda=\mathrm{Cov}_{Q_\lambda}(\phi)$ satisfy $\kappa_\phi I\preceq\Sigma_\lambda\preceq\Lambda_\phi I$, $\kappa_\phi>0$, on the segment between $\lambda^\star$ and the fitted $\widehat\lambda$, whose moments $\widehat m$ are the pilot means. Then for every $f$,
\begin{align*}
 |\E_{Q_{\widehat\lambda}}f-\E_{Q_{\lambda^\star}}f|
 &\le\tfrac12\mathrm{osc}(f)\,\Lambda_\phi^{1/2}\,\|\widehat\lambda-\lambda^\star\|_2\\
 &\le\tfrac12\mathrm{osc}(f)\,\frac{\Lambda_\phi^{1/2}}{\kappa_\phi}\,\|\widehat m-m^\star\|_2 .
\end{align*}
If each of $J_p$ piloted offers per layer receives $n$ independent trajectories and $\phi\in[0,1]^d$, then with probability at least $1-\delta$, simultaneously for all layers and piloted offers,
\[
 \|\widehat m-m^\star\|_2\le\sqrt{d\log(2dLJ_p/\delta)/(2n)}.
\]
Offers whose $\lambda$ is interpolated linearly between piloted offers, or between $\lambda=0$ at the no-offer plan and a piloted offer, satisfy the first inequality (given the covariance bounds on their own segment) with $\lambda^\star$ the interpolated synthetic parameter and $\|\widehat\lambda-\lambda^\star\|_2$ at most the larger error of their two neighbours; the comparison is with $P^{\mathrm{syn}}$, not with the projection of the offer's own population moments.
\end{lemma}
\begin{proof}

\emph{Step 1: derivative along the segment.} With $u=\widehat\lambda-\lambda^\star$ and $\lambda_t=\lambda^\star+tu$,
\begin{align*}
 \Bigl|\frac{d}{dt}\E_{Q_{\lambda_t}}f\Bigr|
 &=|\mathrm{Cov}_{\lambda_t}(f,u^\top\phi)|\\
 &\le\mathrm{sd}(f)\,(u^\top\Sigma_{\lambda_t}u)^{1/2}\\
 &\le\tfrac12\mathrm{osc}(f)\Lambda_\phi^{1/2}\|u\|.
\end{align*}

\emph{Step 2: from parameters to moments.} Also
\[
 \widehat m-m^\star=\int_0^1\Sigma_{\lambda_t}dt\,u
\]
with
\[
 \int_0^1\Sigma_{\lambda_t}dt\succeq\kappa_\phi I,
\]
so $\|u\|\le\|\widehat m-m^\star\|/\kappa_\phi$.

\emph{Step 3: the rate.} Hoeffding's inequality \citep{hoeffding1963probability} for each of the $dLJ_p$ coordinates with a union bound gives the rate.

\emph{Step 4: interpolated offers.} Interpolation weights are convex, so the interpolated error is a convex combination of the neighbours' errors.
\end{proof}
If the no-offer law is estimated from $N_0$ history weeks, changing the base at fixed $\lambda$ moves every expectation by at most
\[
 2\,\mathrm{osc}(f)e^{s(\lambda)}\mathrm{TV}(\widehat p_0,p_0),
\]
where
\[
 s(\lambda)=\max_z\lambda^\top\phi(z)-\min_z\lambda^\top\phi(z),
\]
since, for $w=e^{\lambda^\top\phi}$,
\[
 \|Q-Q'\|_1\le2\sum_zw|p_0-\widehat p_0|/\E_{p_0}w;
\]
and $\E\,\mathrm{TV}(\widehat p_0,p_0)\le\tfrac12\sqrt{|\mathcal Z|/N_0}$. With abundant history ($N_0\gg n$) this part is small; the measured moment term includes it.

This bound holds at fixed $\lambda$. The method, however, re-fits the tilt on $\widehat p_0$ so that it matches the pilot moments again, and the change of $\lambda$ must be controlled as well.

\begin{lemma}[Estimated base and re-fitted tilt]\label{lem:joint-base}
Let $Q=Q_\lambda[p_0]$ be the tilt of $p_0$ with moments $m$, let $Q'=Q_\lambda[\widehat p_0]$ keep $\lambda$ on the estimated base, and let $\widehat Q=Q_{\widehat\lambda}[\widehat p_0]$ be re-fitted to the same moments $m$. Assume support overlap, that is, $m$ lies in the interior of the convex hull of $\phi(\mathrm{supp}\,\widehat p_0)$. Let the feature covariance of the tilts of $\widehat p_0$ satisfy $\widehat\kappa I\preceq\Sigma\preceq\widehat\Lambda I$, $\widehat\kappa>0$, on the segment between $\lambda$ and $\widehat\lambda$. Put
\[
 \tau=2e^{s(\lambda)}\mathrm{TV}(\widehat p_0,p_0),\qquad D_\phi=\Bigl(\sum_i\mathrm{osc}(\phi_i)^2\Bigr)^{1/2}.
\]
Then for every $f$,
\[
 |\E_{\widehat Q}f-\E_Qf|\ \le\ \mathrm{osc}(f)\,\tau\Bigl(1+\tfrac12D_\phi\,\frac{\widehat\Lambda^{1/2}}{\widehat\kappa}\Bigr).
\]
\end{lemma}
\begin{proof}
\emph{Step 1: base change at fixed $\lambda$.} By the display above, $\mathrm{TV}(Q',Q)\le\tau$, hence $|\E_{Q'}f-\E_Qf|\le\mathrm{osc}(f)\tau$.

\emph{Step 2: moment shift.} Let $m'$ be the moments of $Q'$. Each coordinate moves by at most $\mathrm{osc}(\phi_i)\tau$, so $\|m'-m\|_2\le D_\phi\tau$.

\emph{Step 3: re-fit on the estimated base.} $Q'$ and $\widehat Q$ are tilts of the same base $\widehat p_0$ with moments $m'$ and $m$. Lemma~\ref{lem:mm}, whose proof uses only a fixed base and the covariance bounds on the segment, gives
\begin{align*}
 |\E_{\widehat Q}f-\E_{Q'}f|
 &\le\tfrac12\mathrm{osc}(f)\,\frac{\widehat\Lambda^{1/2}}{\widehat\kappa}\,\|m-m'\|_2\\
 &\le\tfrac12\mathrm{osc}(f)\,\frac{\widehat\Lambda^{1/2}}{\widehat\kappa}\,D_\phi\tau .
\end{align*}
The triangle inequality concludes.
\end{proof}
The three conditions are the quantities that the diagnostics of Appendix~\ref{app:strong} measure: support overlap (the target inside the range of the stored weeks), the weight range $e^{s(\lambda)}$, which bounds $N\max_iw_i$ on $N$ equally weighted stored weeks, and the conditioning $\widehat\Lambda/\widehat\kappa$ of the feature covariance. If the estimated base loses support on which the matched moments rely, $\widehat\kappa$ becomes small or $\widehat\lambda$ fails to exist, and the bound grows accordingly; stability is claimed only where these quantities are controlled. The rates of Corollary~\ref{cor:rates} are stated for a given base; on the event that the conditions hold, the estimated base adds this term. When the corrected law is the tilt of the empirical stored-week law, as in Appendices~\ref{app:structural} and~\ref{app:strong}, the weighted sum over the stored weeks is exact for that law, and the finite history enters through $\widehat p_0$ in this lemma rather than through the integration term.

\paragraph{Experiment length.}
Lemma~\ref{lem:mm} needs the experiment to fix $d$ moments at $J_p$ offers per layer, so its error scales as $\sqrt{d/n}$. A trial that estimates the complete law of every offer directly avoids the synthetic term, but the analogous bound for empirical frequencies is $\tfrac12\sqrt{|\mathcal Z|/n_{\mathrm R}}$ per offer and must hold at all $|\mathcal A_l|$ offers. Comparing these sufficient conditions, the trial needs roughly
\[
 \frac{|\mathcal A_l||\mathcal Z|}{J_pd}
\]
times as many rider-weeks for the same estimation error (48 in the controlled design with $|\mathcal Z|=32$, $|\mathcal A_l|=9$, $J_p=3$, $d=2$), up to logarithms and $\Lambda_\phi^{1/2}/\kappa_\phi$. The price is the synthetic term, which does not shrink with experiment length; this is a comparison of upper bounds, not a lower bound for trials.

\begin{lemma}[Integration term]\label{lem:int}
If integration averages $S$ independent draws from the moment-matched law for each of the $N=\sum_l|\mathcal A_l|$ layer--plan pairs, then with probability at least $1-\delta$ every cost and value entry is within $\mathrm{osc}_l\sqrt{\log(4N/\delta)/(2S)}$ of its exact value, where $\mathrm{osc}_l$ is the range of the integrand for layer $l$; hence
\[
 \varepsilon^c_{\mathrm{int}},\ \varepsilon^v_{\mathrm{int}}\ \le\ \sum_l\mathrm{osc}_l\sqrt{\log(4N/\delta)/(2S)}.
\]
Summing exactly over a finite trajectory space gives $\varepsilon_{\mathrm{int}}=0$.
\end{lemma}
\begin{proof}
Hoeffding's inequality for each of the $2N$ entries and a union bound. Draws may be shared across plans (common random numbers): the bound is per entry and needs no independence across entries.
\end{proof}

\paragraph{Scope of the rates.}
Lemmas~\ref{lem:syn}--\ref{lem:int} concern fixed integrands; Appendix~\ref{app:rates} gives the rates in expectation, including the conditions under which they extend to the mean-centred stability score and an example in which they fail. The controlled experiment uses pure uplift as its primary endpoint for this reason.

\begin{lemma}[Decision term]\label{lem:dec}
A budget-grid program with step $h$ that rounds each layer cost to the nearest grid point and optimizes exactly over the rounded tables is an exact planner ($\eta=0$) on tables with $\varepsilon^c_{\mathrm{dec}}\le Lh/2$ and $\varepsilon^v_{\mathrm{dec}}=0$. Rounding increments above each layer's minimum upward gives $0\le c^{\mathrm{dec}}-c^{\mathrm{int}}<h$ per layer, so $\varepsilon^c_{\mathrm{dec}}\le Lh$; exact enumeration and the segment program on its declared monetary grid have $\varepsilon_{\mathrm{dec}}=0$.
\end{lemma}
\begin{proof}
Rounding moves each layer cost by at most $h/2$ (respectively less than $h$) and leaves values unchanged; the program maximizes exactly over the rounded problem.
\end{proof}

\paragraph{Cross-day dependence.}
Suppose a member's day-activity indicators are conditionally independent Bernoulli$(p(\xi))$ given a shared factor $\xi$ over $H$ days, so $A$ counts active days, and let $A^{\mathrm{ind}}\sim\mathrm{Bin}(H,\E p(\xi))$ be the independent-day law with the same daily marginals. Then, by Jensen's inequality,
\[
 \Pr(A=H)=\E p(\xi)^H\ge(\E p(\xi))^H=\Pr(A^{\mathrm{ind}}=H),
\]
strictly unless $p(\xi)$ is constant. Since
\[
 \E A=\E A^{\mathrm{ind}}=\sum_{d\ge1}\Pr(A\ge d),
\]
the two survival functions must cross: composing days independently understates the full-attendance requirement and overstates some lower one.
\subsection{Rates of the moment-matching and integration terms}\label{app:rates}
The rates quoted in Section~\ref{sec:theory} hold in expectation for every table entry with a fixed integrand. They do not hold for a stability score whose band is centred at the law's own mean payment: in the two examples below such an entry keeps an error of constant size while the pilot moments converge. This subsection states the rates with the conditions they need.

\paragraph{Setting.}
Fix a layer and a piloted offer. The trajectory space $\mathcal Z$ is finite, the base $p_0$ has full support, and the features $\phi:\mathcal Z\to[0,1]^d$ do not all lie in one hyperplane. Write
\begin{align*}
 Q_\lambda&\propto p_0e^{\lambda^\top\phi},\\
 \Sigma_\lambda&=\mathrm{Cov}_{Q_\lambda}(\phi),\\
 m(\lambda)&=\E_{Q_\lambda}\phi .
\end{align*}
The pilot mean $\widehat m$ averages $\phi$ over $n$ independent trajectories drawn from the true offer law $P$; when riders of one week share a shock, $n$ counts independent weeks or blocks, not rider-weeks. With $m^\star=\E_P\phi$,
\begin{align*}
 P^{\mathrm{syn}}&=Q_{\lambda(m^\star)},\\
 P^{\mathrm{mm}}&=Q_{\lambda(\widehat m)};
\end{align*}
the base is held fixed, and its estimation adds the term stated after Lemma~\ref{lem:mm}. The closed ball $\bar B(m^\star,\rho)$ lies in the interior of $\mathrm{conv}\,\phi(\mathcal Z)$, and $\kappa_\rho$ is the smallest eigenvalue of $\Sigma_{\lambda(m)}$ over that ball. If $\widehat m$ is not interior, the fit is replaced by any law or by abstention, and the entry stays in its range.

The map $\lambda\mapsto m(\lambda)$ is a bijection from $\R^d$ onto the interior of $\mathrm{conv}\,\phi(\mathcal Z)$ with a smooth inverse, so $\lambda(m)$ is well defined and $\kappa_\rho>0$. Indeed, $u^\top\Sigma_\lambda u=\mathrm{Var}_{Q_\lambda}(u^\top\phi)>0$ for $u\neq0$. For interior $m$, the strictly convex function
\[
 \psi(\lambda)=\log\sum_zp_0(z)e^{\lambda^\top\phi(z)}-\lambda^\top m
\]
satisfies, for every unit vector $u$,
\[
 \psi(tu)\ge t\{\max_zu^\top\phi(z)-u^\top m\}+\log\min_zp_0(z)\to\infty ,
\]
so it has a unique minimizer, where $m(\lambda)=m$. The inverse function theorem gives smoothness, and continuity on the compact ball gives $\kappa_\rho>0$.

\paragraph{Mean-centred stability.}
For a law $q$ and a payment $y:\mathcal Z\to[0,\bar y]$ let
\begin{align*}
 M(q)&=\E_qy,\\
 S(q)&=\{z:\alpha M(q)\le y(z)\le\beta M(q)\}
\end{align*}
with $0<\alpha<1<\beta$ (the paper uses $0.8$ and $1.2$). The monetary stability score is $F(q)=M(q)\,q(S(q))$, and its \emph{margin}
\[
 \gamma(q)=\min_z\min\{|y(z)-\alpha M(q)|,|y(z)-\beta M(q)|\}
\]
is the distance from the band's endpoints to the nearest payment value.

\begin{corollary}[Rates of the moment-matching and integration terms]\label{cor:rates}\label{cor:resources}
In the setting above let $M^\star=M(P^{\mathrm{syn}})$, $\gamma=\gamma(P^{\mathrm{syn}})$ and $r=\min\{\rho,\gamma\kappa_\rho^{1/2}/(\beta\bar y)\}$.

(a) \emph{Fixed integrand.} For every $f$ (payment, completions, uplift, an award indicator, or stability in a band with a fixed centre),
\begin{align*}
 \E|\E_{P^{\mathrm{mm}}}f-\E_{P^{\mathrm{syn}}}f|&\le\mathrm{osc}(f)\bigl\{\sqrt d/(4\sqrt{\kappa_\rho n})+2de^{-2n\rho^2/d}\bigr\},
\end{align*}
and for $S$ independent draws from $P^{\mathrm{mm}}$,
\[
 \E|\E_{P^{\mathrm{int}}}f-\E_{P^{\mathrm{mm}}}f|\le\mathrm{osc}(f)/(2\sqrt S).
\]
Summed over the finite menu, $\E\varepsilon_{\mathrm{mm}}=O(n^{-1/2})$ and $\E\varepsilon_{\mathrm{int}}=O(S^{-1/2})$ for these entries; exact summation gives $\varepsilon_{\mathrm{int}}=0$.

(b) \emph{Centred band with a margin.} If $\gamma>0$, then
\[
 \E|F(P^{\mathrm{mm}})-F(P^{\mathrm{syn}})|\le\bar y\bigl\{\sqrt d/(2\sqrt{\kappa_\rho n})+2de^{-2nr^2/d}\bigr\}.
\]
Given $P^{\mathrm{mm}}$ with margin $\gamma'>0$,
\[
 \E|F(P^{\mathrm{int}})-F(P^{\mathrm{mm}})|\le\bar y\bigl\{1/\sqrt S+2e^{-2S\gamma'^2/(\beta\bar y)^2}\bigr\}.
\]

(c) \emph{No margin.} If $\gamma=0$, $\mathrm{Cov}_{P^{\mathrm{syn}}}(\phi,y)\neq0$ and $\mathrm{Cov}_P(\phi)$ is nonsingular, then, as $n\to\infty$,
\[
 \E|F(P^{\mathrm{mm}})-F(P^{\mathrm{syn}})|\to\tfrac12M^\star(w_-+w_+)>0,
\]
where $w_-$ and $w_+$ are the $P^{\mathrm{syn}}$-masses of the trajectories that pay exactly $\alpha M^\star$ and $\beta M^\star$.

For an offer whose tilt is interpolated between piloted offers, (a) and (b) hold with the constants of Lemma~\ref{lem:mm-rate}(c).
\end{corollary}
Under (b) the rate is $n^{-1/2}$, but it sets in only once $n\gtrsim d\beta^2\bar y^2/(\kappa_\rho\gamma^2)$, so a small margin delays it. Neither (b) nor (c) affects the upper bound of Proposition~\ref{prop:four-term}(i), which uses the measured table errors, whatever their size. Centring the band at a quantity fixed before estimation, such as the requested spend share or the mean issued with the plan, turns the score into a fixed-integrand entry covered by (a).

\begin{lemma}[Moment-matching term in expectation]\label{lem:mm-rate}
Let $g_f(m)=\E_{Q_{\lambda(m)}}f$.
(a) For $m,m'\in\bar B(m^\star,\rho)$,
\[
 |g_f(m')-g_f(m)|\le\tfrac12\mathrm{osc}(f)\,\kappa_\rho^{-1/2}\|m'-m\|_2 .
\]
(b) $\E\|\widehat m-m^\star\|_2\le\sqrt{d/(4n)}$, and $\Pr(\|\widehat m-m^\star\|_2>s)\le2de^{-2ns^2/d}$ for $s>0$.
(c) Let an offer's tilt be $\lambda_t=(1-t)\lambda_j+t\lambda_k$ for piloted neighbours $j,k$ (at the no-offer plan $\lambda=0$ is exact), and let both pilot means lie in their balls, with $\kappa_\rho$ the smaller of the two constants. Then
\[
 |\E_{Q_{\widehat\lambda_t}}f-\E_{Q_{\lambda^\star_t}}f|\le\tfrac{\sqrt d}{4}\mathrm{osc}(f)\,\kappa_\rho^{-1}\max_{i\in\{j,k\}}\|\widehat m_i-m^\star_i\|_2 .
\]
\end{lemma}
\begin{proof}
(a) By the chain rule, $\nabla g_f(m)=\Sigma^{-1}c$ with $\Sigma=\Sigma_{\lambda(m)}$ and $c=\mathrm{Cov}_{Q_{\lambda(m)}}(\phi,f)$. Since $c^\top\Sigma^{-1}c$ is the variance of the best linear predictor of $f$ from $\phi$,
\begin{align*}
 \|\Sigma^{-1}c\|_2^2
 &\le\kappa_\rho^{-1}c^\top\Sigma^{-1}c\\
 &\le\kappa_\rho^{-1}\mathrm{Var}(f)\\
 &\le\kappa_\rho^{-1}\mathrm{osc}(f)^2/4 .
\end{align*}
Integrate along the segment from $m$ to $m'$, which stays in the ball.

(b) Since $\phi\in[0,1]^d$,
\begin{align*}
 \E\|\widehat m-m^\star\|_2^2
 &=\mathrm{tr}\,\mathrm{Cov}_P(\phi)/n\\
 &\le d/(4n).
\end{align*}
For the tail apply Hoeffding's inequality \citep{hoeffding1963probability} to each coordinate at level $s/\sqrt d$ and take a union bound.

(c) \emph{Step 1: covariance bound.} For $\phi\in[0,1]^d$,
\[
 \Sigma\preceq(\mathrm{tr}\,\Sigma)I\preceq\tfrac d4I .
\]
Along $\lambda^\star_t+\tau u$ with $u=\widehat\lambda_t-\lambda^\star_t$, the derivative of the expectation is $\mathrm{Cov}(f,u^\top\phi)$, and
\begin{align*}
 |\mathrm{Cov}(f,u^\top\phi)|
 &\le\tfrac12\mathrm{osc}(f)(u^\top\Sigma u)^{1/2}\\
 &\le\tfrac12\mathrm{osc}(f)\tfrac{\sqrt d}{2}\|u\| .
\end{align*}

\emph{Step 2: integrate along the segment.} By convexity,
\[
 \|u\|\le\max_i\|\widehat\lambda_i-\lambda^\star_i\| .
\]
Integrating $d\lambda/dm=\Sigma^{-1}$ along the segment from $m^\star_i$ to $\widehat m_i$ gives
\[
 \|\widehat\lambda_i-\lambda^\star_i\|\le\kappa_\rho^{-1}\|\widehat m_i-m^\star_i\| .
\]
\end{proof}

\begin{lemma}[Stability of the centred band]\label{lem:band}
Let $q,q'$ be laws on $\mathcal Z$.
(i) If $\beta|M(q')-M(q)|<\gamma(q)$, then $S(q')=S(q)$ and
\[
 |F(q')-F(q)|\le|M(q')-M(q)|+M(q)\,|q'(S(q))-q(S(q))|.
\]
(ii) Let $q$ have full support and $M(q)>0$, let $Z_-$ and $Z_+$ be the trajectories paying exactly $\alpha M(q)$ and $\beta M(q)$, and let $w_\pm=q(Z_\pm)$. If $q_k\to q$ with $M(q_k)>M(q)$ for all $k$, then
\[
 F(q_k)\to M(q)\{q(S(q))-w_-\};
\]
if $M(q_k)<M(q)$ for all $k$, then $F(q_k)\to M(q)\{q(S(q))-w_+\}$.
\end{lemma}
\begin{proof}
(i) Write $M=M(q)$, $M'=M(q')$, $\gamma=\gamma(q)$ and $\Delta=|M'-M|<\gamma/\beta$. A trajectory in $S(q)$ has $\alpha M+\gamma\le y\le\beta M-\gamma$, so
\begin{align*}
 \alpha M'&\le\alpha M+\alpha\Delta<\alpha M+\gamma\le y,\\
 y&\le\beta M-\gamma<\beta M-\beta\Delta\le\beta M'.
\end{align*}
A trajectory outside $S(q)$ has either
\[
 y\le\alpha M-\gamma<\alpha M'
\]
or
\[
 y\ge\beta M+\gamma>\beta M'.
\]
Hence $S(q')=S(q)=:S$ and
\[
 F(q')-F(q)=(M'-M)\,q'(S)+M\{q'(S)-q(S)\}.
\]
(ii) A trajectory with $y\notin\{\alpha M,\beta M\}$ keeps its membership once $\beta|M(q_k)-M|$ is below its distance to the endpoints, as in (i). If $M(q_k)>M$, then $\alpha M(q_k)>y$ on $Z_-$ and $\beta M(q_k)>y$ on $Z_+$, so eventually $S(q_k)=S(q)\setminus Z_-$ and
\[
 F(q_k)=M(q_k)\,q_k(S(q)\setminus Z_-)\to M\{q(S(q))-w_-\}.
\]
The other case is symmetric.
\end{proof}

\begin{proof}[Proof of Corollary~\ref{cor:rates}]
(a) On $\{\|\widehat m-m^\star\|\le\rho\}$, Lemma~\ref{lem:mm-rate}(a,b) bounds the expected difference by $\tfrac12\mathrm{osc}(f)\kappa_\rho^{-1/2}\sqrt{d/(4n)}$. Off this event the difference is at most $\mathrm{osc}(f)$, and the event has probability at most $2de^{-2n\rho^2/d}$. Given $P^{\mathrm{mm}}$, $\E_{P^{\mathrm{int}}}f$ averages $S$ independent values with variance at most $\mathrm{osc}(f)^2/4$. The interpolated case uses Lemma~\ref{lem:mm-rate}(c) in the same way.

(b) \emph{Moment term.} On $G=\{\|\widehat m-m^\star\|\le r\}$, Lemma~\ref{lem:mm-rate}(a) with $f=y$ gives $\beta|M(P^{\mathrm{mm}})-M^\star|\le\gamma/2<\gamma$. Lemma~\ref{lem:band}(i), with Lemma~\ref{lem:mm-rate}(a) for $f=y$ and $f=\mathbf 1_{S(P^{\mathrm{syn}})}$, then gives
\begin{align*}
 |F(P^{\mathrm{mm}})-F(P^{\mathrm{syn}})|
 &\le\tfrac12(\bar y+M^\star)\kappa_\rho^{-1/2}\|\widehat m-m^\star\|\\
 &\le\bar y\kappa_\rho^{-1/2}\|\widehat m-m^\star\|.
\end{align*}
Off $G$ the difference is at most $\bar y$, since $0\le F\le\bar y$.

\emph{Integration term.} Hoeffding's inequality gives
\[
 \Pr\bigl(\beta|M(P^{\mathrm{int}})-M(P^{\mathrm{mm}})|\ge\gamma'\bigr)\le2e^{-2S\gamma'^2/(\beta\bar y)^2};
\]
otherwise Lemma~\ref{lem:band}(i) applies, with $\E|\Delta M|\le\bar y/(2\sqrt S)$ and $\E|\Delta q(S)|\le1/(2\sqrt S)$.

(c) \emph{Step 1: asymptotic normality.} Almost surely $\widehat m\to m^\star$, and $\sqrt n(\widehat m-m^\star)$ is asymptotically normal with covariance $\mathrm{Cov}_P(\phi)$.

\emph{Step 2: delta method.} The gradient of $M$ with respect to the moments at $m^\star$ is $v=\Sigma^{-1}\mathrm{Cov}_{P^{\mathrm{syn}}}(\phi,y)\neq0$, so $\sqrt n\{M(P^{\mathrm{mm}})-M^\star\}$ is asymptotically normal with variance $v^\top\mathrm{Cov}_P(\phi)v>0$, and each sign has limiting probability $\tfrac12$.

\emph{Step 3: the limit.} By Lemma~\ref{lem:band}(ii) and continuity of $m\mapsto Q_{\lambda(m)}$, with $\Delta M=M(P^{\mathrm{mm}})-M^\star$,
\[
 |F(P^{\mathrm{mm}})-F(P^{\mathrm{syn}})|-M^\star\{w_-\mathbf 1(\Delta M>0)+w_+\mathbf 1(\Delta M<0)\}\to0
\]
in probability. The difference is bounded by $\bar y$, so its expectation converges to $\tfrac12M^\star(w_-+w_+)$.
\end{proof}

\paragraph{Two examples.}
Both fall under (c).

\emph{First example.} If $Y\in\{1,\tfrac32\}$ with equal probabilities, then $M^\star=\tfrac54$, the band is $[1,\tfrac32]$, and both payment values sit on its endpoints ($w_-=w_+=\tfrac12$): shifting mass $t\neq0$ between them gives
\[
 F=(\tfrac54+\tfrac t2)(\tfrac12+|t|)\to\tfrac58,
\]
whereas $F=\tfrac54$ at $t=0$.

\emph{Second example.} If a unit prize is won with probability $p^\star=\tfrac56$, the prize sits on the upper endpoint $\beta M^\star=1$ ($w_+=\tfrac56$), and the limit is $\tfrac12\cdot\tfrac56\cdot\tfrac56=\tfrac{25}{72}$. This example runs through Eq.~\eqref{eq:tilt} with a nonsingular covariance: four trajectories $(C,D)\in\{(1,1),(2,1),(3,2),(4,2)\}$, features $(C/4,D/2)$, a uniform base, a prize for $C\ge3$, and truth $Q_{\lambda^\star}$ with $\lambda^\star=(0,2\log5)$.

\subsection{What the lower bounds cover}\label{app:chain-lower}
Proposition~\ref{prop:four-term}(ii) is a statement about tables. Its instances perturb the entries of a fixed menu directly, and the truths consistent with the final tables range over all tables within the stage errors, not over laws produced by an I-projection, a finite pilot and finite sampling. The next proposition asks what the statistical chain itself attains.

\begin{proposition}[Lower bounds attained by the statistical chain]\label{prop:chain-lower}
Take one layer with two plans $x,y$ that pay nothing, so that every portfolio is feasible and $\omega\equiv0$, and a value integrand $h$.

(a) \emph{Synthetic stage, constant attained.} Let $p_0$ and $\phi$ be as in Appendix~\ref{app:rates}, and let $h:\mathcal Z\to[0,1]$ not be affine in $\phi$. There is $\bar\varepsilon>0$ such that for every $\varepsilon\in(0,\bar\varepsilon]$ there are offer laws $P_x,P_y$ with the same I-projection $Q^\star$ and
\[
 \E_{P_x}h-\E_{Q^\star}h=\varepsilon=\E_{Q^\star}h-\E_{P_y}h .
\]
With population moments, exact summation and exact enumeration, $\varepsilon^v_{\mathrm{syn}}=\varepsilon$ and the other terms vanish, and every deterministic planner that uses the laws only through their $\phi$-moments has regret $2\varepsilon$ under one of the two labellings. Thus Eq.~\eqref{eq:four-term} holds with equality for the actual I-projection.

(b) \emph{Moment-matching stage, order attained.} Let $\mathcal Z=\{0,1\}$, $\phi(z)=h(z)=z$ and any full-support base, so that the tilt family contains every nondegenerate Bernoulli law and $P^{\mathrm{mm}}$ is the pilot frequency. Let $P_x,P_y$ be Bernoulli$(\tfrac12\pm\delta)$ with $\delta=(3/(128n))^{1/2}$, and let each plan receive $n$ independent pilot trajectories. Every rule that picks a plan from the pilot data has expected regret at least
\[
 \delta/2=(3/(512n))^{1/2}
\]
under one of the two labellings.

(c) \emph{Integration stage, order attained by sampling only.} The construction of (b), with $S$ independent draws from $P^{\mathrm{mm}}$ in place of the pilot, gives the $S$-draw integrator expected regret at least $(3/(512S))^{1/2}$ on some instance. Exact summation over a finite $\mathcal Z$ has no integration error (Lemma~\ref{lem:int}).
\end{proposition}
\begin{proof}
(a) \emph{Step 1: a direction invisible to the moments.} Since $h\notin\mathrm{span}\{1,\phi_1,\dots,\phi_d\}$ in $\R^{\mathcal Z}$, its residual $w$ after orthogonal projection onto that span is nonzero, with
\begin{align*}
 \sum_zw(z)&=0,\\
 \sum_zw(z)\phi(z)&=0,\\
 \langle w,h\rangle&=\|w\|^2;
\end{align*}
rescale $w$ so that $\langle w,h\rangle=1$.

\emph{Step 2: the two laws.} Let $Q^\star=Q_\lambda$ for any $\lambda$ and $\bar\varepsilon=\min_{w(z)\neq0}Q^\star(z)/|w(z)|$. Then
\begin{align*}
 P_x&=Q^\star+\varepsilon w,\\
 P_y&=Q^\star-\varepsilon w
\end{align*}
are laws with moments $\E_{Q^\star}\phi$. By the identity of Lemma~\ref{lem:syn}, for every $q$ with these moments,
\[
 \mathrm{KL}(q\Vert p_0)=\mathrm{KL}(q\Vert Q^\star)+\mathrm{KL}(Q^\star\Vert p_0),
\]
so $Q^\star$ is the I-projection of both, and their values are $\E_{Q^\star}h\pm\varepsilon$.

\emph{Step 3: the regret.} All later tables of $x$ and $y$ coincide, so a deterministic planner's choice depends only on the labels, and the labelling under which it picks the plan with law $P_y$ has regret $2\varepsilon$. A randomized planner has regret at least $\varepsilon$ averaged over the two labellings.

(b) \emph{Step 1: two-point reduction.} Let $\mathsf P_1,\mathsf P_2$ be the laws of the $2n$ pilot outcomes under the two labellings. A rule $\psi$ has regrets $2\delta\,\mathsf P_1(\psi=y)$ and $2\delta\,\mathsf P_2(\psi=x)$, whose sum is at least
\[
 2\delta\{1-\mathrm{TV}(\mathsf P_1,\mathsf P_2)\}.
\]

\emph{Step 2: bounding the distance.} For $\delta\le\frac14$,
\[
 \mathrm{KL}(\mathsf P_1\Vert\mathsf P_2)=4n\delta\log\frac{1+2\delta}{1-2\delta}\le\frac{64}{3}n\delta^2,
\]
so Pinsker's inequality gives
\[
 \mathrm{TV}\le\Bigl(\frac{32}{3}n\delta^2\Bigr)^{1/2}=\frac12 .
\]

(c) This is (b) with the $S$ draws as the data.
\end{proof}
Part~(a) shows that the I-projection itself attains the constant $2$ of the synthetic term, and that this term is the price of compressing the pilot to $d$ moments: $P_x$ and $P_y$ have different trajectory histograms, which a trial over every offer would observe. Parts~(b) and~(c) show that the $n^{-1/2}$ and $S^{-1/2}$ orders of Corollary~\ref{cor:rates} cannot be improved in general: the first by no use of the same pilot, the second only by giving up sampling.

\paragraph{Scope of the decomposition.}
Four kinds of loss are distinct, and only some are bounded.
(1) Proposition~\ref{prop:four-term}(i) is a perturbation bound: it compares the planner with the best portfolio over the \emph{same layer menus}. The value lost by restricting individuals to layer menus,
\[
 \Delta_{\mathrm{menu}}=V^\star_{\mathrm{ind}}-V^\star(0),
\]
with $V^\star_{\mathrm{ind}}$ the best true value over individual plans, is not bounded here, and neither are the ranking and roster choices of M1--M2; they are measured only against the layerings that were tested.
(2) The loss from tightening the budget, $\omega(2E_c)$, is bounded only under a margin or growth condition: on a finite menu $V^\star$ is a step function, so $\omega(2E_c)=0$ when an optimal portfolio's true spend lies in $\mathcal B_{2E_c}$, and $\omega(2E_c)\le2\kappa E_c$ when $\omega(\delta)\le\kappa\delta$. The tightening itself needs $E_c$, which is unknown when a plan is issued unless a certificate supplies it, and Algorithm~\ref{alg:request} uses the untightened band.
(3) Part~(ii) of the proposition holds at the level of tables; Proposition~\ref{prop:chain-lower} is the statement for the chain of laws.
(4) Part~(iii) and Proposition~\ref{prop:chain-lower} assert that certain instances exist. They do not say that on a given instance an accurate stage cannot partly offset another: errors can cancel, and the direct error in Table~\ref{tab:loss-check} is below the sum of the stage errors. Nor do they say that realized regret grows whenever one stage term grows. That a stage is needed in practice rests on the ablations and holds relative to the replacements tested there.

\clearpage
\section{Registered test families}\label{app:registered}
\begin{table}[ht]
\caption{Pre-specified decision rules and their outcomes (rules fixed before the runs, and except for the integrated study also the analysis scripts). Rider held-out validation (Appendix~\ref{app:holdout}): paired change, held-out minus reference, of plan-cost error and selection regret at budget fraction 1, 95\% intervals over origins; the rule requires upper bounds $\le0.045$ and $\le0.01$. Synthetic trajectories and a short pilot (Section~\ref{sec:synthetic}): the pre-specified primary arm is the framework with sampled integration ($S=2{,}048$); differences of penalized relative regret at 10\% scoring on the 18 response laws, 95\% intervals over laws. Integrated study (Appendix~\ref{app:integrated}): the registered confirmatory family on 36 response laws; H1 is the framework minus the same-calendar trial in relative regret (band within 20\% of the budget), H2 the geometric-mean speed-up of answering all 60 requests of an origin by reuse, one-off preparation charged, over re-running the whole pipeline per request; 95\% intervals over laws. Its analysis script was completed after the run; H2 is computed on the registered 36 response laws. New generator (Appendix~\ref{app:structural}): the same family on its 36 response laws, with the rules and the analysis script fixed before the run. Stronger controls, behaviour change, joint fit and population reduction (Appendix~\ref{app:strong}, studies C--F): registered families of seven, five, five and three one-sided tests, Holm step-down at level 0.025 within each family; estimates with 95\% law-bootstrap intervals, relative regret within 20\% of the budget (study F: retained share of the individual-level optimum). Study E: full re-run after the correction described in Appendix~\ref{app:strong}.}
\label{tab:prereg}
\centering\scriptsize\setlength{\tabcolsep}{3pt}
\begin{tabular}{@{}>{\raggedright\arraybackslash}p{0.21\linewidth}lll@{}}
\toprule
Held-out condition & Change in plan-cost error & Change in regret & Met\\
\midrule
Unseen period (fit once, 16 origins) & $+$0.051 [$+$0.002, $+$0.098] & $+$0.000 [$+$0.000, $+$0.000] & no\\
Unseen cities (four folds, 45 origins) & $+$0.003 [$-$0.018, $+$0.020] & $+$0.000 [$+$0.000, $+$0.000] & yes\\
\midrule
Pilot rule & 1-week pilot & 2-week pilot & 4-week pilot\\
\midrule
Near-lossless vs.\ 18-week trial (upper $\le0.05$) & $+$0.156 [$+$0.114, $+$0.201] (no) & $+$0.093 [$+$0.061, $+$0.126] (no) & $+$0.059 [$+$0.039, $+$0.077] (no)\\
Beats same-calendar trial (upper $<0$) & $-$0.256 [$-$0.290, $-$0.222] (yes) & $-$0.160 [$-$0.186, $-$0.134] (yes) & $-$0.076 [$-$0.109, $-$0.043] (yes)\\
Without history is worse (lower $>0$) & $+$0.371 [$+$0.266, $+$0.488] (yes) & $+$0.388 [$+$0.261, $+$0.525] (yes) & $+$0.380 [$+$0.258, $+$0.514] (yes)\\
Without moment matching is worse (lower $>0$) & $+$0.369 [$+$0.235, $+$0.499] (yes) & $+$0.433 [$+$0.302, $+$0.561] (yes) & $+$0.467 [$+$0.359, $+$0.574] (yes)\\
Without joint integration is worse (lower $>0$) & $+$0.296 [$+$0.226, $+$0.368] (yes) & $+$0.376 [$+$0.299, $+$0.458] (yes) & $+$0.428 [$+$0.374, $+$0.487] (yes)\\
Without budget allocation is worse (lower $>0$) & $+$0.581 [$+$0.516, $+$0.652] (yes) & $+$0.643 [$+$0.568, $+$0.719] (yes) & $+$0.676 [$+$0.612, $+$0.737] (yes)\\
\midrule
Integrated-study rule & 1-week pilot & 2-week pilot & 4-week pilot\\
\midrule
H1: beats same-calendar trial within 20\% (upper $<0$) & $-$0.171 [$-$0.201, $-$0.140] (yes) & $-$0.124 [$-$0.160, $-$0.089] (yes) & $-$0.069 [$-$0.104, $-$0.035] (yes)\\
H2: 60 requests faster with preparation (lower $>1$, same answers) & 14.1$\times$ [13.8, 14.5] (yes) & 14.2$\times$ [13.8, 14.5] (yes) & 14.3$\times$ [14.0, 14.7] (yes)\\
\midrule
New-generator rule & 1-week pilot & 2-week pilot & 4-week pilot\\
\midrule
H1: beats same-calendar trial within 20\% (upper $<0$) & $-$0.289 [$-$0.311, $-$0.267] (yes) & $-$0.176 [$-$0.192, $-$0.160] (yes) & $-$0.123 [$-$0.140, $-$0.108] (yes)\\
H2: 60 requests faster with preparation (lower $>1$, same answers) & 2.70$\times$ [2.67, 2.73] (yes) & 2.72$\times$ [2.69, 2.75] (yes) & 2.78$\times$ [2.76, 2.81] (yes)\\
\midrule
Stronger-control rule & 1-week pilot & 2-week pilot & 4-week pilot\\
\midrule
H6.1: tilted trial minus trial ($<0$) & $-$0.075 [$-$0.078, $-$0.072] (yes) & & \\
H6.2: minus dose interpolation ($<0$) & $-$0.022 [$-$0.024, $-$0.021] (yes) & & \\
H6.3: minus normal, exact truncated moments ($<0$) & $-$0.053 [$-$0.055, $-$0.051] (yes) & & \\
H6.4: calendar shocks, minus trial ($<0$) & $-$0.267 [$-$0.274, $-$0.259] (yes) & $-$0.214 [$-$0.222, $-$0.206] (yes) & \\
H6.5: ESS deficit vs.\ cost error, Spearman ($>0$) & $+$0.415 [$+$0.400, $+$0.431] (yes) & & \\
H6.6: support 1,024 minus 16,384 weeks ($>0$) & $+$0.006 [$+$0.006, $+$0.006] (yes) & & \\
\midrule
Behaviour-change rule & 1-week pilot & 2-week pilot & 4-week pilot\\
\midrule
H7.1--H7.3: request-aware minus direct reuse ($<0$) & $-$0.195 [$-$0.202, $-$0.189] (yes) & $-$0.255 [$-$0.261, $-$0.249] (yes) & $-$0.300 [$-$0.306, $-$0.293] (yes)\\
H7.4: same, selected subpopulations ($<0$) & $-$0.347 [$-$0.356, $-$0.338] (yes) & & \\
H7.5: request-aware minus trial, 288 laws ($<0$) & $-$0.129 [$-$0.136, $-$0.123] (yes) & & \\
\midrule
Joint-fit rule (study E) & 1-week pilot & 2-week pilot & 4-week pilot\\
\midrule
H9.1: framework minus joint fit of the trial ($<0$) & $-$0.055 [$-$0.059, $-$0.051] (yes) & $-$0.052 [$-$0.056, $-$0.048] (yes) & $-$0.060 [$-$0.066, $-$0.053] (yes)\\
H9.2: joint fit minus per-offer tilt of the trial ($<0$) & $-$0.152 [$-$0.158, $-$0.146] (yes) & & \\
H9.3: cost-table error, framework minus joint fit ($<0$) & $-$0.014 [$-$0.015, $-$0.014] (yes) & & \\
\midrule
Reduction rule (study F) & \multicolumn{3}{l}{retained share of the individual-level optimum, difference}\\
\midrule
H8.1: five layers minus group-level plan ($>0$) & $+$0.120 [$+$0.116, $+$0.125] (yes) & & \\
H8.2: five quantile layers minus five equal-width bins ($>0$) & $-$0.002 [$-$0.002, $-$0.001] (no) & & \\
H8.3: five layers minus two layers ($>0$) & $+$0.046 [$+$0.044, $+$0.047] (yes) & & \\
\bottomrule
\end{tabular}
\end{table}

\end{document}